\documentclass{article}
\usepackage[english]{babel}
\usepackage[utf8]{inputenc}

\usepackage{color,xcolor,ucs}

\usepackage{floatrow}
\usepackage{tabularx}
\usepackage{float}
\usepackage{amsfonts}
\usepackage{helvet}         % selects Helvetica as sans-serif font
\usepackage{courier}        % selects Courier as typewriter font
\usepackage{type1cm}        % activate if the above 3 fonts are
\usepackage{amsmath}
\usepackage{amssymb}

\usepackage{amsthm}
\usepackage{makeidx}         % allows index generation
\usepackage{comment}         % allows index generation
\usepackage{graphicx}        % standard LaTeX graphics tool
\usepackage{multicol}        % used for the two-column index
\usepackage[bottom]{footmisc}% places footnotes at page bottom
\usepackage{bm}

\usepackage{cite}
\usepackage{url}

\usepackage[unicode=true, bookmarks=true,colorlinks=true]{hyperref}
\usepackage{xr-hyper}

\usepackage{epsfig}
\usepackage{epstopdf}
\usepackage{caption}
\usepackage{subcaption}

\usepackage{xspace}
\usepackage{rotating}

\usepackage{tikz}
\usepackage{tkz-graph}
\usetikzlibrary{arrows,shapes,shadows,positioning,calc}

\usepackage{graphicx}
\usepackage{threeparttable}
\usepackage{multirow}
\usepackage[font=scriptsize,labelfont=bf]{caption}

\usepackage[margin=0.75in,bottom=0.85in]{geometry}

\usepackage{algpseudocode,algorithm,algorithmicx}
\DeclareMathOperator*{\argmax}{arg\,max}

\definecolor{vadim}{rgb}{0.0,0.0,1.0}
\definecolor{vova}{rgb}{0.05,0.5,0.05}

\newcommand{\prob}[1]{\ensuremath{{\mathbb{P}({#1})}}}

\algrenewcommand\algorithmicrequire{\textbf{Input:}}
\algnewcommand{\LineComment}[1]{\State \(\triangleright\) #1}

\newcommand{\poses}{\ensuremath{{\cal X}}\xspace}

\newcommand{\likelihoods}{\ensuremath{{\cal L}}\xspace}
\newcommand{\motionmodel}{\ensuremath{{\cal M}}\xspace}

\newcommand{\loggamma}{\ensuremath{{l\gamma}}\xspace}
\newcommand{\logGamma}{\ensuremath{{l\Gamma}}\xspace}
\newcommand{\loglambda}{\ensuremath{{l\lambda}}\xspace}

\newcommand{\his}{\ensuremath{{\cal H}}\xspace}

\newcommand{\appropto}{\mathrel{\vcenter{
			\offinterlineskip\halign{\hfil$##$\cr
				\propto\cr\noalign{\kern2pt}\sim\cr\noalign{\kern-2pt}}}}}

\newcommand{\sembsp}{\texttt{EUS-BSP}\xspace}		
\newcommand{\mh}{\texttt{MH}\xspace}			
\newcommand{\jlp}{\texttt{JLP}\xspace}
\newcommand{\mhbsp}{\texttt{MH-BSP}\xspace}
\newcommand{\jlpbsp}{\texttt{JLP-BSP}\xspace}		
\newcommand{\weu}{\texttt{WEU}\xspace}

\title{Epistemic Uncertainty Aware Semantic Localization and Mapping for Inference and Belief Space Planning}

\author{Vladimir Tchuiev and Vadim Indelman% <-this % stops a space
\thanks{The authors are with the Department of Aerospace Engineering, Technion - Israel Institute of Technology, Haifa 32000, Israel. {\tt\{vovatch, vadim.indelman\}@technion.ac.il}. 
This work was  partially supported by US NSF/US-Israel BSF, and by the Israel Ministry of Science \& Technology (MOST).
}
}

\date{}

\newtheorem{lemma}{Lemma}

\graphicspath{{figures/}}

\makeatletter
\def\endthebibliography{%
	\def\@noitemerr{\@latex@warning{Empty `thebibliography' environment}}%
	\endlist
}
\makeatother

\begin{document}
	
	\maketitle
	\thispagestyle{empty}
	\pagestyle{empty}
	
	% !TeX root = Paper Main.tex

\begin{abstract}

We investigate the problem of autonomous object classification and semantic SLAM, which in general exhibits a tight coupling between classification, metric SLAM and planning under uncertainty. We contribute a unified framework for inference and  belief space planning (BSP) that addresses  prominent sources of uncertainty in this context: classification aliasing (classifier cannot distinguish between candidate classes from certain viewpoints), classifier epistemic uncertainty (classifier receives data "far" from its training set), and  localization uncertainty (camera and object poses are uncertain). Specifically, we develop two methods for maintaining a joint distribution over robot and object poses, and over posterior class probability vector that  consider epistemic uncertainty in a Bayesian fashion. The first approach is Multi-Hybrid (MH), where multiple hybrid beliefs over poses and classes are maintained to approximate the joint belief over poses and posterior class probability. The second approach is Joint Lambda Pose (JLP), where the joint belief is maintained directly using a novel JLP factor. Furthermore, we extend both methods to BSP, planning while reasoning about future posterior epistemic uncertainty indirectly, or directly via a novel information-theoretic reward function. Both inference methods utilize a novel viewpoint-dependent classifier uncertainty model that leverages the coupling between poses and classification scores, and predicts the epistemic uncertainty from certain viewpoints.  In addition, this model is used to generate predicted measurements during planning. To the best of our knowledge, this is the first work that reasons about classifier epistemic uncertainty within semantic SLAM and BSP.  We  evaluate extensively our inference and BSP approaches in simulation and using real data from the Active Vision Dataset. Results clearly indicate superior classification performance of our methods compared to an approach that is not epistemic uncertainty aware.

\end{abstract}

	% ugly?
	%\setlength{\abovecaptionskip}{-10pt}
	%\setlength{\textfloatsep}{14pt}
	
	%\vspace{-5pt}
	
	\section{Introduction}
	\label{sec:introduction}
	% !TeX root = Paper Main.tex

Simultaneous localization and mapping (SLAM) is a fundamental problem in robotics and computer vision, with wide reaching applications such as autonomous vehicles and UAVs, agriculture, medical, search and rescue, and more \cite{Cadena16tro}. Specifically semantic SLAM, where a robot localizes itself and maps the environment using information from objects within it, is an actively researched field. For semantic SLAM, object classification is a crucial problem. With advances in recent years with deep-learning-based algorithms, classifiers today outperform humans in multiple classification tasks. Yet, classifiers are limited by their training, and as such, may provide unreliable results in different conditions such as lighting, image resolution, and occlusions.  In addition, from certain viewpoints the classifier may struggle distinguishing between different classes, resulting in classification aliasing. Faced with these uncertainties, classification scores may appear sporadic and unreliable, making reliable decision making a significant challenge. State-of-the-art semantic SLAM approaches do not directly reason about these uncertainties, a gap which we aim to address.

In recent years, the field of object classification saw many advances with the introduction of deep-learning-based classifiers; Most modern deep-learning based classifiers provide, given a set of candidate classes, a vector of class probabilities for a photographed object. These classifiers are trained on a set of examples for each class, and during deployment infer the observed objects' class based on said training set. If the observation does not match images on the training set, the classification result is unreliable, and if not accounted for may result in erroneous classification. Consequentially, a slight variation in classifier weights or the input may greatly change the output. This variation is referred to as \emph{epistemic} uncertainty, or model uncertainty. Several approaches were proposed to identify this uncertainty, such as Monte-Carlo (MC) dropout \cite{Gal16icml} or Bootstrapping \cite{Paass93nips}. In our work, we utilize an epistemic-uncertainty-aware classifier and incorporate it within our semantic SLAM framework.

In general, the classifier output depends on the relative viewpoint between camera and object. This dependency can be modeled \cite{Kopitkov18iros, Tchuiev19iros, Tchuiev20ral, Feldman20arj}, and then be used to improve classification and localization accuracy within a SLAM setting. But, this kind of model was not used in epistemic uncertainty aware classification. While approaches that consider the accumulated epistemic uncertainty from multiple images, i.e. the posterior epistemic uncertainty exist (see e.g. \cite{Tchuiev18ral}), they decouple the relative pose between object and camera, which is a gap we address. We introduce a viewpoint dependent classifier uncertainty model that can be both utilized for inference and later in planning.

Eventually, semantic SLAM with epistemic-uncertainty-aware classification opens the possibility of performing 'safe' decision making based on a new type of reward functions that consider epistemic uncertainty. Thereafter, this paper presents a novel active semantic SLAM approach that reasons about epistemic uncertainty. To the authors' best knowledge, this is the first work that plans over classifier epistemic uncertainty with uncertain localization as well.

%The general problem that corresponds to decision-making under uncertainty is a partially observable Markov decision process (POMDP) \cite{Kaelbling98ai}. 

Specifically, we formulate the active semantic SLAM problem within the belief space planning (BSP) framework, which is an instantiation of a partially observable Markov decision process (POMDP) \cite{Kaelbling98ai}, and consider belief-dependent reward functions. Our approach considers both localization and classifier epistemic uncertainty within BSP by maintaining a joint belief over the robot and object poses, and importantly, over the objects' class posterior probabilities. Having access to such a joint belief within BSP allows to consider classifier posterior epistemic uncertainty \emph{implicitly} using standard reward functions over the state and information-theoretic rewards. Crucially, it enables also utilizing \emph{novel} reward functions, directly over the classifier's posterior epistemic uncertainty. In this paper we introduce such a reward function and develop methods for its computation (Section \ref{sec:Uncertainty_rewards}).

Further, an inherent aspect in BSP is belief propagation and reward calculation considering different candidate actions while accounting for possible future observations (see e.g.~\cite{Indelman15ijrr, Farhi19icra}). As an analytical calculation of the corresponding expectation operator is generally not available, a common approach is to resort to a sampling-based approximation, which however involves generating future observations.  In our context, one may consider doing so by generative new images, such as, e.g.~in \cite{Ha18arxiv}. With this alternative, these images would be fed to a classifier to get the corresponding cloud of future semantic measurements (that represents the epistemic uncertainty). However, our key observation is that we can use instead the viewpoint-dependent classifier uncertainty model to generate these future semantic measurements directly.

\subsection{Related Work}

Various works presented approaches for sequential classification. Coates and Y. Ng \cite{Coates10icra} presented an approach that maintains a posterior class probability via multiplication of classification scores from an image with the prior class probability. Static State Bayes Filter (SSBF) by Omidshaifei \cite{Omidshafiei16arxiv} expanded the aforementioned approach for multiple classes. Hierarchical Bayesian Noise Inference by Omidshafiei at el. \cite{Omidshafiei16arxiv} maintains a posterior class probability vector by utilizing a Dirichlet distributed classifier model. All of these approaches do not consider epistemic uncertainty. Tchuiev and Indelman \cite{Tchuiev18iros} presented an epistemic-uncertainty-aware sequential method, while utilizing MC-dropout by Gal et al. \cite{Gal16icml} as the mechanism for extracting the epistemic uncertainty for each image. An alternative mechanism might be e.g. Bootstrapping \cite{Paass93nips} where multiple classifiers are trained on the same training set, or using auxiliary training techniques and post hoc statistics to detect out-of-distribution input data as proposed by Nitsch et al. \cite{Nitsch20arxiv}. Malinin and Gales \cite{Malinin18nips} proposed prior networks for reasoning about epistemic uncertainty in neural network outputs. These works either did not reason about epistemic uncertainty in classification, or did so without considering localization uncertainty as well. We propose a semantic SLAM approach that performs sequential classification, addresses localization uncertainty and reasons about posterior epistemic uncertainty in classification.

Some works utilized a viewpoint dependent classifier model; Velez at el. \cite{Velez12jair} and Teacy et al. \cite{Teacy15aamas} utilized a viewpoint dependent classifier model in the context of active classification with known poses. Segal and Reid \cite{Segal14iros} proposed an inference approach for general hybrid beliefs based on message passing. Kopitkov and Indelman \cite{Kopitkov18iros} presented a Gaussian viewpoint dependent classifier model, and used it for robot localization in a setting where the object class and pose are already known. Feldman and Indelman \cite{Feldman18icra} presented a sequential classification approach with a viewpoint dependent classifier model where the poses are known a-priori. Tchuiev et al. \cite{Tchuiev19iros} showed that utilizing a viewpoint dependent classifier model in a setting of semantic SLAM assists in solving the data association problem. The approach utilized a hybrid belief over poses and classes. This approach was expanded upon to a multi-robot semantic SLAM setting in \cite{Tchuiev20ral}. Ok et al. \cite{Okliu19icra} presented an approach for objected based SLAM that used a viewpoint-dependent texture plane measurement model, which is similar in concept to a viewpoint dependent classifier model. All these approaches utilized a viewpoint dependent classifier model that did not consider epistemic uncertainty, while on the other hand we do consider the epistemic uncertainty for inference and planning.

Approaches that incorporate object classification within planning include 
Atanasov et al. \cite{Atanasov14tro} and Patten et al. \cite{Patten18arj} which presented approaches for active classification using a viewpoint dependent classifier mode using a sampling based method. In the former, the robot and object poses are known, while in the latter they are part of the state. Continuous state partially observable Markov decision process (CPOMDP) by Burks et al. \cite{Burks19tro} is also capable of reasoning about hybrid beliefs. These approaches, however, did not consider the classifier's epistemic uncertainty.

Several planning approaches that do reason about epistemic uncertainty were proposed;
Faddoul et al. \cite{Faddoul15ejor} reasoned about epistemic uncertainty in MDP and POMDP transition matrices, creating a framework for decision making. 
Hayashi et al. \cite{Hayashi19icra} proposed an approach that actively trains uncertain dynamic models via neural network priors. These works do not consider epistemic uncertainty in the context of classification.
Lutjens et al. \cite{Lutjens18arxiv} presented a reinforcement learning approach that reasons about epistemic uncertainty for obstacle avoidance with known object poses. The approach utilized both MC dropout and bootstrapping for extracting epistemic uncertainty from measurements. On the other hand, we consider a BSP approach with a belief over poses and class probabilities, jointly considering both localization and classifier epistemic uncertainty within a semantic SLAM framework.

To generate measurements, one may consider generating raw images when performing classifier epistemic-uncertainty-aware planning.
Ha et al. \cite{Ha18arxiv} proposed World Models: a neural network that creates an image given a pose within the environment the network was trained on. Wang et al. \cite{Wang19cvpr} proposed an image extrapolation approach using feature expansion network (FEN) and context prediction network (CPN). Mildenhall et al. \cite{Mildenhall20eccv} presented Neural Radiant Fields (NeRF) which rendered images using volume-rendering techniques with a neural network trained on images of the environment with corresponding poses.
On the other hand, we present an approach that generates measurements via our proposed viewpoint dependent classifier uncertainty model.

\subsection{Contributions}

In this paper we contribute a unified framework for epistemic uncertainty aware inference and belief space planning in the context of semantic perception and SLAM. Our framework considers prominent sources of uncertainty --- classification aliasing, classifier epistemic uncertainty, and localization uncertainty --- within inference and BSP.

Specifically, the main contributions of this paper are as follows.
\begin{enumerate}
	
	\item We develop two methods for maintaining a joint distribution over robot and object poses, and over the posterior class probability vector that considers epistemic uncertainty in a Bayesian fashion. The first approach is Multi-Hybrid (MH), where multiple hybrid beliefs over poses and classes are maintained to approximate the joint belief over poses and posterior class probability. The second approach is Joint Lambda Pose (JLP), where the joint belief is maintained directly using a novel JLP factor. 
	
	\item We extend both methods to a BSP framework, planning over posterior epistemic uncertainty indirectly, or directly via a novel information-theoretic reward over the distribution of posterior class probability.
	
	\item Our inference and BSP methods utilize a novel viewpoint dependent classifier model that predicts epistemic classifier uncertainty given a candidate class and relative viewpoint, allowing us to reason about the coupling between poses and classification scores, and predict future epistemic classifier uncertainty, while avoiding predicting and generating entire images.		
	
	\item We extensively study our inference and BSP methods in simulation and using real data from the Active Vision Dataset \cite{Ammirato17icra}. 	
\end{enumerate}

\subsection{Paper Structure}

This paper is structured as follows: We cover preliminary material  and formulate the addressed problem in Sec.~\ref{sec:preliminaries}, and then provide a brief approach overview in Sec.~\ref{sec:ApproachOverview}. 
 In Sec.~\ref{sec:approach-inference} we address epistemic-uncertainty-aware inference; MH and JLP are introduced,  first for the single object case and afterwards for the  multiple objects case. In Sec.~\ref{sec:approach-planning} we expand both approaches to BSP; specifically, in Sec.~\ref{sec:Uncertainty_rewards} we introduce  and develop the calculation of our novel information-theoretic reward over the distribution of posterior class probability. Finally, we validate our approaches first in simulation in Sec.~\ref{sec:Sim}, and then using Active Vision Dataset and BigBIRD in Sec.~\ref{sec:Exp}.

	\section{Background and Problem Formulation}
	\label{sec:preliminaries}
	% !TeX root = Paper Main.tex

In this section we introduce notations, provide preliminary material, and formulate the problem addressed in this work. First, we introduce our setting and simulatneous localization and mapping (SLAM) notations. Afterwards, we introduce notations specifically for classification in the context of epistemic uncertainty. Finally, we  briefly introduce belief space planning (BSP), and present the problem formulation for epistemic uncertainty aware semantic inference and planning.

For the reader's convenience, main notations used in this paper are  summarized in Table \ref{table:Notation}.

\begin{table}\scriptsize{
		\caption{Main notations used in the paper.\label{table:Notation}}
		\begin{tabularx}{\textwidth}{p{0.12\textwidth}X}
			
			\textbf{Parameters} \\
			$x$ & Robot pose\\
			$x^o$ & Object $o$'s pose \\
			$\poses_k$ & All robot and object poses up to $k$ \\
			$x^{rel}$ & Relative pose between $x$ and $x^o$ \\ 
			$O_k$ & Set of all objects observed at time $k$ \\ 
			$x^{inv}_k$ & Set that contains the last robot pose and all object poses from $O_k$ \\
			$c^o$ & Object $o$'s class \\
			$C$ & Class realization of all objects \\
			$z^g$ & Geometric measurement \\
			$z^s$ & Semantic measurement \\
			$n$ & The amount of all objects in the environment \\
			$n_k$ & Number of objects observed at time $k$ \\
			$N_k$ & Number of objects observed up to time $k$ \\
			$\motionmodel_k$ & Motion model from $x_{k-1}$ to $x_k$ \\
			$a$ & Robot action \\
			$\his_k$ & History of measurements and action up to time $k$ \\
			$\his^g_k$ & History of geometric measurements and action up to time $k$ \\
			$Z^g_k$ & All geometric measurements for all objects at time $k$ \\
			$\likelihoods^s$ & Semantic measurement likelihood \\
			$h_c$ & Expectation of class $c$'s classifier uncertainty model \\
			$\Sigma_c$ & Covariance of class $c$'s classifier uncertainty model \\
			$\likelihoods_k$ & Geometric and semantic measurement likelihood at time $k$ \\
			$D$ & Classifier training dataset \\
			$\{ \cdot \}$ & Set or point cloud \\
			$I$ & Raw image \\
			$l\square$ & Logit transformation of probability vector \\
			$\gamma$ & Probability vector classifier output \\
			$\gamma^c$ & Element of $\gamma$ of class $c$ \\
			$\Gamma_k$ & Set of all $\gamma$ observed at time $k$, one per object \\ 
			$\logGamma_k$ & Set of all logit transformations for all $\gamma_k \in \Gamma_k$ \\
			$\lambda$ & Posterior class probability vector \\
			$\lambda^c$ & Element of $\lambda$ of class $c$ \\
			$\Lambda_k$ & Posterior probability vector for class realizations \\
			$\bar{\loglambda}_k$ & Set of $\loglambda$ of all objects observed up to $k$ \\
			$W$ & Set of all possible classifier weight realizations $w$ \\ 
			$b[\cdot]$ & Belief, probability conditioned on history $\prob{\cdot|I_{1:k},\his^g_k,D}$. \\
			$b^c_w$ & Continuous belief conditioned on history, $c$, and $w$ \\
			$hb_w$ & Hybrid belief conditioned on $w$ \\
			$l\likelihoods^s$ & Logit transformation of semantic measurement likelihood \\

			\textbf{Subscripts} \\
			$w$ & Classifier weight realization\\
			$k$ & Time step \\
			$L$ & Planning horizon \\
			
			\textbf{Superscript} \\
			$o$ & Object $o$ \\
			$c$ & Class hypothesis of an object \\ 
			$C$ & Class hypothesis of all objects \\
			
	\end{tabularx}}
\end{table}

\subsection{Simultaneous Localization and Mapping (SLAM)}\label{sec: SLAMHybrid}

Consider a robot operating in an unknown environment represented by object landmarks. For inference and planning over a distribution of posterior class probabilities, we need to solve an underlying object based simultaneous localization and mapping problem (SLAM). 
The robot's and objects pose, and objects' classes are all unknown. Let $x_k$ denote the robot pose at time $k$; Let $x^o$ and $c$ denote object pose and class respectively. To shorten notations, denote $\poses_x \triangleq \{ x^o, x_{1:k} \}$ as all poses of robot and the observed (expanded later to multiple objects) up until time $k$.

The robot receives from observed objects both geometric and semantic measurements. Let $z_{k}$ denote a measurement received at time $k$ from the object. This measurement is split into geometric $z^g_k$ and semantic $z^s_k$ measurements; All those measurements are aggregated to a set $z_k \triangleq \{ z^g_k, z^s_k \}$. The robot action at time $k$ is denoted $a_k$, and finally we  denote the measurement history as $\his_k \triangleq \{ z_{1:k}, a_{0,k-1} \}$. We assume independence between semantic and geometric measurements, as well between different time steps. 

We utilize a known Gaussian motion model with constant parameters, denoted $\motionmodel_{k}$, and defined as:
\begin{equation}\label{eq:Motion_Model}
	\mathcal{M}_k \triangleq \prob{x_k|x_{k-1},a_{k-1}}
\end{equation}
and a known geometric model $\prob{z^g_k|x^o,x_k}$. In addition, we use an externally trained viewpoint dependent classifier and uncertainty model $\prob{z^s_{k,n}|c_n,x^o,x_k}$ that will be discussed in Section \ref{sec:VDCM}. Let us denote the corresponding measurement likelihood term,
\begin{equation}\label{eq:Likelihoods}
	\mathcal{L}_k \triangleq \prob{z^g_k|x^o,x_k} \cdot \prob{z^s_k|c,x^o,x_k},
\end{equation}
where, both geometric and classifier models are considered Gaussian as well.

%----------------------------------------------------------------------------------------
\subsection{Distribution Over Class Probability Vector}\label{sec:Gamma_Lambda}

During inference the robot receives a raw image in which observed objects are segmented. In standard (deep-learning) approaches a classification model, i.e. a classifier, is learned beforehand and used to classify the objects within each segment (e.g. bounding box) by producing an output of a class probability vector. Given fixed classifier weights $w$, we  denote a probability vector from a classifier at time $k$ as
\begin{equation}
	\gamma_{k} \triangleq \prob{c|I_{k},w},
\end{equation}
where $I_{k}$ is the raw image of the object. Also, denote $\gamma_{k,w}$ as the probability vector given a specific $w$. In practice, the image fed into the classifier is a cropped image of an object via a bounding box. Note that $\gamma_k \triangleq [ \gamma^1_k,...,\gamma^m_k ] \in \mathbb{R}^m$ is a probability vector, thus it must satisfy the following conditions:
\begin{itemize}
	
	\item All its elements must sum to 1, i.e. $\sum_{i=1}^m \gamma^i_k = 1$.
	
	\item Each element is bounded between 0 and 1, i.e. $0 \leq  \gamma^i_k \leq 1, \;\; \forall i = 1,...,m$.
	
\end{itemize}

In contrast to this standard approach, in this work we reason about classifier epistemic uncertainty. Denote $D$ as the classifier's training set. In literature, these approaches rely on describing the trained weights $w$ as random variables by themselves distributed $w \sim \prob{w|D}$, thus making $\gamma_k$ a random variable. In this paper we create a set $W$ of sampled $w$ to produce a point cloud of $\gamma_k$ vectors per object and time step, such that we can describe the distribution over $\gamma_k$ with the delta Dirac function $\delta(\cdot)$:
%\textsl{}
\begin{equation}
	\gamma_{k} \sim \prob{\gamma_k|I_k,D} = \int_w \delta(\gamma_k = \prob{c|I_{k},w}) \prob{w|D} dw,
\end{equation}
which we approximate via sampling as:
\begin{equation}
	\prob{\gamma_k|I_k,D} \approx \frac{1}{|W|} \sum_w \delta(\gamma = \prob{c|I_{k},w}).
\end{equation}
Thus for each time step we get a point cloud $\{\gamma_k\}$ per object where its spread describes the epistemic model uncertainty of the classifier. See a simplified illustration in Fig.~\ref{fig:SUV_demo}, where an object is observed from multiple viewpoints, and the classifier outputs a cloud of $\gamma$'s for each viewpoint. For example, the cloud $\{\gamma_k\}$ obtained  by observing the object from the bottom right corner is spread widely, therefore the epistemic uncertainty from that viewpoint is high. Contrast it with the upper-right viewpoint where the spread is tight, representing low epistemic uncertainty. In this paper the semantic measurements are those point clouds within the $m-1$ simplex, such that $z^s_k = \{ \gamma_{k} \}$. The set of sampled $w$ can be created by, for example, MC-dropout \cite{Gal16icml} or Bootstrapping \cite{Paass93nips}.

\begin{figure}[!htbp]
	\centering
	
	\includegraphics[width=0.5\textwidth]{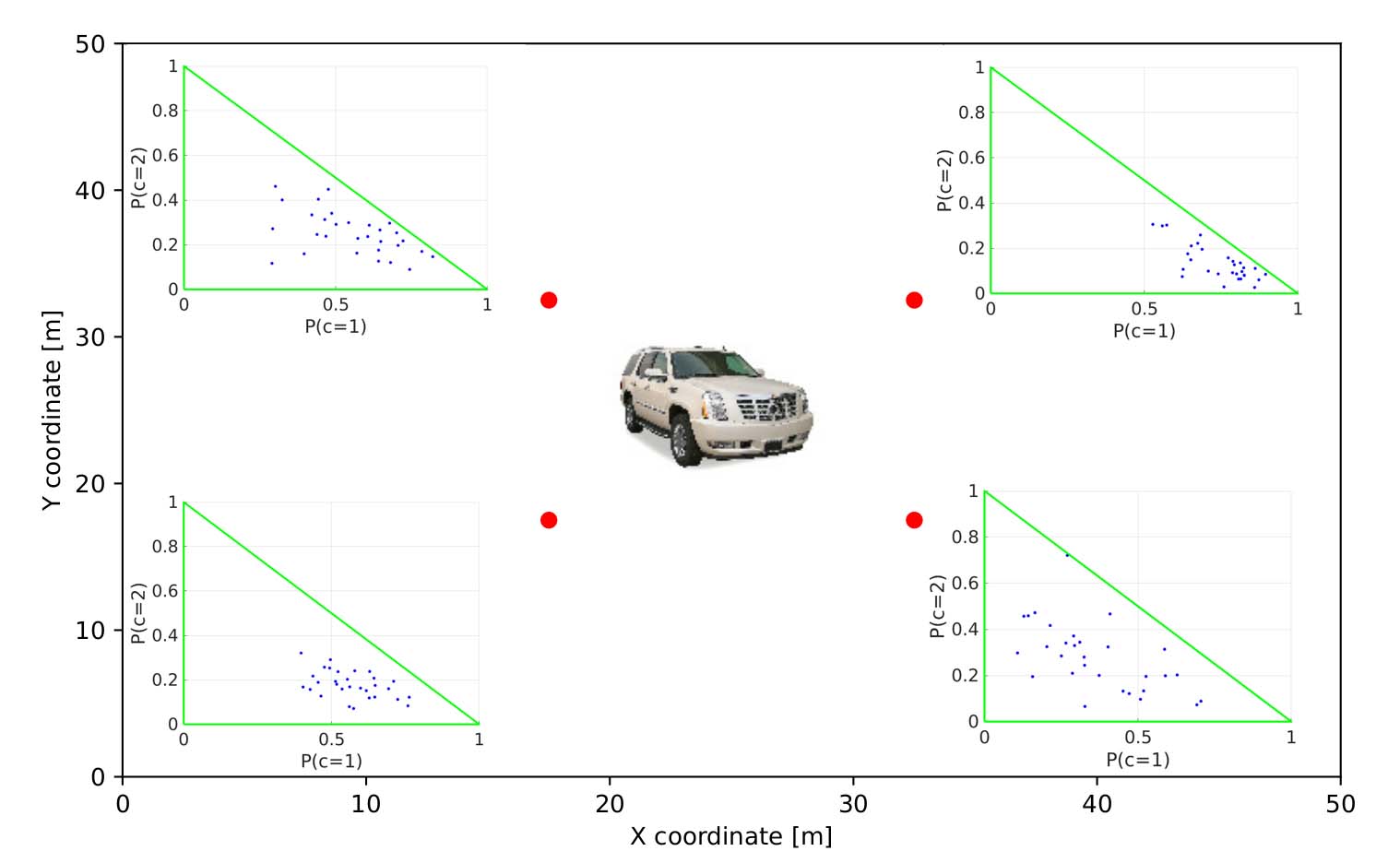}
	\caption{Illustration of viewpoint dependency for both classification scores and epistemic uncertainty. The figure presents simplex graphs for different viewpoints, where $m=3$. The individual class probability scores are shown as blue points in the simplex, where it's borders are in green. The red points represent possible viewpoints observing the SUV in the middle.}
	
	\label{fig:SUV_demo}
\end{figure}

%--------------------------------------------
\subsection{Distribution Over Posterior Class Probability Vector}

Eventually the posterior over a \emph{sequence} of $\gamma$ vectors can be inferred. This posterior takes into account both the epistemic uncertainty from multiple observations of an object, as well as localization uncertainty induced by coupling between relative poses and class probabilities. The posterior is defined as follows: 
\begin{equation}\label{eq:Lambda_Def}
	\lambda_k \triangleq \prob{c|\gamma_{1:k},z^g_{1:k}},
\end{equation}
where $\lambda_k$ is deterministically determined by both a sequence $\gamma_{1:k}$ and the geometric measurement history. For a specific $\gamma_{1:k,w}$ sequence which is created by a specific $w$, we use the notation $\lambda_{k,w}$. Because we consider $\gamma_{1:k}$ to be a random variable (as $w$ is a random variable), so is $\lambda_k$. As such, we can define a belief over $\lambda_k$ the following way:
\begin{equation}\label{eq:Belief_Lambda_Def}
	b[\lambda_k] \triangleq \prob{\lambda_k|I_{1:k},\his^g_k,D}.
\end{equation}
The belief $b[\lambda_k]$ encompasses both the posterior classification probability vector via $\mathbb{E}(\lambda_k)$, and the epistemic and localization uncertainty via $Cov(\lambda_k)$.
The belief $b[\lambda_{k}]$ representation is more expressive than a single class probability vector representation, and it can reflect four possible archetypes, as seen in Fig.~\ref{fig:simplex_figures} (see \cite{Malinin18nips}). Fig.~\ref{fig:Fig_Unknown_Unknown} presents an out-of-distribution case where the inputs to the classifier are totally alien, therefore the output is completely unpredictable. Fig.~\ref{fig:Fig_Known_Known} represent a case where the classifier can safely identify the object with high degree of certainty, i.e. the input is close to the training set. Intuitively, this is the case that we aim for, and generally has the highest reward. Fig.~\ref{fig:Fig_Known_Unknown} represents the case of high data uncertainty where the classifier certainly cannot disambiguate between different classes, i.e. the classifier "knows" that it does not know. This can be resulted from ambiguity in the training set between different classes, when objects from different classes look identical from certain viewpoints. Finally, Fig.~\ref{fig:Fig_Uncertain_Edge} represent a case where the classifier can vaguely infer the object class, but it's still far from the training set (e.g. a car of an unusual shape that there are no similar images in the training set), therefore with a large degree of uncertainty.

\begin{figure*}[!htbp]
	
	\begin{subfigure}[b]{0.23\textwidth}
		\includegraphics[width=\textwidth]{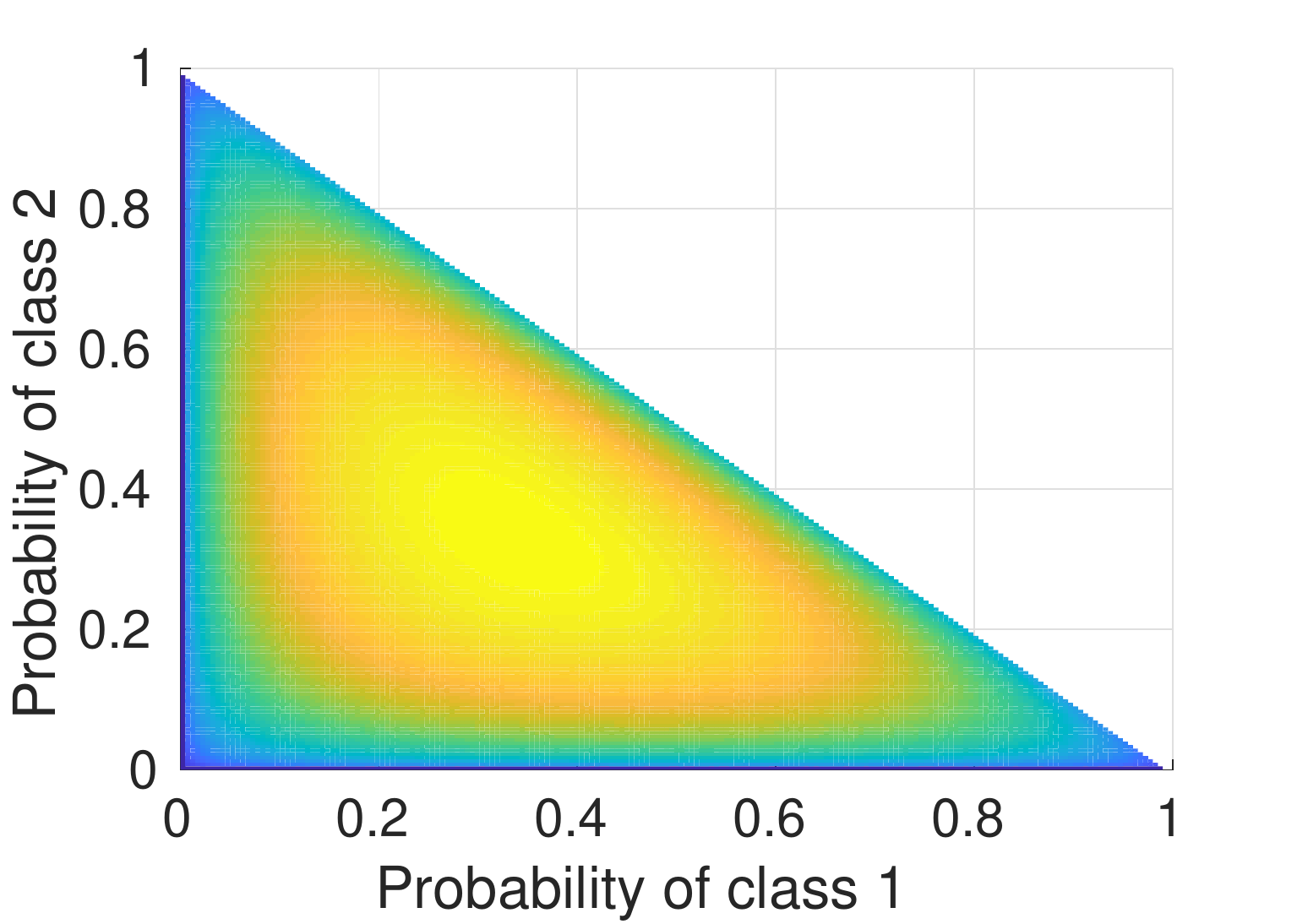}
		\caption{Unknown-unknown}\label{fig:Fig_Unknown_Unknown}
	\end{subfigure}
	\begin{subfigure}[b]{0.23\textwidth}
		\includegraphics[width=\textwidth]{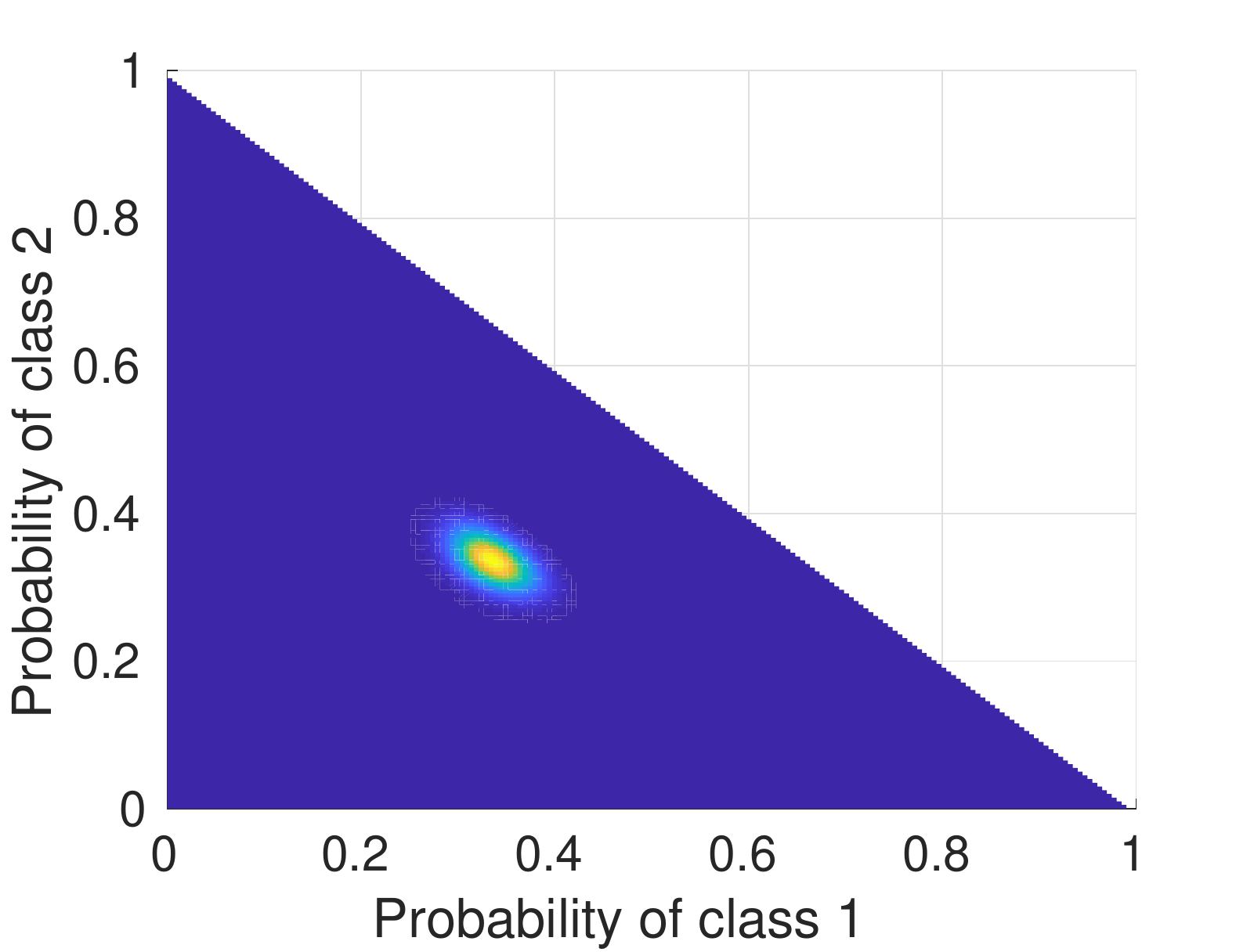}
		\caption{Known-unknown}\label{fig:Fig_Known_Unknown}
	\end{subfigure}
	\begin{subfigure}[b]{0.23\textwidth}
		\includegraphics[width=\textwidth]{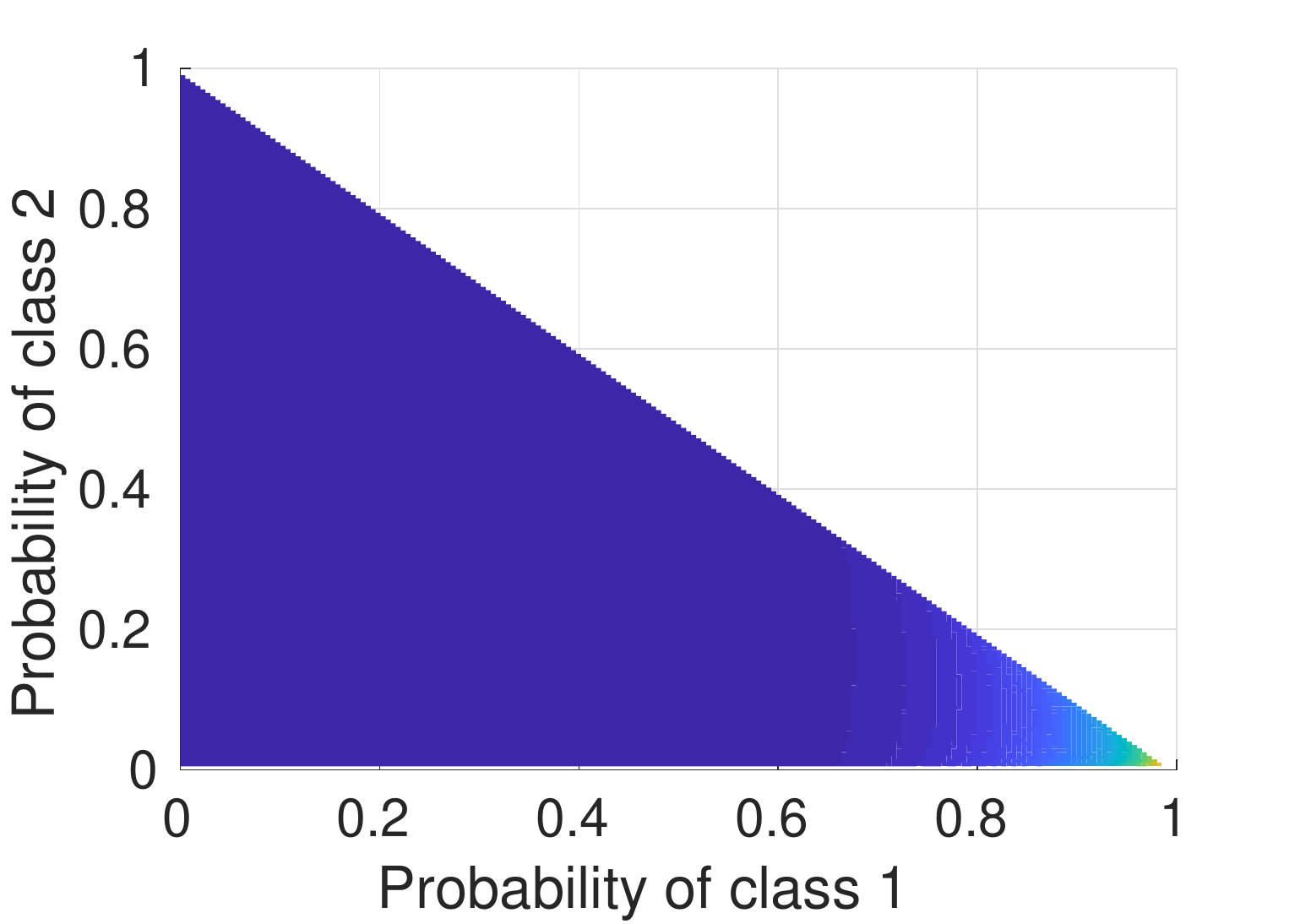}
		\caption{Known-known}\label{fig:Fig_Known_Known}
	\end{subfigure}
	\begin{subfigure}[b]{0.23\textwidth}
		\includegraphics[width=\textwidth]{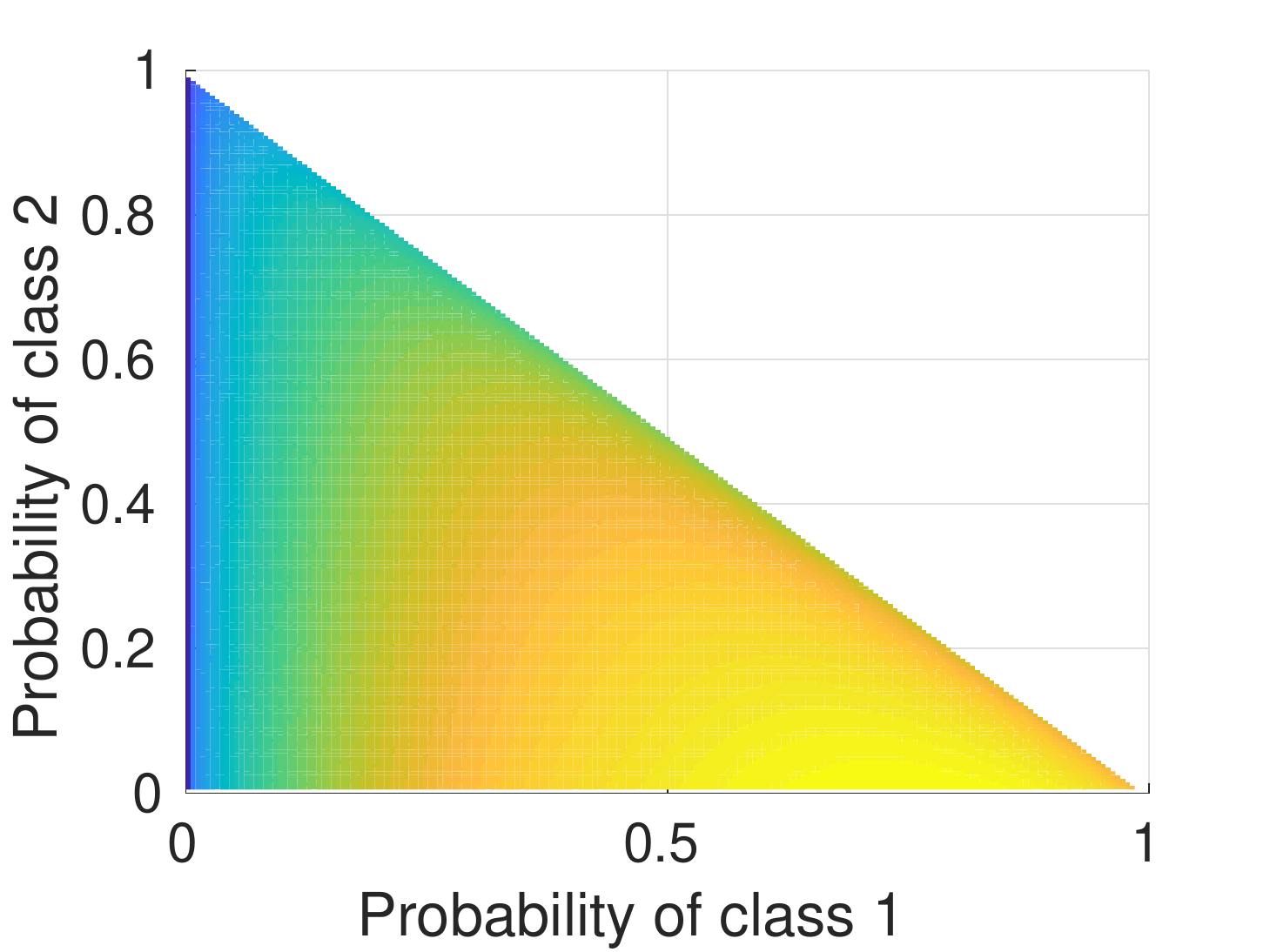}
		\caption{Uncertain classification}\label{fig:Fig_Uncertain_Edge}
	\end{subfigure}
	
	\vspace{-0.2cm}
	\caption[]{The 4 archetypes of $b[\lambda_k]$, shown in a 3 dimensional Dirichlet simplex example, where blue to yellow correspond to low to high probability respectively. \textbf{(a)} is an out-of-distribution setting, where the classifier does not identify the object and the epistemic uncertainty is high. \textbf{(b)} is high data uncertainty setting, which close to $D$, but the class is identifiable in the training set itself. In \textbf{(c)} the classifier recognizes that the object of a certain class with certainty, a scenario we aim for. In \textbf{(d)} the classifier gives preference to one of the classes, but with a high degree of uncertainty.}
	\label{fig:simplex_figures}
\end{figure*} 

As we shall see, this belief can be used within belief space planning, e.g. going to relative poses where the epistemic uncertainty is the smallest to safely classify objects, or vice-versa going to relative poses with high epistemic uncertainty to potentially learn a model online.

More generally $\lambda_k$ is coupled with object and camera poses $\poses_k$. A joint belief over $\lambda_k$ and $\poses_k$ can be maintained, denoted as: 
\begin{equation}\label{eq:Belief_Lambda_Pose_Def}
	b[\lambda_k,\poses_k] \triangleq \prob{\lambda_k,\poses_k|I_{1:k},\his^g_k,D}.
\end{equation}
We may require to compute $\mathbb{E}(\lambda_k)$ to, for example, compute a reward function that depends on $\mathbb{E}(\lambda_k)$. Consider that $\prob{c=i|\lambda_k,\poses_k} = \lambda_k^i$, then for every object class $c=i$:
\begin{equation}\label{eq:PosteriorClass}
	\prob{c=i|I_{1:k},Z^g_k,D} = \int_{\lambda_k,\poses_k} \lambda_k^i  \cdot
	b[\lambda_k,\poses_k] d\lambda_k d\poses_k = \mathbb{E}(\lambda_k^i).
\end{equation}
In \cite{Tchuiev18ral} we presented an approach to maintain $b[\lambda_k]$ in a setting with a single object, didn't consider the coupling between $\poses_k$ and $\lambda_k$, and the approach was limited to inference. On the other hand, here we account for the coupling between $\lambda_k$ and $\poses_k$, present an active approach, and expand to a multi-object setting. First we consider the formulation for a single object. In the approach sections \ref{sec:approach-inference} and \ref{sec:approach-planning} we extend the formulation to the multiple object case, with each method having its specific notations.

%----------------------------------------------------------
\subsection{Belief Space Planning (BSP)}\label{sec:Prelim_BSP}

Given a general current belief $b_k$, one can reason about the best future action from a set of action to maximize (or minimize) an object function. With $b_k$ and a set of future actions $a_{k:k+L}$, it is common to define the objective function as the expected cumulative reward,
%expectation over reward $r(\cdot)$ as an objective function: 
%
\begin{equation}\label{eq:Obj_Func_Max}
	J(b_k,a_{k:k+L}) = \mathbb{E}_{\mathcal{Z}_{k+1:k+L}}(\sum_{i=0}^L r(b_{k+i}(\mathcal{Z}_{k+1:k+i})), a_{k+i}),
\end{equation} 
where $r(\cdot)$ is a belief-dependent reward function, and $L$ is the planning horizon. This formulation can be extended to policies as well.

The above equation can also be also written in a recursive form as in, %such that $J(b_k,a_{k:k+L})$ is a function of the future objective function $J(b_k,a_{k:k+L})$:
\begin{equation}\label{eq:Obj_Func_Recursive}
\begin{split}
	J(b_k,a_{k:k+L}) =& \int_{\mathcal{Z}_{k+1}} \prob{\mathcal{Z}_{k+1}|\his_k,a_k} \cdot \\ 
	& \cdot J(b_{k+1},a_{k+1:k+L})  d\mathcal{Z}_{k+1},
\end{split}
\end{equation}
where $b_{k+1} = b_{k+1}(\mathcal{Z}_{k+1})$. The term $\prob{\mathcal{Z}_{k+1}|\his_k,a_k}$ is the measurement likelihood of future measurement history thus far and and $a_k$, and is essential for BSP. In practice, most of the time the integral in Eq.~\eqref{eq:Obj_Func_Recursive} cannot be analytically computed, thus it is approximated in sampled form:
\begin{equation}\label{eq:Obj_Func_Recursive_Sample}
	J(b_k,a_{k:k+L}) \approx \frac{1}{N_z} \sum_{\mathcal{Z}_{k+1}} J(b_{k+1}(\mathcal{Z}_{k+1}),a_{k+1:k+L}),
\end{equation}
where $N_z$ is the number of $Z_{k+1}$ samples, and $Z_{k+1} \sim \prob{Z_{k+1}|\his_k,a_k}$.

The optimal action sequence $a^*_{k:k+L}$ is chosen such that it maximizes the objective function:
\begin{equation}
	a^*_{k:k+L} = \argmax_{a_{k:k+L}} \left( J(b_k,a_{k:k+L}) \right).
\end{equation}
To evaluate the optimal action sequence, one must consider all possible sequences (possibly via search algorithms) and select the one that produces the highest objective function.

Specifically, in this paper  we consider the belief $b_k = b[\lambda_k,\poses_k]$ for BSP, and discuss planning using various reward functions, while focusing on classifier epistemic uncertainty reward function, namely the entropy of $b[\lambda_{k+i}]$ for a future time $k+i$. Yet, first, we must address the corresponding inference problem.

%--------------------------------------------------------------------------------------
\subsection{Problem Formulation}\label{sec:Problem_Formulation}

Given geometric measurement history $\his^g_k$, an image sequence $I_{1:k}$, actions $a_{0;k-1}$, an epistemic-uncertainty-aware classifier trained on training dataset $D$ with a set $W$ of weight realizations $w \in W$, the problems of inference and planning are defined as follows:

\begin{enumerate}
	
	\item \emph{Inference:} Infer the posterior joint belief $b[\lambda_{k}, \poses_k]$, as defined in Eq.~\eqref{eq:Belief_Lambda_Pose_Def}.
	
	\item \emph{Planning:} Given $b[\lambda_{k}, \poses_k]$, find the future action sequence $a^*_{k:k+L}$ that maximizes $J(b[\lambda_k, \poses_k], a_{k:k+L})$ with the reward function $r(b[\lambda_k, \poses_k])$.
	
\end{enumerate}
In this paper we address the inference and planning problems in Sections \ref{sec:approach-inference} and \ref{sec:approach-planning}, respectively.%, first for a single object, and then considering multiple objects.

%In our approach, we address the inference problem first for a single object, then expand it to multiple objects. The planning problem, in a direct continuation of the inference approaches, considers multiple objects.

%\VI{[Revisit above. Consider formulating as Problem 1 and 2.]}\VT{[Rewritten, still conflicted how to address the single/multiple object formulation]}

\section{Approach Overview}\label{sec:ApproachOverview}

Two approaches are presented for solving each of the problems presented in Sec.~\ref{sec:Problem_Formulation}. The first approach is Multi-Hybrid (\mh), a particle-based approach where multiple hybrid beliefs are maintained simultaneously. The second approach is Joint Lambda Pose (\jlp), where a single continuous belief is maintained and the posterior class probabilities are states within this belief. 

The approach sections are divided to inference and planning; Starting with inference,
we introduce the viewpoint dependent classifier uncertainty model, which predicts the distribution of the classifier output and is used by both methods for inference and planning. In particular, the classifier uncertainty model is used to generate predicted measurements during planning. Then, we introduce \mh and \jlp for inference; First, for simplicity we consider the single object case, afterwards the formulation is expanded to multiple objects. The section concludes with a computation complexity analysis and comparison.

Section \ref{sec:approach-planning} addresses the planning problem; First we discuss measurement generation in general for a single object, then delve into the specifics of both \mh and \jlp of generating measurements for multiple objects. Afterwards we discuss reward functions, and specifically expand upon information-theoretic reward for $b[\lambda]$. Finally, we discuss Dirichlet distribution and LG as possible distributions of $\lambda$ when using \mh (\jlp is limited to LG).

%\VI{[consider adding a diagram that highlights the aspects you consider; maybe better also to have an approach overview section - can decide later]} \VT{[I think that the approach overview overlaps with the paragraph in the introduction that talks about, I prefer to have a diagram at least for now, and refer to it in the problem formulation section instead of open a new (and probably redundant) one]}

	\section{Epistemic Uncertainty Aware Inference}
	\label{sec:approach-inference}
	%\input{03-Approach}
	%\input{03-Approach-Concept}
	%Two approaches are presented. The first one is Multi-Hybrid, a general approach without assumptions on the classifier uncertainty model that relies on maintaining multiple hybrid beliefs simultaneously. The second approach is Joint Lambda Belief; Given some assumptions on the measurements and classifier uncertainty model, we can maintain a single continuous belief that in most cases is faster computationally.

%---------------------------------------------------------
\subsection{Viewpoint Dependent Classifier Uncertainty Model}\label{sec:VDCM}

We use a classifier uncertainty model that accounts both for the coupling between localization and classification, and epistemic model uncertainty. As an example, Fig.~\ref{fig:SUV_demo} illustrates that $\{\gamma_k\}$ measurements varies across different viewpoints, with some containing high epistemic uncertainty and some low. The model we propose learns to predict these measurements, and subsequently which viewpoints will contain high epistemic uncertainty. In contrast, previous works that used a viewpoint dependent classifier model (e.g. \cite{Tchuiev19iros,Tchuiev20ral,Teacy15aamas,Atanasov14tro}) did not consider epistemic uncertainty while learning the model.

The conditions for $\gamma_k$ being a probability vector must be considered when one requires to sample from the classifier model, thus unlike previous works \cite{Tchuiev19iros, Tchuiev20ral} we cannot use a Gaussian distributed classifier model. One possible solution is to consider the classifier model as Dirichlet distributed (see \cite{Tchuiev18ral}), but that model cannot be incorporated into a Gaussian optimization framework (e.g. iSAM2 \cite{Kaess12ijrr}) with unknown poses which are coupled with classification results. Instead, we consider the following solution: we  use a logit transformation for $\gamma$ to a vector $\loggamma \in \mathbb{R}^{m-1}$ space, such that the support of each element $(-\infty,\infty)$:
\begin{equation}\label{eq:Logit_transformation}
\loggamma \triangleq \left[ \log \left( \frac{\gamma^1}{\gamma^m} \right),
\log \left( \frac{\gamma^2}{\gamma^m} \right),...,
\log \left( \frac{\gamma^{m-1}}{\gamma^m} \right) \right]^T.
\end{equation}
Then, $\loggamma$ can be assumed Gaussian such that:
\begin{equation}\label{eq:Classifier_Uncertainty_Model}
\prob{\loggamma_k|c,x^o,x_k} = \mathcal{N}(h_c(x^o,x_k),\Sigma_c(x^o,x_k)),
\end{equation}
and as a consequence $\gamma_k$ is distributed Logistical Gaussian with parameters $\{ h_c, \Sigma_c \}$. The probability density function (PDF) of $\gamma_k$ is as follows:
\begin{equation}\label{eq:Logistical_Gaussian_PDF}
\prob{\gamma_k|c,x^o,x_k} = \frac{1}{\sqrt{|2\pi \Sigma_c|}} \cdot
{\frac{1}{\prod_{i=1}^m \gamma_k^i}} \cdot
e^{\left(-\frac{1}{2}||\loggamma_k - h_c||^2_{\Sigma_c}\right)}.
\end{equation}
In practice, a classifier provides us with a cloud $\{\gamma_k \}$, and each $\gamma_k \in \{\gamma_k \}$ is transformed to $\loggamma_k$. There are $m$ such models, one for each class. The training set consists of tuples of relative pose and $\loggamma$ point clouds such that $D_{cm} \triangleq \{ x^{rel} , \{ \loggamma \} \}$ for each class, where $x^{rel} \triangleq x^o \ominus x$ is the relative pose between object and robot; the expectation (classification scores) and covariance (epistemic uncertainty) is extracted from $\{ \loggamma \}$ and fitted as known points either in the model using e.g. Gaussian Processes or deep-learning based approaches. Real-life application may require creating $D_{cm}$ from multiple different instances of the same objects, e.g. for class "car" multiple types of cars may be used. Fig.~\ref{fig:Model_demo} illustrates the training data shown in black dots versus the trained model shown in blue. The model attempts to "predict" the epistemic uncertainty based on a given training set.

\begin{figure}[!htbp]
	
	\begin{subfigure}[b]{0.45\textwidth}
		\includegraphics[width=\textwidth]{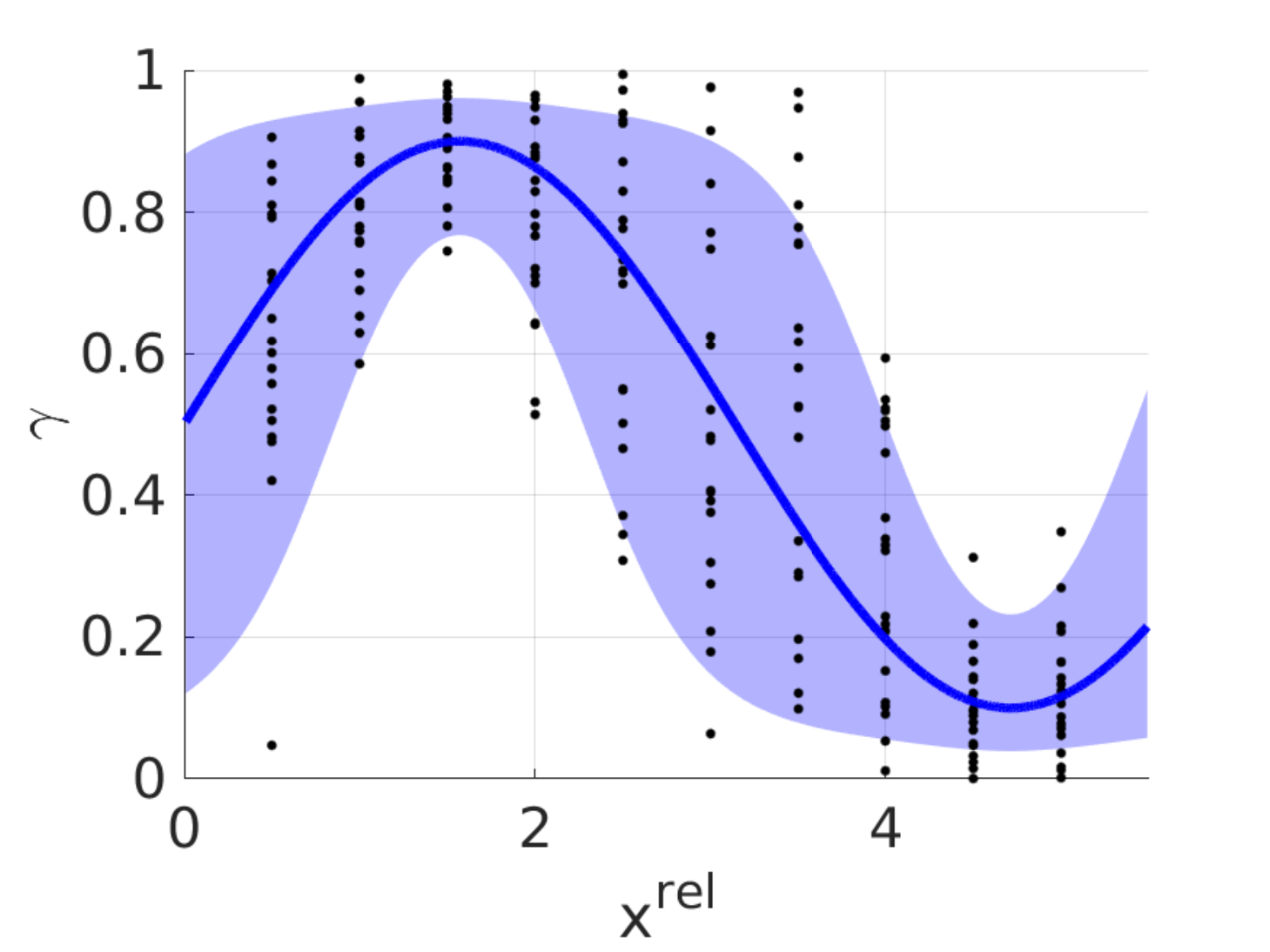}
		\caption{$\gamma$ space}\label{fig:Model_Illustration}
	\end{subfigure}
	\begin{subfigure}[b]{0.45\textwidth}
		\includegraphics[width=\textwidth]{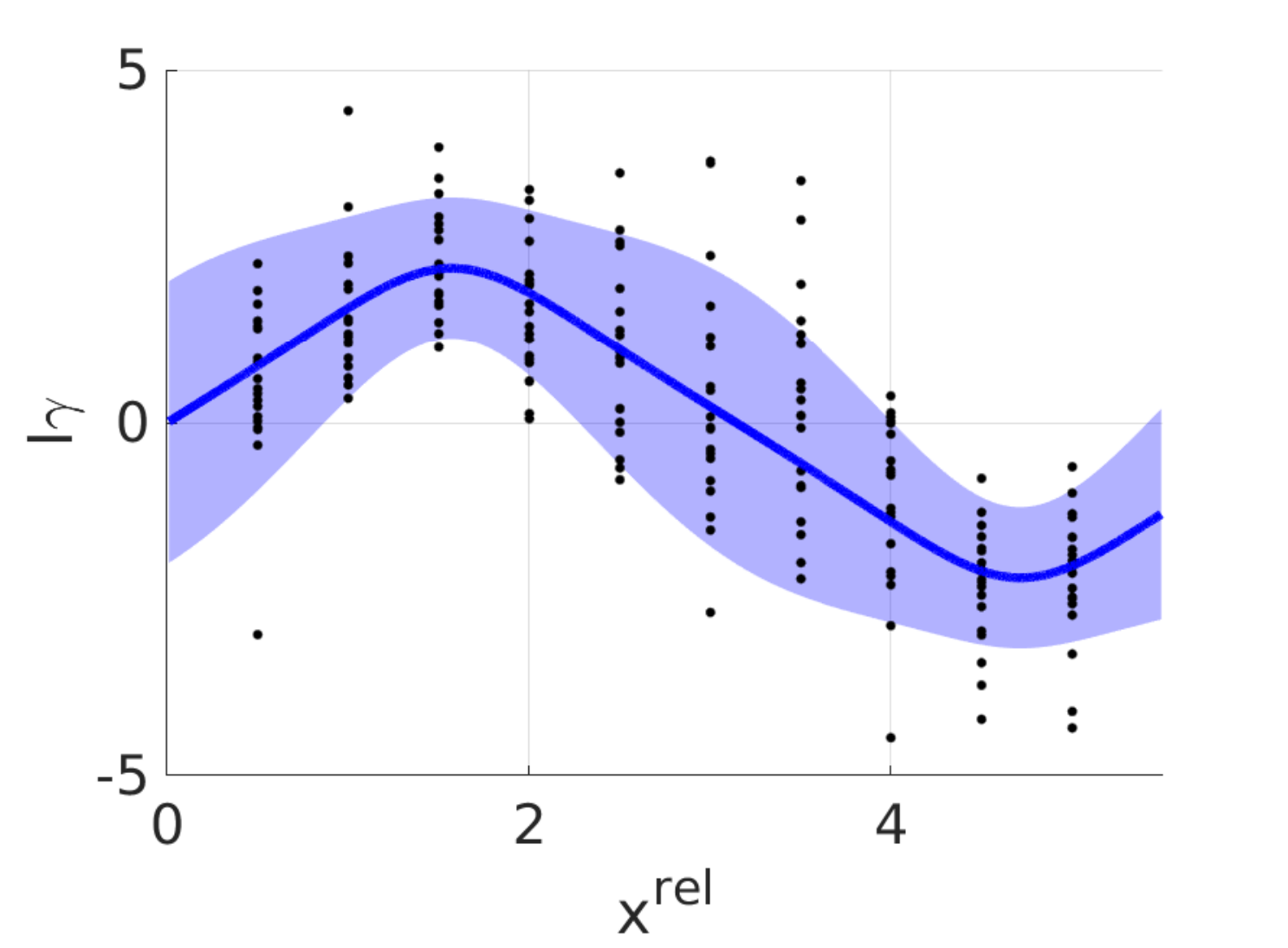}
		\caption{$\loggamma$ space}\label{fig:Model_Illustration_LG}
	\end{subfigure}
	
	\caption{Simplified illustration of the classifier uncertainty model we use in the paper. \textbf{(a)} and \textbf{(b)} represent $\gamma$ and $\loggamma$ space respectively. The black dots represent the corresponding $\gamma(x^{rel},w) \in \{ \gamma \}(x^{rel})$ and $\loggamma(x^{rel},w) \in \{ \loggamma \}(x^{rel})$. The expectation and covariance are learned in \textbf{(b)} and interpolated for new queries of $x^{rel}$, potentially returning to \textbf{(a)} via the inverse logit transformation. The expectation is represented in dark blue, while the one sigma covariance is represented in light blue.}
	
	\label{fig:Model_demo}
\end{figure}

% ---------------------------------------------------------------
\subsection{Multi-Hybrid Inference}\label{subsec:MH-Inference}

In this section we present the Multi-Hybrid (\mh) inference approach to maintain the belief from $b[\lambda_k, \poses_k]$ \eqref{eq:Belief_Lambda_Pose_Def}. With this method, we maintain $b[\lambda_k, \poses_k]$ indirectly via a set of hybrid beliefs, each for a realization of classifier weights. The posterior class probabilities from each hybrid belief together represent the posterior classifier epistemic uncertainty. From there, we can compute marginal distributions for both $\lambda_k$ and $\poses_k$ if needed, e.g.~when computing reward functions for planning (see Section \ref{sec:approach-planning}). We present this approach first when observing a single object, then we extend it to multiple objects, and finally address computational complexity aspects.

\subsubsection{Single Object}\label{sec:MH_single_lambda_inf}

The belief $b[\lambda_{k}, \poses_k]$, as defined by Eq.~\eqref{eq:Belief_Lambda_Pose_Def}, is conditioned both on geometric measurements $\his^g_k$ and raw images $I_{1:k}$, and the classifier training set $D$. As discussed in Sec.~\ref{sec:Gamma_Lambda}, an epistemic uncertainty aware classifier provides us a cloud $\{ \loggamma_{1:k} \}$, which we consider as semantic measurements. In the \mh approach, we  maintain $b[\lambda_{k}, \poses_k]$ by splitting it to components by marginalizing over object classes and classifier weight realization $w \in W$, where $W$ is a predetermined discrete set of $w$ that are used throughout the entire scenario. First, we marginalize $b[\lambda_k,\poses_k]$ over $w$:
\begin{equation}
\begin{split}
	b[\lambda_k,\poses_k] & = 
	\int_w \mathbb{P}(\poses_k,\lambda_k|I_{1:k},\his^g_k,w) \cdot
	\mathbb{P}(w|D) dw \\ &
	\approx \frac{1}{|W|} \sum_w \mathbb{P}(\poses_k,\lambda_k|I_{1:k},\his^g_k,w).
\end{split}
\end{equation}
Then, using chain rule yields%we separate between $\poses_k$ and $\lambda_k$:
\begin{equation}\label{eq:MH_dev_1}
	b[\lambda_k,\poses_k] \approx \frac{1}{|W|} \sum_w
	\prob{\poses_k|\lambda_k,w,I_{1:k},\his^g_k} \cdot
	\prob{\lambda_k|w,I_{1:k},\his^g_k}.
\end{equation}
Each term in the right-hand side of the above is addressed separately; $\prob{\poses_k|\lambda_k,w,I_{1:k},\his^g_k}$ is marginalized over $c$ and using chain-rule can be split into the following distributions:
\begin{equation}\label{eq:MH_conditioned_poses}
\begin{split}
	\prob{\poses_k|\lambda_k,w,I_{1:k},\his^g_k} = & \sum_c \prob{\poses_k|c,\lambda_k,w,I_{1:k},\his^g_k} \\ & \cdot 
	\prob{c|\lambda_k,w,I_{1:k},\his^g_k}
\end{split}
\end{equation}
$\poses_k$ is conditioned on $c$, thus $\lambda_k$ can be omitted. For $c$, given the posterior probability vector $\lambda_k$, the rest can be omitted, and $\prob{c|\lambda_k} = \lambda^c_k$ where $\lambda^c_k$ is the element $c$ of $\lambda_k$. 

$\lambda_k$ is a function of $w$, $I_{1:k}$, and $\his^g_k$; therefore, $\prob{\lambda_k|w,I_{1:k},\his^g_k}$ is a Dirac function $\delta(\cdot)$, such that 
\begin{equation}
	\prob{\lambda_k|w,I_{1:k},\his^g_k} = \delta(\lambda_k - \lambda_{k,w}).
\end{equation}

As such, Eq.~\eqref{eq:MH_dev_1} is rewritten as:
\begin{equation}\label{eq:MH_Breakup}
\begin{split}
	b[\lambda_k,\poses_k] \approx \frac{1}{|W|} \sum_c \sum_w \underbrace{\prob{\poses_k | c,\loggamma_{1:k,w},\his^g_k}}_{b^c_w[\poses_k]} \cdot
	\lambda^c_k\cdot \delta(\lambda_k - \lambda_{k,w}),
\end{split}
\end{equation}
where $b^c_w[\poses_k]$ is the continuous belief conditioned on $w$ and $c$.
Each $w \in W$ is constant throughout the scenario with the reasoning of keeping the number of particles constant, thus avoiding  managing an exponentially increasing number of components (such as in \cite{Tchuiev18ral}). This can be achieved by, e.g.,  training multiple models on the same dataset via bootstrapping, or re-using classifier weight sets created by MC-dropout. That way, Eq.~\eqref{eq:MH_Breakup} shows that maintaining $b[\lambda_k,\poses_k]$ is equivalent to maintaining $b^c_w[\poses_k]$ and $\lambda^c_k$ for all $w \in W$ and $c$.	

Each class probability $\lambda_{k,w}^c$ within the particle $\lambda_{k,w}$ is updated using Bayes rule as follows:
\begin{equation}\label{eq:MH_lambda_update_single}
\lambda_{k,w}^c = \eta \cdot \lambda_{k-1,w}^c \cdot \prob{z^g_k,\loggamma_{k,w}|c},
\end{equation}
where  $\eta$ is a normalizing constant such that $\sum_c \lambda^c_k = 1$, and does not affect inference. As our classifier model is viewpoint dependent (Eq.~\eqref{eq:Classifier_Uncertainty_Model}), we must marginalize $\prob{z^g_k,\loggamma_{k,w}|c}$ over $x^o$ and $x_k$ to fully utilize our models as follows:
\begin{equation}\label{eq:Hybrid_Disc}
\lambda_{k,w}^c \propto \lambda_{k-1,w}^c \int_{x^o,x_k} \likelihoods_{k} \cdot  b^{c-}_w[x^o,x_k] dx^o d x_k,
\end{equation} 
where $b^{c-}_w[x^o,x_k]$ is the propagated conditional continuous belief, constructed as follows, as we marginalize out all other variables from $\poses_k$ beside $x^o$ and $x_k$: 
\begin{equation}
b^{c-}_w[x^o, x_k] \triangleq \int_{\poses_k / x^o, x_k} \motionmodel_k \cdot
b^c_w[\poses_{k-1}] d(\poses_k / x^o, x_k).
\end{equation}
$b^c_w[\poses_k]$ from \eqref{eq:MH_Breakup} is incrementally updated using  standard SLAM state of the art approaches (e.g. iSAM2 \cite{Kaess12ijrr}):
\begin{equation}\label{eq:Hybrid_Cont}
b^c_{w}[\poses_{k}] \propto b^c_w[\poses_{k-1}] \cdot \motionmodel_{k} \cdot \likelihoods_{k}.
\end{equation}
Essentially, for every $w$, we maintain a hybrid belief over robot and object poses, and classes, which we define as:
\begin{equation}\label{eq:hb_def}
	hb_w[\poses_x,c] \triangleq 	b_w^c[\poses_k] \cdot \lambda^c_{k,w},
\end{equation}
and using the above definition, and considering that the Dirac function only "blocks" all $\lambda_k$ except for $\lambda_{k,w}$, we can rewrite Eq.~\eqref{eq:MH_Breakup} in terms of $hb_w[\poses_x,c]$:
\begin{equation}\label{eq:MH_hb_breakup}
	b[\lambda_k,\poses_k] \approx \frac{1}{|W|} \sum_c \sum_w hb_w[\poses_k,c] \cdot \delta(\lambda_k - \lambda_{k,w}).
\end{equation}
Practically, for every $w \in W$, we maintain $b^c_w[\poses_k]$ with the accompanying $\lambda^c_{k,w}$ for every object class realization, overall maintaining $|W|$ hybrid beliefs $hb_w[\poses_k,c]$ in parallel.

Further, one may require to infer the marginals $b[\lambda_k]$ or $\prob{\poses_{k}|I_{1:k},\his^g_{k},D}$, e.g.~to compute an appropriate reward function, as we shall see in Section \ref{sec:approach-planning}. We can describe $b[\lambda_k]$ in term of $\lambda_{k,w}$ particles by marginalizing $b[\lambda_k]$ over $w$:
\begin{equation}\label{eq:MH_Lambda_Marginal}
\begin{split}
	b[\lambda_k] & \approx \frac{1}{|W|} \sum_w \prob{\lambda_k|w,I_{1:k},\his^g_k} \\ &
	= \frac{1}{|W|} \sum_w \prob{\lambda_k|\loggamma_{1:k,w},\his^g_k} \\ &
	= \frac{1}{|W|} \sum_w \delta(\lambda_k - \lambda_{k,w}).
\end{split}
\end{equation}
On the other hand, to compute $\prob{\poses_{k}|I_{1:k},\his^g_{k},D}$, we marginalize over $w$ and $c$,  
\begin{equation}\label{eq:MH_Pose_Marginal}
\prob{\poses_k|I_{1:k},\his^g_k,D} \approx \frac{1}{|W|} \sum_w \sum_c hb_w[\poses_k,c],
\end{equation}
utilizing the already-calculated individual hybrid beliefs $hb_w[\poses_k,c]$.

While theoretically this kind of maintenance is computationally expensive, in practice many class realizations can be with probability close to zero, allowing us to prune  $b_w^c[\poses_k]$ with its conditional $\lambda_{k,w}^c$ if needed (see e.g. \cite{Singh09ieee}). In this paper we set a fixed lower limit on $\lambda^c_{k,w}$ and remove the corresponding component if the value of $\lambda^c_{k,w}$ is lower than said limit. In total, $|W|$ hybrid beliefs are maintained to infer $\lambda_{k,w}$ for each $w$.

%---------------------------------------------------------------------
\subsubsection{Multiple Objects}\label{sec:MH_Multi}

We now extend our formulation to consider the environment includes multiple objects observed by the robot.  Let us introduce some notations to support this extension. 
First, we denote variables corresponding to object $o$ with a superscript $\square^o$. At time $k$ a robot may observe a subset of $n_k$ objects within the environment, and up until time $k$, $N_k$ objects. The subset of $n_k$ objects is denoted as $O_k$. Each object is segmented from the image and the classifier outputs $\{ \gamma^o_{k,w} \}_{w \in W}$ corresponding to said object, and the set of all those clouds for all $n_k$ objects in $O_k$ is denoted as $\{ \Gamma_k \}$ with $\Gamma$ defined as a realization of $\gamma$ measurements, one per each observation. For a specific $w \in W$, we define the realization of $\gamma$ measurements as $\Gamma_{k,w} \triangleq \{ \gamma^o_{k,w} \}_{o \in O_k}$, thus $\{ \Gamma_k \} \triangleq \{ \Gamma_{k,w} \}_{w \in W}$. We define $\logGamma_k$ as the logit transformation of all $\gamma_k \in \Gamma_k$ as in Eq.~\eqref{eq:Logit_transformation}.  The set of all geometric measurements at time $k$ is denoted $Z^g_k$, the history $\his_k^g \triangleq \{a_{0:k-1}, Z^g_k\}$ includes all geometric measurements and actions up until time $k$, and subsequently $\his_k \triangleq \{I_{1:k},\his^g_k\}$ includes all measurement and action history up to time $k$.

We define the joint posterior class probability vector as: 
\begin{equation}\label{eq:MH_big_lambda_def}
	\Lambda_k \triangleq \prob{C|\logGamma_{1:k},\his^g_k},
\end{equation}
where $C \triangleq \{ c^o \}_{o \in O_{1:k}}$ is the class realization of all objects observed up to time $k$, with $c^o$ being the $o$-th object class.
In addition, we include in $\poses_k$ the poses of all the objects, such that $\poses_k \triangleq x_{0:k} \cup \{ x^o \}_{o \in O_{1:k}}$. Subsequently, the belief over $\Lambda_k$ and $\poses_k$ is:
\begin{equation}
	b[\Lambda_k,\poses_k] \triangleq \prob{\Lambda_k,\poses_k|I_{1:k},\his^g_k,D}.
\end{equation}
Observe that $\Lambda_k$ is still a probability vector, but with $m^{N_k}$ possible categories.
To illustrate this, consider an example with two objects and three candidate classes, i.e.~$O_{1:k}=\{o, o'\}$ and $m=3$. Then each category contains a class hypothesis for all object classes, e.g. $c^o = 1$, $c^{o'} = 3$. As such, there are 9 possible class realizations and therefore $\Lambda_k$ has 9 categories whose probabilities should sum to one.
That way, the number of categories in $\Lambda_k$ grows exponentially with the number of objects, potentially to intractable levels. Fortunately, this can be mitigated by pruning components with low probability, as was done in \cite{Tchuiev19iros, Tchuiev20ral}.

For every $w \in W$, updating $\Lambda_k$ is largely similar to updating $\lambda_k$ in Sec.~\ref{sec:MH_single_lambda_inf}, except for a few differences. The likelihood terms include all objects $O_k$ observed at time $k$, and the conditional probability over the poses is conditioned on class realization $C$ instead of the class of a single object,
\begin{equation}\label{eq:MH_Multi_Obj_Lambda_Upd}
	\Lambda^C_{k,w} \propto \Lambda^C_{k-1,w} \int_{x^{inv}_{k}} \mathcal{L}_{k} \cdot 
	b^{C-}_w[\poses^{inv}_{k}] dx^{inv}_{k},
\end{equation}
where $\Lambda^C_{k,w}$ denotes the posterior probability of class realization $C$ at time $k$ for weight realization $w$, and $\poses^{inv}_{k}$ represents the last robot pose and all poses of objects observed at time $k$, i.e. $\poses^{inv}_{k} \triangleq x_{k} \cup \{ x^o \}_{o \in O_{k}}$. Likelihood $\mathcal{L}_{k}$ now encompasses all the measurement likelihoods of all objects as follows:
\begin{equation}
	\mathcal{L}_{k} \triangleq \prod_{o \in O_{k,o}} \prob{\loggamma_{k}^o|C,x^o,x_{k}} \cdot 
	\prob{z^g_{k}|x^o,x_{k}}.
\end{equation}
The belief $b^{C-}_w[\poses^{inv}_{k}]$ is the propagated belief conditioned on $C$, and marginalized over the uninvolved variables such that $\poses_{k} = \poses^{inv}_{k} \cup \poses^{\neg inv}_{k}$: 
\begin{equation}\label{eq:Neginv_Def}
	b^{C-}_w[\poses^{inv}_{k}] = \int_{\poses^{\neg inv}_{k}} \motionmodel_{k} \cdot b_w^C[\poses_{k-1}]d\poses^{\neg inv}_{k}.
\end{equation}
Similarly to Sec.~\ref{sec:MH_single_lambda_inf} we can rewrite $b[\Lambda_k,\poses_k]$ as:
\begin{equation}\label{eq:MH_breakup_multi}
\begin{split}
	b[\Lambda_k,\poses_k] \approx \frac{1}{|W|} \sum_w \sum_C 	
	hb_w[\poses_k,C] \cdot \delta(\Lambda_k - \Lambda_{k,w}),
\end{split}
\end{equation}
where $hb_w[\poses_k,C]$ is the hybrid belief conditioned on $w$:
\begin{equation}\label{eq:MH_hb_def_multi}
hb_w[\poses_k,C] \triangleq b^C_w[\poses_k] \cdot \Lambda^C_{k,w},
\end{equation}
and $b^C_w[\poses_k] \triangleq \prob{\poses_k|c,\loggamma_{1:k,w},\his^g_k}$.
Similarly to Eq.~\eqref{eq:MH_hb_breakup}, maintaining $b[\Lambda_k,\poses_k]$ is equivalent to maintaining $hb_w[\poses_k,C]$ for all $w \in W$ and $C$. In case $b[\Lambda_k]$ is required, for the multi-object case Eq.~\eqref{eq:MH_Lambda_Marginal} becomes:
\begin{equation}
	b[\Lambda_k] \approx \frac{1}{|W|} \sum_w \delta(\Lambda_k = \Lambda_{k,w});
\end{equation}
Similarly; In case $\prob{\poses_k|I_{1:k},\his^g_k,D}$ is required, for the multi-object case Eq.~\eqref{eq:MH_Pose_Marginal} becomes: 
\begin{equation}\label{eq:MH_Pose_Marginal_Multi}
	\prob{\poses_k|I_{1:k},\his^g_k,D} \approx \frac{1}{|W|}. \sum_w \sum_C hb_w[\poses_k,C].
\end{equation}
In general, all $\poses_k$ and $C$ are coupled, and subsequently so do $\poses_k$ and $\Lambda_k$.
There are two possible sources of coupling: class priors that depend on other objects' classes (e.g. a computer mouse may be expected to appear next to a monitor), and the coupling between poses and classes induced by the viewpoint-dependent classifier uncertainty model \eqref{eq:Classifier_Uncertainty_Model}. 

Specifically, if the classifier model is not viewpoint dependent, i.e. $\prob{\loggamma_k|c,\poses_k} = \prob{\loggamma_k|c}$, then $\prob{\poses_k,c|\his_k} = \prob{\poses_k|\his_k} \cdot \prob{c|\his_k}$, and each one can be maintained separately. This simplified case can be represented as a single factor graph for the continuous variables, and the discrete variables are maintained via $\prob{c|\his_k} \propto \prob{c|\his_{k-1}} \cdot \prob{\loggamma_k|c}$. 

However, in our case, the viewpoint-dependent model  $\prob{\loggamma_k|c,\poses_k}$ couples between relevant continuous and discrete variables; specifically, it is represented as a factor between robot and object poses  at time steps when the object is observed.  Thus, for  $\loggamma_k$ that corresponds to a semantic observation of some object $o$ at time instant $k$, the factor is $\prob{\loggamma_k|c,x_k, x^o}$, and, according to \eqref{eq:Classifier_Uncertainty_Model}, it differs for each class realization $c$. This is represented by multiple factor graphs as illustrated in a simple example in Fig.~\ref{fig:FG_MH_Single}. Between the graphs, the topology is identical, but the factor  $\prob{\loggamma_k|c,\poses_k}$ changes according to class $c$ hypothesis. Further, each classifier weight $w\in W$ corresponds to its own instance of those factor graphs. Of course, if a given factor graph is connected, all variables in it are coupled. 

\begin{figure}
	\centering
	\begin{subfigure}[b]{0.48\textwidth}
		\centering
		\begin{tikzpicture}[scale=1.0]
		
		% Robot poses
		\node at (0,0) [circle,draw] (x0) {$x_0$};
		\node at (2,0) [circle,draw] (x1) {$x_1$};
		\node at (4,0) [circle,draw] (x2) {$x_2$};
		
		% Objects 
		\node at (3,2) [circle,draw] (o1) {$x^{o}$};
		
		% x0 prior
		\draw[black] (-1,0) -- (x0);
		\draw[fill=black] (-1,0) circle (0.1) node[above right] {$ $};
		
		%Motion models
		{$ $};
		\Edge(x0)(x1)
		\draw[fill=black] (1,0) circle (0.1) node[above right] {$ $};
		\Edge(x1)(x2)
		\draw[fill=black] (3,0) circle (0.1) node[above right] {$ $};
		
		% Connections to objects
		\path[blue, bend left] (x1) edge (o1);
		\path[black] (x1) edge (o1);
		\draw[fill=black] (2.5,1) circle (0.1) node[above right] {$ $};
		\draw[fill=blue] (2,1) circle (0.1) node[above right] {$ $};
		
		\path[blue, bend right] (x2) edge (o1);
		\path[black] (x2) edge (o1);
		\draw[fill=black] (3.5, 1) circle (0.1) node[above right] {$ $};
		\draw[fill=blue] (4, 1) circle (0.1) node[above right] {$ $};

		\end{tikzpicture}
		\caption{$c^{o}=1$}
		\label{fig:FG_MH_1}
	\end{subfigure}
	\begin{subfigure}[b]{0.48\textwidth}
		\centering
		\begin{tikzpicture}[scale=1.0]
		% Robot poses
		\node at (0,0) [circle,draw] (x0) {$x_0$};
		\node at (2,0) [circle,draw] (x1) {$x_1$};
		\node at (4,0) [circle,draw] (x2) {$x_2$};
		
		% Objects 
		\node at (3,2) [circle,draw] (o1) {$x^{o}$};
		
		% x0 prior
		\draw[black] (-1,0) -- (x0);
		\draw[fill=black] (-1,0) circle (0.1) node[above right] {$ $};
		
		%Motion models
		{$ $};
		\Edge(x0)(x1)
		\draw[fill=black] (1,0) circle (0.1) node[above right] {$ $};
		\Edge(x1)(x2)
		\draw[fill=black] (3,0) circle (0.1) node[above right] {$ $};
		
		% Connections to objects
		\path[red, bend left] (x1) edge (o1);
		\path[black] (x1) edge (o1);
		\draw[fill=black] (2.5,1) circle (0.1) node[above right] {$ $};
		\draw[fill=red] (2,1) circle (0.1) node[above right] {$ $};
		
		\path[red, bend right] (x2) edge (o1);
		\path[black] (x2) edge (o1);
		\draw[fill=black] (3.5, 1) circle (0.1) node[above right] {$ $};
		\draw[fill=red] (4, 1) circle (0.1) node[above right] {$ $};
		
		\end{tikzpicture}
		\caption{$c^{o}=2$}
		\label{fig:FG_MH_2}
	\end{subfigure}
	\caption{Factor graphs for a toy scenario where the camera observes an object for a specific $w$,there are $|W|$ such factor graph pairs. The object has two candidate classes. Each dot and line represents a separate factor. The black factors between the camera and object represent the geometric model, while the colored factors represent the classifier models, $c=1$ and $c=2$ by {\color{blue}blue} and {\color{red}red} respectively.}
	\label{fig:FG_MH_Single}
\end{figure}
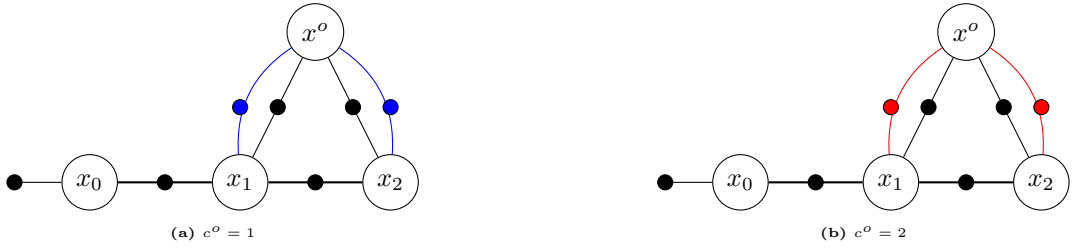

\emph{Remark}: While here we described a straightforward extension of the \mh approach to the multi-object, its worst-case computational complexity (discussed in Section \ref{sec:MH_Inf_Comp}) scales poorly with the number of objects. One could also consider maintaining a marginal distribution for each object (e.g.~$b[\lambda^o_k, x^o]$ ) instead of the joint distribution \eqref{eq:MH_breakup_multi}; yet, without introducing approximations, this would still involve inferring poses via \eqref{eq:MH_Pose_Marginal_Multi}, as all object poses and classes are dependent as discussed in Sec.~\ref{sec:MH_single_lambda_inf}. The marginal poses, i.e. Eq.~\eqref{eq:MH_Pose_Marginal_Multi}, requires the maintenance of all $hb_w[\poses_k,C]$ as maintaining $b[\Lambda_k,\poses_k]$, thus the computational time and memory complexity does not change. Computing a marginal distribution for each object is outside the scope of this paper and might be addressed in future work.

\subsubsection{Computational Complexity and Discussion}\label{sec:MH_Inf_Comp}

With $m$ candidate classes, and $n$ objects, the number of possible class realizations per $\Lambda_{k,w}$ is $m^n$, as $C$ considers all possible class realizations of all objects observed thus far, making $C$ combinatorial in nature. Inference over continuous states ($k$ camera poses and $n$ object poses), i.e. the conditional belief  $\prob{\poses_k|C,\loggamma_{1:k},\his^g_k}$, can be efficiently done using, e.g.~the state of the art iSAM2 approach \cite{Kaess12ijrr}, with a computational complexity  of $O((k+n)^{1.5})$, and at worst $O((k+n)^3)$ for loop closures. To compute an individual $\lambda_k$ particle, we must account for the attached $\prob{\poses_k|C,\loggamma_{1:k},\his^g_k}$, and thus the worst-case computational complexity is $O(m^n (k+n)^3)$. Eventually we maintain $|W|$ particles, and therefore the overall time computational complexity for $b[\lambda_k]$ inference is $O(|W|m^n (k+n)^3)$ at worst without pruning. 

The computational complexity can be further reduced by pruning e.g. low probability classes for individual particles, but as with any pruning this can induce a problem where a certain realization gets "locked" in either probability 0 or 1, rendering the possibility of probability changing for the said realization impossible. 

For memory complexity, we require the latest robot pose, and the $n$ poses and $m$ sized class probability vector for every object per $\lambda$ particle. All in all, we must maintain $O(n m^n)$ random variables in memory per particle, and $O(|W| n m^n)$ variables for $b[\lambda]$. This also can be reduced by e.g. pruning low probability class realizations or incremental inference methods.

To conclude, while accurate,, due to the combinatorial nature of $C$ which considers all possible class realizations, the need to simultaneously maintain $|W|$ hybrid beliefs, the worst-case complexity of \mh scales poorly with number of objects and candidate classes. In practice, pruning class realizations with low probability can reduce computational complexity to manageable levels. Incremental inference approaches for hybrid beliefs, inline with \cite{Hsiao19icra}, could further reduce computational complexity. These, however, are outside the scope of this paper. 

As an alternative, in the next section we propose the \jlp algorithm, which is by far computationally more efficient than \mh.

%----------------------------------------------------------------
\subsection{Joint Lambda Pose  Inference}\label{sec:JLP}

In this subsection we present an alternative approach for inference, which maintains a joint belief over $\poses_k$ and $\lambda_k$. This approach is significantly less computationally expensive than the Multi-Hybrid approach. Its accuracy depends on conditions that we discuss below. This approach is denoted as Joint Lambda Pose (\jlp). Similarly to \mh, we first consider the single object case, and then extend \jlp to the multiple object case.

To the best of our knowledge, there are no approaches that combine Gaussian distributed variables with random variables within a simplex besides sampling based methods. Thus, we cannot maintain $b[\lambda_k,\poses_k]$ as a single continuous belief, e.g. \mh requires maintaining multiple hybrid beliefs, as discussed in Sec.~\ref{subsec:MH-Inference}.%\ref{sec:MH_single_lambda_inf} and \ref{sec:MH_Multi}.

Instead, 
we define $\loglambda_k$ as the logit transformation of $\lambda_k$, and maintain the belief (considering a single object, for now):
\begin{equation}\label{eq:JLP_basic_belief}
	b[\loglambda_k,\poses_k] \triangleq \prob{\loglambda_k,\poses_k|I_{1:k},\his^g_k,D}.
\end{equation}
For a general $\lambda_k$, each $\lambda_k^c$ can be updated using Bayes rule: 
\begin{equation}
	\lambda^c_k = \eta \cdot \lambda^c_{k-1} \cdot \prob{\loggamma_k|c,x^{rel}_k}.
\end{equation}
When $\lambda_k$ is cast into logit space, the above equation transforms into the following sum, written in a vector form:
\begin{equation}\label{eq:l-lambda}
	\loglambda_k = \loglambda_{k-1} + l\mathcal{L}^s_k,
\end{equation}
where for each element in $\lambda_k$, the normalizer $\eta$ gets canceled as it is identical for all elements of $\lambda_k$,
and $l\mathcal{L}^s_k$ is defined as: 
\begin{equation}
\begin{split}
	l\mathcal{L}^s_k \triangleq & \bigg[ \log \left( \frac{\prob{\loggamma_k | c=1,x^{rel}_k}}
	{\prob{\loggamma_k | c=m,x^{rel}_k}} \right) , ... , \\
	& \log \left( \frac{\prob{\loggamma_k | c=m-1,x^{rel}_k}}
	{\prob{\loggamma_k | c=m,x^{rel}_k}} \right) \bigg]^T.
\end{split}
\end{equation}
To recursively update a Gaussian $\loglambda_k$ in closed form from a Gaussian $\loglambda_{k-1}$, $l\mathcal{L}^s_k$ needs to be Gaussian as well. We now discuss conditions for which $l\mathcal{L}^s_k$ is indeed Gaussian.

\subsubsection{Accuracy Conditions}

In this section we analyze the condition under which $l\mathcal{L}^s_k$ is accurately Gaussian distributed. This is formulated in the following Lemma:
\begin{lemma}\label{lemma:FGCondition}
	Given $m$ Gaussian distributed viewpoint-dependent classifier uncertainty models $\prob{\loggamma_k|c,x^{rel}_k}$ as in Eq.~\eqref{eq:Classifier_Uncertainty_Model}, if $\Sigma_{c=i}(x^{rel}_k) \equiv \Sigma_{c=j}(x^{rel}_k) \; \forall i,j \in [1,m]$, then $l\mathcal{L}^s_k$ is Gaussian distributed.
\end{lemma}

\begin{proof}
	In this proof, we will omit time index $k$ and 
	sometimes omit $x^{rel}_k$ from  $h_c$ and $\Sigma_c$ to reduce clutter. % , while remembering that all $h_c$ and $\Sigma_c$ are still dependent on $x^{rel}_k$. 
	We prove  by construction, with writing the PDF of the classifier uncertainty model for the $i$-th element of $l\mathcal{L}^s$. The model for class $i$ has an expectation $h_{c=i}(x^{rel})$ and a covariance matrix $\Sigma_{c=i}(x^{rel})$. Thus: 
	\begin{equation}\label{eq:Condition_Develop_1}
	\begin{split}
	l\mathcal{L}^s_{i} & = \log \left( \frac{(2\pi)^{\frac{m-1}{2}} \sqrt{|\Sigma_{c=m}}|
		e^{ -\frac{1}{2}||\loggamma - h_{c=i} ||^2_{\Sigma_{c=i}}}}
	{(2\pi)^{\frac{m-1}{2}} \sqrt{|\Sigma_{c=i}}|
		e^{ -\frac{1}{2}||\loggamma - h_{c=m} ||^2_{\Sigma_{c=m}}}} \right) \\ 
	& = - \frac{1}{2} \log(|\Sigma_{c=i}|) + \frac{1}{2} \log(|\Sigma_{c=m}|) \\
	& -\frac{1}{2}||\loggamma - h_{c=i} ||^2_{\Sigma_{c=i}} +\frac{1}{2}||\loggamma - h_{c=m} ||^2_{\Sigma_{c=m}}.
	\end{split}
	\end{equation}
	Now, applying the condition $\Sigma_{c=i} (x^{rel}) \equiv \Sigma_{c=j}(x^{rel})$, and denoting both as $\Sigma_c$ we get the following expression:
	\begin{equation}\label{eq:Likelihood_reduced_full_JLP}
	\begin{split}
	l\mathcal{L}^s_{i} = & \loggamma^T \Sigma_{c}^{-1} h_{c=i} -
	\frac{1}{2} h_{c=i}^T \Sigma_{c}^{-1} h_{c=i} \\ & -
	\loggamma^T \Sigma_{c}^{-1} h_{c=m} +
	\frac{1}{2} h_{c=m}^T \Sigma_{c}^{-1} h_{c=m}.
	\end{split}
	\end{equation}
	From the above equation, if $\loggamma$ is a multi-variate Gaussian random variable, then $l\mathcal{L}^s_{i}$ is a linear combination of Gaussian random variables, therefore Gaussian by itself. This is valid for every $i \in [1,m-1]$.
\end{proof}

In general, the classifier model covariance functions may  not be equivalent; Therefore, Eq.~\eqref{eq:Condition_Develop_1} includes a quadratic expression of $\loggamma$, making $l\mathcal{L}^s_{i}$ a mixture of Gaussian and Generalized Chi distributions. To counter this, the models' covariances must be "close" to each other to approximately describe $l\mathcal{L}^s_{i}$ as a Gaussian.

If $l\mathcal{L}^s$ is assumed Gaussian via moment matching or other methods, it will only approximate the true distribution of $l\mathcal{L}^s$ with the accuracy dependent on the "distance" between $\Sigma_{c=i}(x^{rel})$ and $\Sigma_{c=j}(x^{rel})$ for all $i,j \in [1,m]$. This distance can be represented by, for example, Forbenius Norm.

\emph{Remark:}  
%and is defined for given information matrices $\Sigma^{-1}_{c=i}$ and $\Sigma^{-1}_{c=m}$ as:
%%
%\begin{equation*}
%	f_{\Sigma^{-1}_{c=i},\Sigma^{-1}_{c=m}} = Tr \left( (\Sigma^{-1}_{c=i} - \Sigma^{-1}_{c=m}) \cdot
%	(\Sigma^{-1}_{c=i} - \Sigma^{-1}_{c=m})^T \right).
%\end{equation*}
%%
The Forbenius norm can be inserted into the loss function while training the viewpoint-dependent classifier uncertainty models \eqref{eq:Classifier_Uncertainty_Model}, thereby enforcing sufficiently close covariance functions between different models such that the approach presented in this section can be used. Our implementation utilizes this concept, as we further explain in Sec.~\ref{sec:Exp_setup}.
%The Forbenius norm is used in our implementation to train the viewpoint-dependent classifier uncertainty models \eqref{eq:Classifier_Uncertainty_Model} that are used in the experiments reported in this paper.

Having discussed conditions for $l\likelihoods^s$ to be Gaussian (accurately or approximately), in the following section we introduce a new factor, termed joint Lambda pose (\jlp) factor, which constructs $b[\loglambda_k,\poses_k]$.

\subsubsection{Joint Lambda Pose (\jlp) Factor}\label{subsec:JLP-Factor}

Assume the conditions in Lemma \ref{lemma:FGCondition} are satisfied. Considering Eq.~\eqref{eq:Likelihood_reduced_full_JLP} for all $i \in [1,m-1]$, we can describe $l\mathcal{L}^s_k$ as follows:
\begin{equation}\label{eq:phi_expression}
l\mathcal{L}^s_k = \Phi \loggamma_k - \frac{1}{2} \phi,
\end{equation} 
where the matrix $\Phi \in \mathbb{R}^{(m-1) \times (m-1)}$ and the vector $\phi \in \mathbb{R}^{m-1}$ depend on the individual classifier models \eqref{eq:Classifier_Uncertainty_Model} and $x_k^{rel}$. Using Eq.~\eqref{eq:Likelihood_reduced_full_JLP}, the matrix $\Phi$ is defined as
\begin{equation}
	\Phi \triangleq \left[ 
	\begin{array}{c}
	h^T_{c=1} \Sigma^{-1}_{c=1} - h^T_{c=m} \Sigma^{-1}_{c=m} \\
	\vdots \\
	h^T_{c=m-1} \Sigma^{-1}_{c=m-1} - h^T_{c=m} \Sigma^{-1}_{c=m}
	\end{array}
	\right],
\end{equation}
and $\phi$ is defined as
\begin{equation}
	\phi \triangleq \left[
	\begin{array}{c}
	h^T_{c=1} \Sigma^{-1}_{c=1} h_{c=1} - h^T_{c=m} \Sigma^{-1}_{c=m} h_{c=m} \\
	\vdots \\
	h^T_{c=m-1} \Sigma^{-1}_{c=m-1} h_{c=m-1} - h^T_{c=m} \Sigma^{-1}_{c=m} h_{c=m}
	\end{array}
	\right].
\end{equation}
If the conditions of Lemma \ref{lemma:FGCondition} are satisfied, we can substitute $l\likelihoods^s_k$ in Eq.~\eqref{eq:l-lambda} with the expression in Eq.~\eqref{eq:phi_expression}:
\begin{equation}
	\loglambda_k = \loglambda_{k-1} + \Phi \loggamma_k - \frac{1}{2} \phi.
\end{equation}
Now, as $\loggamma_k$ is assumed Gaussian, its distribution is defined by expectation $\mathbb{E}(\loggamma_k)$ and covariance  $\Sigma(\loggamma_k)$. Assuming a non-singular matrix $\Phi$, we define the \jlp factor as:
\begin{equation}\label{eq:JLP_Factor}
\begin{array}{c}
	\prob{\loglambda_k|\loglambda_{k-1},I_k,D,x_k^{rel}} \triangleq \\
	\mathcal{N} \left( \loglambda_{k-1} + \Phi \mathbb{E}(\loggamma_k) - \frac{1}{2} \phi,
	\Phi \Sigma(\loggamma_k) \Phi^T \right),
\end{array}
\end{equation}
As mentioned before, we utilize a classifier that outputs a set $\{ \gamma_k \}$ instead of a single $\gamma_k$. Each $\gamma_k \in \{ \gamma_k \}$ is then transformed via the logit transformation \eqref{eq:Logit_transformation} to $\loglambda_k$, thus the entire set $\{ \gamma_k \}$ is transformed to $\{ \loggamma_k \}$. From there $\mathbb{E}(\loggamma_k)$ and $\Sigma(\loggamma_k)$ are inferred, and the \jlp factor can be written as $\prob{\loglambda_k|\loglambda_{k-1},\{ \loggamma_k \}, x^{rel}_k}$. As in Sec.~\ref{sec:MH_single_lambda_inf}, $\{ \loggamma_k \}$ represents the classifier's epistemic uncertainty.

The factor \eqref{eq:JLP_Factor} is a four variable factor of $\loglambda_k$, $\loglambda_{k-1}$, $x^o$, and $x$, with the latter two used to compute $x^{rel}$ via $x^{rel} \triangleq x^o \ominus x$. The factor can be inserted into a graph structure that can be optimized using standard SLAM methods, where $\loglambda_k$ for different $k$ are separate variable nodes. This factor enables us to maintain $b[\loglambda_k,\poses_k]$ using a single continuous belief as we discuss in the next section, and in turn be faster computationally than \mh.

The term $\Phi \Sigma(\loggamma_k) \Phi^T$ is positive definite when $\Phi$ is not singular, but in practice we cannot guarantee this condition. If there is some $x^{rel}_k$ for classes $c=i$ and $c=j$ where $h_{c=i}(x^{rel}) = h_{c=j}(x^{rel})$ and $\Sigma_{c=i}(x^{rel}) = \Sigma_{c=j}(x^{rel})$, then at that point $\Phi$ is singular. This means: at that certain $x^{rel}_k$, we cannot differentiate between the two classes with the given classifier models. To keep $\Phi \Sigma(\loggamma_k) \Phi^T$ non-singular, we add to it an identity matrix multiplied by a small positive constant $\epsilon \cdot I^{(m-1) \times (m-1)}$. 

\subsubsection{Recursive Update Formulation}

In Section  \ref{subsec:JLP-Factor} we introduced a novel four variable factor\footnote{In case the object is not observed at $k-1$, instead of $\loglambda_{k-1}$ we connect the factor to the latest previous $\loglambda$.} 
$\prob{\loglambda_k|\loglambda_{k-1},I_k,D,x_k^{rel}}$, denoted as the \jlp factor. This factor allows to update $b[\loglambda_{1:k},\poses_k]$ as a single continuous belief, instead of multiple conditioned ones as with \mh. 

We may consider a smoothing formulation where we maintain the joint belief $b[\loglambda_{1:k},\poses_k]$. Using Bayes and chain rules, $b[\loglambda_{1:k},\poses_k]$ can then be updated as:
\begin{equation}\label{JLP:Smoothing}
\begin{split}
	b[\loglambda_{1:k},\poses_k] = & \eta \cdot
	\prob{\loglambda_k|\loglambda_{k-1},I_k,D,x_k^{rel}} \cdot
	\motionmodel_k \cdot \\ & \cdot
	\prob{z^g_k|x_k^{rel}} \cdot
	b[\loglambda_{1:k-1},\poses_{k-1}], 
\end{split}
\end{equation}
where $\eta$ is a normalization constant. 
In practice, previous $\loglambda_{1:k-1}$ are typically not required 
for classification inference and planning, so we can consider the belief $b[\lambda_k,\poses_k]$, without maintaining a large number of states per object.
To update this belief recursively, we must express it as a function of the prior 	$b[\loglambda_{k-1},\poses_{k-1}]$. To do so, we marginalize over $\loglambda_{k-1}$ 
and use the Bayes rule: 
\begin{equation}\label{eq:Lambda_X_inference}
\begin{split}
	b[\loglambda_k,\poses_k] \propto & \int_{\loglambda_{k-1}} 
	\prob{\loglambda_k|\loglambda_{k-1},I_k,D,x_k^{rel}} \cdot
	\motionmodel_k \cdot \\ & \cdot
	\prob{z^g_k|x_k^{rel}} \cdot
	b[\loglambda_{k-1},\poses_{k-1}] d\loglambda_{k-1}.
\end{split}
\end{equation}
Fig.~\ref{fig:Factor_JLP_Single} presents a simple example to illustrate the factor graph structure using \jlp. In this figure, we present a scenario with two time steps in which the robot observes a single object.

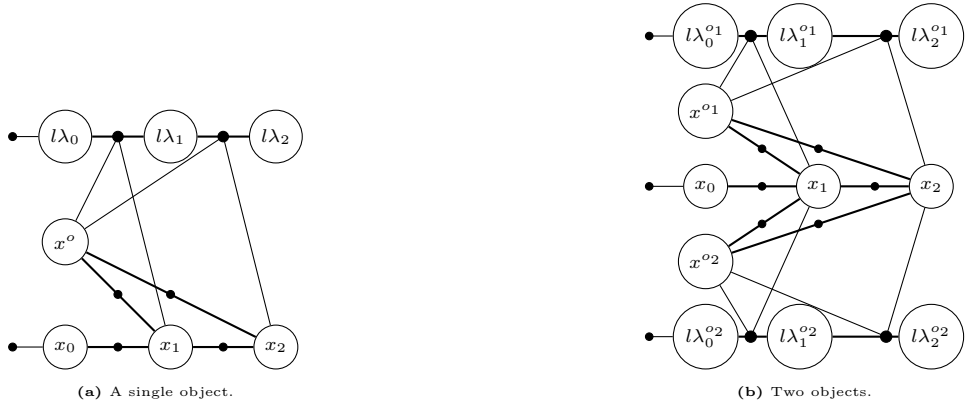
\begin{figure}
	\centering
	\scriptsize{
	\begin{subfigure}[b]{0.48\textwidth}
		\centering
		\begin{tikzpicture}[scale=0.7]
		\node at (0,1) [circle,draw] (x0) {$x_0$};
		\node at (2,1) [circle,draw] (x1) {$x_1$};
		\node at (4,1) [circle,draw] (x2) {$x_2$};
		
		% Objects 
		\node at (0,3) [circle,draw] (o1) {$x^o$};
		
		% Lambda
		\node at (0,5) [circle,draw] (l0) {$\loglambda_0$};
		\node at (2,5) [circle,draw] (l1) {$\loglambda_1$};
		\node at (4,5) [circle,draw] (l2) {$\loglambda_2$};
		
		% Lambda connections
		\Edge(l0)(l1)
		\draw[fill=black] (1,5) circle (0.1) node[above right] {$ $};
		\Edge(l1)(l2)
		\draw[fill=black] (3,5) circle (0.1) node[above right] {$ $};
		
		%Motion models
		{$ $};
		\Edge(x0)(x1)
		\draw[fill=black] (1,1) circle (0.075) node[above right] {$ $};
		\Edge(x1)(x2)
		\draw[fill=black] (3,1) circle (0.075) node[above right] {$ $};
		
		% x1 geometric model
		\Edge(x1)(o1)
		\draw[fill=black] (1,2) circle (0.075) node[above right] {$ $};
		
		% x2 geometric model
		\Edge(x2)(o1)
		\draw[fill=black] (2,2) circle (0.075) node[above right] {$ $};
		
		% lambda_1 factor
		\draw[black] (1,5) -- (x1);
		\draw[black] (1,5) -- (o1);
		
		% lambda 2 factor
		\draw[black] (3,5) -- (x2);
		\draw[black] (3,5) -- (o1);
		
		% x0 prior
		\draw[black] (-1,1) -- (x0);
		\draw[fill=black] (-1,1) circle (0.075) node[above right] {$ $};
		
		% l0 prior
		\draw[black] (-1,5) -- (l0);
		\draw[fill=black] (-1,5) circle (0.075) node[above right] {$ $};
		
		\end{tikzpicture}
		\caption{A single object.}
		\label{fig:Factor_JLP_Single}
	\end{subfigure}
	\begin{subfigure}[b]{0.48\textwidth}
		\centering
		\begin{tikzpicture}[scale=0.5]
		\node at (0,0) [circle,draw] (x0) {$x_0$};
		\node at (3,0) [circle,draw] (x1) {$x_1$};
		\node at (6,0) [circle,draw] (x2) {$x_2$};
		
		% Objects 
		\node at (0,2) [circle,draw] (o1) {$x^{o_1}$};
		\node at (0,-2) [circle,draw] (o2) {$x^{o_2}$};
		
		% Lambda o1
		\node at (0,4) [circle,draw] (l0) {$\loglambda_0^{o_1}$};
		\node at (2.5,4) [circle,draw] (l1) {$\loglambda_1^{o_1}$};
		\node at (6,4) [circle,draw] (l2) {$\loglambda_2^{o_1}$};
		
		% Lambda o2
		\node at (0,-4) [circle,draw] (l02) {$\loglambda_0^{o_2}$};
		\node at (2.5,-4) [circle,draw] (l12) {$\loglambda_1^{o_2}$};
		\node at (6,-4) [circle,draw] (l22) {$\loglambda_2^{o_2}$};
		
		%---------- Object 1
		% Lambda connections
		\Edge(l0)(l1)
		\draw[fill=black] (1.2,4) circle (0.15) node[above right] {$ $};
		\Edge(l1)(l2)
		\draw[fill=black] (4.8,4) circle (0.15) node[above right] {$ $};
		
		%Motion models
		{$ $};
		\Edge(x0)(x1)
		\draw[fill=black] (1.5,0) circle (0.1) node[above right] {$ $};
		\Edge(x1)(x2)
		\draw[fill=black] (4.5,0) circle (0.1) node[above right] {$ $};
		
		% x1 geometric model
		\Edge(x1)(o1)
		\draw[fill=black] (1.5,1) circle (0.1) node[above right] {$ $};
		
		% x2 geometric model
		\Edge(x2)(o1)
		\draw[fill=black] (3,1) circle (0.1) node[above right] {$ $};
		
		% lambda_1 factor
		\draw[black] (1.2,4) -- (x1);
		\draw[black] (1.2,4) -- (o1);
		
		% lambda 2 factor
		\draw[black] (4.8,4) -- (x2);
		\draw[black] (4.8,4) -- (o1);
		
		%-----------------Object 2
		% Lambda connections
		\Edge(l02)(l12)
		\draw[fill=black] (1.2,-4) circle (0.15) node[above right] {$ $};
		\Edge(l12)(l22)
		\draw[fill=black] (4.8,-4) circle (0.15) node[above right] {$ $};
		
		% x1 geometric model
		\Edge(x1)(o2)
		\draw[fill=black] (1.5,-1) circle (0.1) node[above right] {$ $};
		
		% x2 geometric model
		\Edge(x2)(o2)
		\draw[fill=black] (3,-1) circle (0.1) node[above right] {$ $};
		
		% lambda_1 factor
		\draw[black] (1.2,-4) -- (x1);
		\draw[black] (1.2,-4) -- (o2);
		
		% lambda 2 factor
		\draw[black] (4.8,-4) -- (x2);
		\draw[black] (4.8,-4) -- (o2);
		
		% x0 prior
		\draw[black] (-1.5,0) -- (x0);
		\draw[fill=black] (-1.5,0) circle (0.1) node[above right] {$ $};
		
		% l01 prior
		\draw[black] (-1.5,4) -- (l0);
		\draw[fill=black] (-1.5,4) circle (0.1) node[above right] {$ $};
		
		% l02 prior
		\draw[black] (-1.5,-4) -- (l02);
		\draw[fill=black] (-1.5,-4) circle (0.1) node[above right] {$ $};
		
		\end{tikzpicture}
		\caption{Two objects.}
		\label{fig:Factor_JLP_Two}
	\end{subfigure}}
	\caption{Example of a factor graph for a scenario until $k=2$, where at each time step an object \textbf{(a)} or objects \textbf{(b)} are observed. There are priors for $x_0$, and $\loglambda_0$ for every object. Between the camera poses are motion factors, connecting camera and object poses are geometric measurement factors, and between $\loglambda$'s at different time steps the 4 variable factors are connecting.}
	\label{fig:Factor_JLP}
\end{figure}

\subsubsection{Multiple Objects}\label{JLP_Multi}

The extension to multiple objects within the \jlp framework is straight-forward. Each object $o$ has its own set of $\loglambda_o$ nodes, as seen in the example in Fig.~\ref{fig:Factor_JLP_Two}. The set of all $\loglambda_k^o$ for objects observed thus far is denoted as $\bar{\loglambda}_{k} \triangleq \{ \loglambda_k^o \}_{o \in O_{1:k}}$. In contrast with $\Lambda_k$, which is a single probability vector over class realization with $m^{N_k}$ categories, $\bar{\loglambda}_{k}$ is a set of vectors with $m-1$ elements each, being the logit transformation of a probability vector of an object, to a total of $(m-1)\cdot N_k$ elements for $\bar{\loglambda}_k$. As such, the joint belief is updated in a similar manner to the single object case:
\begin{eqnarray}
	b[\bar{\loglambda}_{k},\poses_k] \!\!&= & \!\! \eta \int_{\bar{\loglambda}_{k-1}} \!
\prod_{o \in O_k}
\prob{\loglambda_{k}^o|\loglambda_{k-1}^o,\{ \loggamma_{k}^o \},x^{rel}_k} \cdot
\motionmodel_k \cdot 
\nonumber \\ 
&& \cdot
\prob{Z^g_k|\poses_k} \cdot
b[\bar{\loglambda}_{{k-1}},\poses_{k-1}] d \bar{\loglambda}_{k-1},
\end{eqnarray}
with $\eta$ being a normalization constant that does not participate in inference. 

With this formulation, the difference between maintaining $b[\Lambda_k,\poses_k]$ in \mh and $b[\bar{\loglambda}_k,\poses_k]$ must be discussed.

\subsubsection{Computation Complexity and Discussion}

\mh maintains $\Lambda_k$, which is a posterior joint probability vector for all class realizations, with  $m^{N_k}$ categories as seen in Sec.~\ref{sec:MH_Multi}. 
Thus $\Lambda_k$ grows exponentially with the number of objects observed, and considering that the inference is done for $|W|$ times, each $w$ with its own $\Lambda_k$, \mh is intractable unless pruning methods are applied. On the other hand, \jlp maintains $\bar{\loglambda}_k$ which essentially maintains a separate $\loglambda_k$ per object, resulting in size of $(m-1) \cdot N_k$, which grows linearly with the number of objects, and results in better scaling. Moreover, the inference is done only a single time, as the set $W$ only plays a role in classifier output $\gamma$.

If we consider again the example scenario with two objects and three candidate classes, $\Lambda_k$ is a probability vector with 9 categories. On the other hand, $\bar{\loglambda}_k = \{ \loglambda_k^o , \loglambda_k^{o'} \}$ where $\loglambda_k^o $ and $\loglambda_k^{o'}$, each, is a vector with two elements as they are the logit transformation of $\lambda^o_k$ and $\lambda^{o'}_k$, respectively, totaling in 4 elements for $\bar{\loglambda}_k$.

\jlp has a single continuous belief with at most $k$ pose states and $n$ object pose states. In addition each object has $\loglambda$ with $m-1$ variables. Consider pose states with $d$ variables each, at worst the total number of variables in \jlp is $dk + dn + n\cdot (m-1)$. In total, the computational time complexity is $O\left( (dk +dn + nm)^3 \right)$ at worst.  Compared to \mh, as discussed in Sec.~\ref{sec:MH_Inf_Comp} where the time computational complexity is at worst $O \left( |W|m^n(k+n)^3 \right)$, \jlp scales significantly better with the number of objects.

We can compare Fig.~\ref{fig:FG_MH} and Fig.~\ref{fig:Factor_JLP_Two} for illustration of the difference between \mh and \jlp inference. In those figures, a scenario with two objects and two candidate is presented. Fig.~\ref{fig:FG_MH} presents the factor graphs for a specific $w$, therefore in this case we have to maintain $4|W|$ factor graphs.  On the other hand, with \jlp we have to maintain a single one that also contains additional $\loglambda$ nodes.

\begin{figure}
	\centering
	\begin{subfigure}[b]{0.48\textwidth}
		\centering
		\begin{tikzpicture}[scale=0.7]
		
		% Robot poses
		\node at (0,0) [circle,draw] (x0) {$x_0$};
		\node at (2,0) [circle,draw] (x1) {$x_1$};
		\node at (4,0) [circle,draw] (x2) {$x_2$};
		
		% Objects 
		\node at (3,2) [circle,draw] (o1) {$x^{o_1}$};
		\node at (3,-2) [circle,draw] (o2) {$x^{o_2}$};
		
		% x0 prior
		\draw[black] (-1,0) -- (x0);
		\draw[fill=black] (-1,0) circle (0.1) node[above right] {$ $};
		
		%Motion models
		{$ $};
		\Edge(x0)(x1)
		\draw[fill=black] (1,0) circle (0.1) node[above right] {$ $};
		\Edge(x1)(x2)
		\draw[fill=black] (3,0) circle (0.1) node[above right] {$ $};
		
		% Connections to objects
		\path[blue, bend left] (x1) edge (o1);
		\path[black] (x1) edge (o1);
		\draw[fill=black] (2.5,1) circle (0.1) node[above right] {$ $};
		\draw[fill=blue] (2,1) circle (0.1) node[above right] {$ $};

		\path[blue, bend right] (x1) edge (o2);
		\path[black] (x1) edge (o2);
		\draw[fill=black] (2.5,-1) circle (0.1) node[above right] {$ $};
		\draw[fill=blue] (2,-1) circle (0.1) node[above right] {$ $};
		
		\path[blue, bend right] (x2) edge (o1);
		\path[black] (x2) edge (o1);
		\draw[fill=black] (3.5, 1) circle (0.1) node[above right] {$ $};
		\draw[fill=blue] (4, 1) circle (0.1) node[above right] {$ $};
		
		\path[blue, bend left] (x2) edge (o2);
		\path[black] (x2) edge (o2);
		\draw[fill=black] (3.5, -1) circle (0.1) node[above right] {$ $};
		\draw[fill=blue] (4, -1) circle (0.1) node[above right] {$ $};

		\end{tikzpicture}
		\caption{$c^{o_1}=1$, $c^{o_2}=1$}
		\label{fig:FG_MH_1_1}
	\end{subfigure}
	\begin{subfigure}[b]{0.48\textwidth}
		\centering
		\begin{tikzpicture}[scale=0.7]
				% Robot poses
		\node at (0,0) [circle,draw] (x0) {$x_0$};
		\node at (2,0) [circle,draw] (x1) {$x_1$};
		\node at (4,0) [circle,draw] (x2) {$x_2$};
		
		% Objects 
		\node at (3,2) [circle,draw] (o1) {$x^{o_1}$};
		\node at (3,-2) [circle,draw] (o2) {$x^{o_2}$};
		
		% x0 prior
		\draw[black] (-1,0) -- (x0);
		\draw[fill=black] (-1,0) circle (0.1) node[above right] {$ $};
		
		%Motion models
		{$ $};
		\Edge(x0)(x1)
		\draw[fill=black] (1,0) circle (0.1) node[above right] {$ $};
		\Edge(x1)(x2)
		\draw[fill=black] (3,0) circle (0.1) node[above right] {$ $};
		
		% Connections to objects
		\path[red, bend left] (x1) edge (o1);
		\path[black] (x1) edge (o1);
		\draw[fill=black] (2.5,1) circle (0.1) node[above right] {$ $};
		\draw[fill=red] (2,1) circle (0.1) node[above right] {$ $};
		
		\path[blue, bend right] (x1) edge (o2);
		\path[black] (x1) edge (o2);
		\draw[fill=black] (2.5,-1) circle (0.1) node[above right] {$ $};
		\draw[fill=blue] (2,-1) circle (0.1) node[above right] {$ $};
		
		\path[red, bend right] (x2) edge (o1);
		\path[black] (x2) edge (o1);
		\draw[fill=black] (3.5, 1) circle (0.1) node[above right] {$ $};
		\draw[fill=red] (4, 1) circle (0.1) node[above right] {$ $};
		
		\path[blue, bend left] (x2) edge (o2);
		\path[black] (x2) edge (o2);
		\draw[fill=black] (3.5, -1) circle (0.1) node[above right] {$ $};
		\draw[fill=blue] (4, -1) circle (0.1) node[above right] {$ $};
		
		\end{tikzpicture}
		\caption{$c^{o_1}=2$, $c^{o_2}=1$}
		\label{fig:FG_MH_2_1}
	\end{subfigure}

	\begin{subfigure}[b]{0.48\textwidth}
		\centering
		\begin{tikzpicture}[scale=0.7]
		
				% Robot poses
		\node at (0,0) [circle,draw] (x0) {$x_0$};
		\node at (2,0) [circle,draw] (x1) {$x_1$};
		\node at (4,0) [circle,draw] (x2) {$x_2$};
		
		% Objects 
		\node at (3,2) [circle,draw] (o1) {$x^{o_1}$};
		\node at (3,-2) [circle,draw] (o2) {$x^{o_2}$};
		
		% x0 prior
		\draw[black] (-1,0) -- (x0);
		\draw[fill=black] (-1,0) circle (0.1) node[above right] {$ $};
		
		%Motion models
		{$ $};
		\Edge(x0)(x1)
		\draw[fill=black] (1,0) circle (0.1) node[above right] {$ $};
		\Edge(x1)(x2)
		\draw[fill=black] (3,0) circle (0.1) node[above right] {$ $};
		
		% Connections to objects
		\path[blue, bend left] (x1) edge (o1);
		\path[black] (x1) edge (o1);
		\draw[fill=black] (2.5,1) circle (0.1) node[above right] {$ $};
		\draw[fill=blue] (2,1) circle (0.1) node[above right] {$ $};
		
		\path[red, bend right] (x1) edge (o2);
		\path[black] (x1) edge (o2);
		\draw[fill=black] (2.5,-1) circle (0.1) node[above right] {$ $};
		\draw[fill=red] (2,-1) circle (0.1) node[above right] {$ $};
		
		\path[blue, bend right] (x2) edge (o1);
		\path[black] (x2) edge (o1);
		\draw[fill=black] (3.5, 1) circle (0.1) node[above right] {$ $};
		\draw[fill=blue] (4, 1) circle (0.1) node[above right] {$ $};
		
		\path[red, bend left] (x2) edge (o2);
		\path[black] (x2) edge (o2);
		\draw[fill=black] (3.5, -1) circle (0.1) node[above right] {$ $};
		\draw[fill=red] (4, -1) circle (0.1) node[above right] {$ $};

		\end{tikzpicture}
		\caption{$c^{o_1}=1$, $c^{o_2}=2$}
		\label{fig:FG_MH_1_2}
	\end{subfigure}
	\begin{subfigure}[b]{0.48\textwidth}
		\centering
		\begin{tikzpicture}[scale=0.7]
			% Robot poses
		\node at (0,0) [circle,draw] (x0) {$x_0$};
		\node at (2,0) [circle,draw] (x1) {$x_1$};
		\node at (4,0) [circle,draw] (x2) {$x_2$};
		
		% Objects 
		\node at (3,2) [circle,draw] (o1) {$x^{o_1}$};
		\node at (3,-2) [circle,draw] (o2) {$x^{o_2}$};
		
		% x0 prior
		\draw[black] (-1,0) -- (x0);
		\draw[fill=black] (-1,0) circle (0.1) node[above right] {$ $};
		
		%Motion models
		{$ $};
		\Edge(x0)(x1)
		\draw[fill=black] (1,0) circle (0.1) node[above right] {$ $};
		\Edge(x1)(x2)
		\draw[fill=black] (3,0) circle (0.1) node[above right] {$ $};
		
		% Connections to objects
		\path[red, bend left] (x1) edge (o1);
		\path[black] (x1) edge (o1);
		\draw[fill=black] (2.5,1) circle (0.1) node[above right] {$ $};
		\draw[fill=red] (2,1) circle (0.1) node[above right] {$ $};
		
		\path[red, bend right] (x1) edge (o2);
		\path[black] (x1) edge (o2);
		\draw[fill=black] (2.5,-1) circle (0.1) node[above right] {$ $};
		\draw[fill=red] (2,-1) circle (0.1) node[above right] {$ $};
		
		\path[red, bend right] (x2) edge (o1);
		\path[black] (x2) edge (o1);
		\draw[fill=black] (3.5, 1) circle (0.1) node[above right] {$ $};
		\draw[fill=red] (4, 1) circle (0.1) node[above right] {$ $};
		
		\path[red, bend left] (x2) edge (o2);
		\path[black] (x2) edge (o2);
		\draw[fill=black] (3.5, -1) circle (0.1) node[above right] {$ $};
		\draw[fill=red] (4, -1) circle (0.1) node[above right] {$ $};
			
		\end{tikzpicture}
		\caption{$c^{o_1}=2$, $c^{o_2}=2$}
		\label{fig:FG_MH_2_2}
	\end{subfigure}
	\caption{Factor graphs used by \mh with the same scenario as in Fig.~\ref{fig:Factor_JLP_Two} for a single $w$. Each dot and line represents a separate factor. The black factors between the camera and object represent the geometric model, while the colored factors represent the classifier models, $c=1$ and $c=2$ by {\color{blue}blue} and {\color{red}red} respectively.}
	\label{fig:FG_MH}
\end{figure}
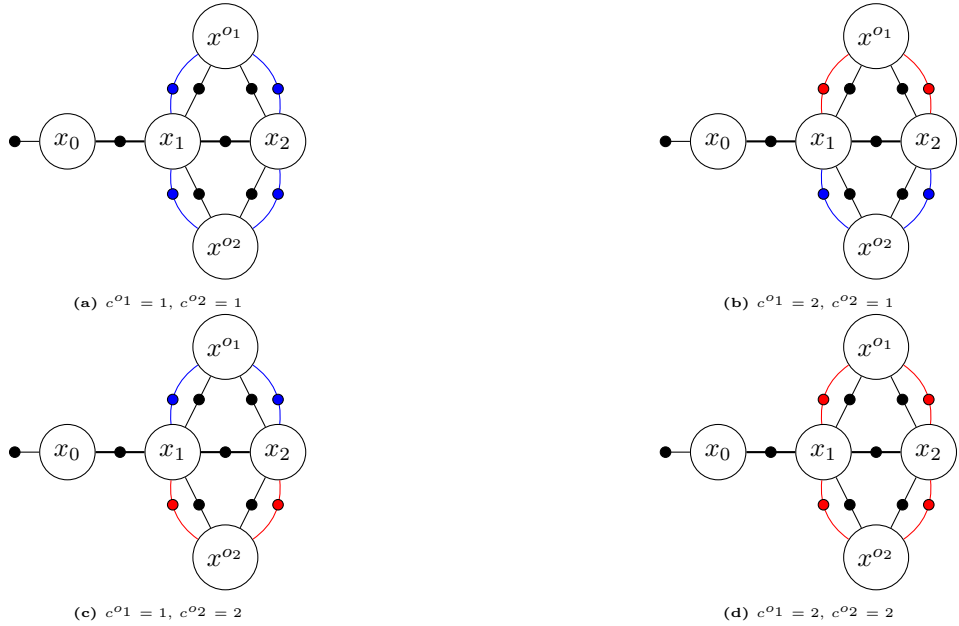

As we can see, the advantage of the \jlp approach compared to the \mh approach is that it does not require maintaining multiple hybrid beliefs. It maintains a single continuous belief that encompasses all poses and object classes while reasoning about classifier epistemic uncertainty, allowing a rich, viewpoint dependent representation of object class probabilities. 

The main disadvantage of this approach is the requirement of Lemma \ref{lemma:FGCondition} to hold for accuracy. While the requirement can be offset by enforcing additional constrains on the training of classifier uncertainty models, using the resulting models will render \jlp as approximation compared to \mh. Another potential drawback is that \jlp forces $\lambda$ to be LG distributed, which, as we will see in Sec.~\ref{sec:Dir_LG_discussion}, results in slower entropy computation. Despite of that, the advantage in computational efficiency of \jlp is significant enough to offset slower entropy computation relative to \mh, thus practically significantly more feasible.

	\section{Epistemic Uncertainty Aware Semantic Belief Space Planning}
	\label{sec:approach-planning}
	In this section we present a framework for epistemic uncertainty aware semantic BSP (\sembsp). Our framework incorporates reasoning about future posterior epistemic uncertainty within BSP; moreover, we appropriately generate future semantic and geometric observations while utilizing the coupling between $\lambda$ and $\poses$. Importantly, maintaining the corresponding future posterior belief $b[\lambda,\poses]$ within BSP allows to utilize a variety of reward functions, and in particular, information-theoretic rewards over epistemic uncertainty. As such, \sembsp provides key capabilities for reliable autonomous semantic perception in uncertain environments.

Each of the inference approaches developed in Section \ref{sec:approach-inference} has 
its own BSP counterpart. As we discuss in detail below, they are not compatible with each other, i.e.~\mh planning must be used with inference, and the same for \jlp. 

This section is structured as follows; First, in Sec.~\ref{sec:plan_meas_gen}, we discuss future measurement generation given candidate actions: For semantic measurements, we consider generating raw images, and then propose to generate semantic measurements directly from the viewpoint-dependent classifier uncertainty model from Eq.~\eqref{eq:Classifier_Uncertainty_Model}. Then, we detail the specifics of generating measurements from the model for \mh in Sec.~\ref{sec:MH_plan} and \jlp in Sec.~\ref{sec:JLP_plan}. Afterwards, we discuss possible reward functions, first mentioning rewards in the form of $r(b[\poses])$ and $r(\mathbb{E}(\lambda))$ in Sec.~\ref{sec:Non_uncertainty_rewards}, both indirectly involving reasoning about epistemic uncertainty. Further, we discuss an epistemic uncertainty information-theoretic reward $r(b[\lambda])$ in Sec.~\ref{sec:Uncertainty_rewards}, specifically the negative of differential entropy $-H(\lambda)$. We discuss computing $-H(\lambda)$ for both LG and Dirichlet distributed $\lambda$. For \mh approach, $\lambda$ can be distributed as either, but for the \jlp approach, $\lambda$ is limited to LG. Fig.~\ref{fig:Plan_Diag}  presents a diagram of all aspects considered in this section.

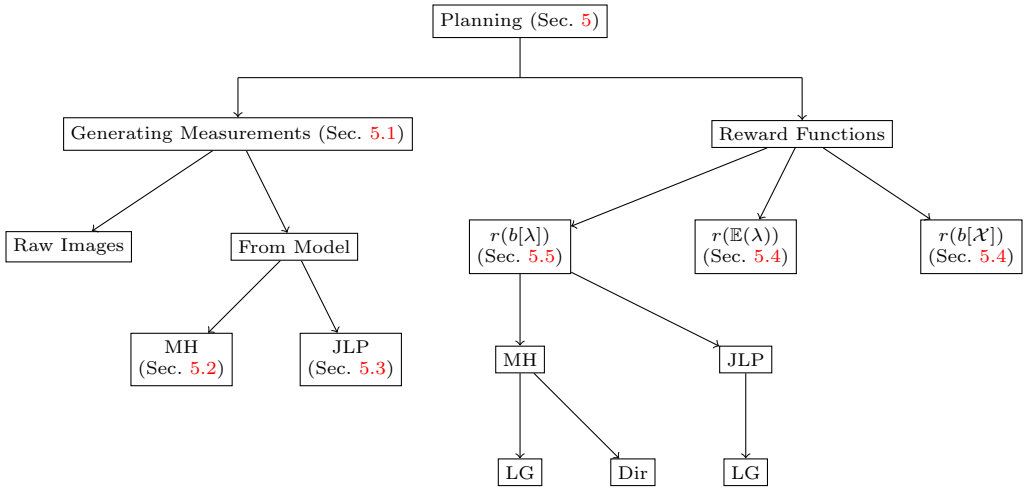
\begin{figure}[!htbp]
	\scriptsize{
	\begin{tikzpicture}[scale = 1.5]
		\node[draw] (Planning) at (0,10) {Planning (Sec.~\ref{sec:approach-planning})};
		\draw[black] (Planning) -- (0, 9.5);
		\draw[black] (-2.5,9.5) -- (2.5,9.5);
		
		% Sampling
		\node[draw] (Generating) at (-2.5,9) {Generating Measurements (Sec.~\ref{sec:plan_meas_gen})};
		\draw[->, black] (-2.5,9.5) -> (Generating);
		\node[draw] (Image) at (-4,8) {Raw Images};
		\draw[->, black] (Generating) -> (Image);
		\node[draw] (Model) at (-2,8) {From Model};
		\draw[->, black] (Generating) -> (Model);
		\node[draw, align=center] (MHGen) at (-3,7) {MH \\ (Sec.~\ref{sec:MH_plan})};
		\draw[->, black] (Model) -> (MHGen);
		\node[draw, align=center] (JLPGen) at (-1.5,7) {JLP \\ (Sec.~\ref{sec:JLP_plan})}; 
		\draw[->, black] (Model) -> (JLPGen);
		
		% Reward functions
		\node[draw] (Reward) at (2.5,9) {Reward Functions};
		\draw[->, black] (2.5,9.5) -> (Reward);
		\node[draw, align=center] (Poses) at (4,8) {$r(b[\poses])$ \\(Sec.~\ref{sec:Non_uncertainty_rewards})};
		\draw[->, black] (Reward) -> (Poses);
		\node[draw, align=center] (Class) at (2,8) {$r(\mathbb{E}(\lambda))$ \\(Sec.~\ref{sec:Non_uncertainty_rewards})};
		\draw[->, black] (Reward) -> (Class);
		\node[draw, align=center] (lambda) at (0,8) {$r(b[\lambda])$ \\(Sec.~\ref{sec:Uncertainty_rewards})};
		\draw[->, black] (Reward) -> (lambda);
		\node[draw] (MHReward) at (0,7) {MH};
		\draw[->, black] (lambda) -> (MHReward);
		\node[draw] (JLPReward) at (2,7) {JLP};
		\draw[->, black] (lambda) -> (JLPReward);
		\node[draw] (LGMH) at (0,6) {LG};
		\draw[->, black] (MHReward) -> (LGMH);
		\node[draw] (DirMH) at (1,6) {Dir};
		\draw[->, black] (MHReward) -> (DirMH);
		\node[draw] (LGJLP) at (2,6) {LG};
		\draw[->, black] (JLPReward) -> (LGJLP);
		
	\end{tikzpicture}}
	\caption{A diagram of aspects considered in Sec.~\ref{sec:approach-planning}. \emph{Dir} stands for Dirichlet distribution.}
	\label{fig:Plan_Diag}
\end{figure}

%\VI{[needs to be revisited once this section is more mature]}\VT{[...]}
%-------------------------------------------------
\subsection{Measurement Generation}\label{sec:plan_meas_gen}

As part of the objective function \eqref{eq:Obj_Func_Max} evaluation, we need to reason about future observations, both geometric and semantic. 
While geometric measurements can be sampled given through the geometric measurement model given sampled poses, the semantic measurement generation, especially when accounting for epistemic uncertainty is not immediate. For clarity, in this section we consider the single object case, while in the next sections we expand to the multiple object case in each method.

One alternative is to consider, for the $i$th look ahead step, generating $\{ \loggamma_{k+i} \}$ by first predicting raw measurements, i.e.~future images $I_{k+i}$. Given each such image, we can produce $\{ \loggamma_{k+1} \}$
by forwarding $I_{k+1}$ through a classifier 
for each $w \in W$,  similarly to passive inference.
In such a case, the objective function \eqref{eq:Obj_Func_Max} becomes:
\begin{equation}\label{eq:init_plan_split_image}
\begin{split}
	& J(b[\lambda_k,\poses_k],a_{k:k+L}) = \\ &\mathbb{E}_{I_{k+1:k+L}, z^g_{k+1:k+L}}
	(\sum_{i=1}^l r(b[\lambda_{k+i},\poses_{k+i}], a_{k+i})),
\end{split}	
\end{equation}
where
\begin{equation}\label{eq:plan_bel_image}
\begin{split}
	& b[\lambda_{k+i},\poses_{k+i}] = 
	\\ & \prob{\lambda_{k+i},\poses_{k+i}|I_{k+1:k+i},z^g_{k+1:k+i},I_{1:k},\his^g_k,D}.
\end{split}
\end{equation}
As presented in Sec.~\ref{sec:Prelim_BSP}, we have to use a generative model for generating measurements with the general form of $\prob{\mathcal{Z}_{k+1:k+L}|\his_k,a_k}$.  In this case, it takes the form of $\prob{I_{k+1:k+L},z^g_{k+1:k+L}|\his_k,a_k}$, which is a generative model for generating raw images and geometric measurements.

This model generates images from a candidate viewpoint of a scene yet to be observed, given a set of environments it was trained on. While such works do exist (e.g. \cite{Ha18arxiv}, the problem is high dimensional and feasible only in specifically trained environments.

In contrast, we propose an alternative approach that generates semantic measurements directly via a learned viewpoint dependent classifier uncertainty model 
\eqref{eq:Classifier_Uncertainty_Model}, thereby avoiding generating raw, high-dimensional images. 

Specifically, we use the LG model presented in Eq.~\eqref{eq:Logistical_Gaussian_PDF} for generating semantic measurements $z^s_{k+1:k+L}$ with the specifics discussed in Sec.~\ref{sec:MH_plan} and Sec.~\ref{sec:JLP_plan} for \mh and \jlp respectively. Thus, as alternative to Eq.~\eqref{eq:init_plan_split_image}, the objective function \eqref{eq:Obj_Func_Max} becomes:
\begin{equation}\label{eq:init_plan_split_clsmodel}
\begin{split}
	& J(b[\lambda_k,\poses_k],a_{k:k+L}) = \\ &\mathbb{E}_{z^s_{k+1:k+L}, z^g_{k+1:k+L}}
	(\sum_{i=1}^L r(b[\lambda_{k+i},\poses_{k+i}], a_{k+i})),
\end{split}	
\end{equation}
where, as opposed to Eq.~\eqref{eq:plan_bel_image},  $b[\lambda_{k+i},\poses_{k+i}]$ is conditioned on $z^s_{k+1:k+i}$, i.e.
\begin{equation}
\begin{split}
	& b[\lambda_{k+i},\poses_{k+i}] = \\ & \prob{\lambda_{k+i},\poses_{k+i}|z^s_{k+1:k+i},z^g_{k+1:k+i},I_{1:k},\his^g_k,D}.
\end{split}
\end{equation}
As both the geometric $\prob{z^g_k|x^{rel}_k}$ and classifier \eqref{eq:Classifier_Uncertainty_Model} models require $x^{rel}_{k+i}$, in addition to the class hypothesis $c$, measurement generation involves sampling both. We now discuss the specifics for each method, addressing \mh in  Sec.~\ref{sec:MH_plan} and \jlp  in Sec.~\ref{sec:JLP_plan} while expanding both to multiple objects.

%-----------------------------------------------------------
\subsection{Multi-Hybrid Planning (\mhbsp)}\label{sec:MH_plan}

In this section we discuss the specifics of generating measurements for planning using \mh.
Now considering multiple objects, we must generate future $Z^g$ and $\{\logGamma\}$, s.t. the objective function is as follows:
\begin{equation}\label{eq:MH_plan_mo}
\begin{split}
	& J(b[\Lambda_k,\poses_k],a_{k:k+L}) = \\ &\mathbb{E}_{\{\logGamma_{k+1:k+L}\}, Z^g_{k+1:k+L}}
	(\sum_{i=1}^l r(b[\Lambda_{k+i},\poses_{k+i}], a_{k+i})),
\end{split}	
\end{equation}
where $b[\Lambda_k,\poses_k]$ is obtained by \mh from Sec.~\ref{sec:MH_Multi}, and:
\begin{equation}
\begin{split}
	& b[\Lambda_{k+i},\poses_{k+i}] = \\ & \prob{\Lambda_{k+i},\poses_{k+i}|\{\logGamma_{k+1:k+i}\},Z^g_{k+1:k+i},I_{1:k},\his^g_k,D}.
\end{split}
\end{equation}
As each $\logGamma_{k+i,w}$ consists of separate $\loggamma^o_{k+i,w}$, and similarly $Z^g_{k+i}$ consists of $z^{g,o}_{k+i}$, we must first predict which objects will be observed at time $k+i$. This can be done using an object observation model (see e.g. \cite{Tchuiev19iros}) and sampled robot and object poses (either by sampling all objects or using a heuristic, see e.g.~\cite{Farhi19icra}); These objects are included in the predicted $O_{k+i}$ set, and form $\poses^{inv}_{k+1} \triangleq  x_{k+1} \cup \{ x^o \}_{o \in O_{k+1}}$.

To present that generative model, we first consider the generation of $\{\logGamma_{k+1}\}$ and $Z^g_{k+1}$ from $b[\Lambda_k, \poses_k]$ conditioned on action $a_k$. We present a sampling hierarchy that is described by the following marginalization scheme:
\begin{equation}\label{eq:Future_Likelihood}
\begin{array}{c}
	\prob{\{\logGamma_{k+1}\},Z^g_{k+1}|\his_k,a_k} = \\ 
	\sum_C \int_{\Lambda_k,\poses_{k+1}} \likelihoods_{k+1} \cdot \prob{C|\Lambda_k} \cdot
	\motionmodel_{k+1} \cdot b[\Lambda_k,\poses_k] d\Lambda_k d\poses_{k+1},
\end{array}
\end{equation}
which induces the following sampling hierarchy, for every object $o \in O_{k+1}$:
\begin{eqnarray}\label{eq:Sample_Generative_Model}
	\{ \loggamma_{k+1}^o \} &\sim& \prob{\loggamma_{k+1}^o | c^o,x^{rel,o}_{k+1}} \
	\\
	z^{g,o}_{k+1} &\sim& \prob{z^{g,o}_{k+1}|x^{rel,o}_{k+1}} 
	\\
	C &\sim& Cat(\Lambda_{k,w})
	\\
	\poses_{k+1} &\sim& \motionmodel_{k+1} \prob{\poses_{k}|\Lambda_{k,w},\loggamma_{1:k,w},\his^g_k}
	\\
	w &\sim& \prob{w|D},
\end{eqnarray}
where $x^{rel,o}_{k+1} \triangleq x^o \ominus x_{k+1}$, and is determined by $\poses^{inv}_{k+1}$.
Recall that $O_{k+1}$ must be determined by sampling $\poses_{k+1}$. First, $w$ is sampled uniformly from $|W|$. From there, as $b[\Lambda_k]$ is already represented by a set of samples $\{ \Lambda_{k,w} \}$, sampling $w$ chooses $\Lambda_{k,w}$ as well. Next, following from Eq.~\eqref{eq:MH_conditioned_poses} for the multiple object case, definition \eqref{eq:MH_hb_def_multi} for $hb_w[\poses_k,C]$, and that $\Lambda_{k,w}$ is chosen:
\begin{equation}
	\prob{\poses_{k}|\Lambda_{k,w},\loggamma_{1:k,w},\his^g_k} =
	\sum_c hb_w[\poses_k,C].
\end{equation}
Then $\prob{\poses_{k}|\Lambda_{k,w},\loggamma_{1:k,w},\his^g_k}$ is propagated via:
\begin{equation}
	\prob{\poses_{k+1}|\Lambda_{k,w},\loggamma_{1:k,w},\his^g_k,a_k} = \sum_c \motionmodel_{k+1} hb_w[\poses_k,C],
\end{equation}
and $\poses_{k+1}$ is sampled, from there we determine $O_{k+1}$.

Now for each object $o \in O_{k+1}$ we determine the appropriate $x^{rel,o}_{k+1}$, and generate its own geometric measurement $z^{g,o}_{k+1}$. Next we sample class realization $C$; As the action $a_k$ alone doesn't change $\Lambda$ from time $k$ to $k+1$ without measurements, $\Lambda_{k,w}$ is used to sample $C$. As such, $C$ is a categorical random variable with the probability vector $\Lambda_{k,w}$ as its parameters. Finally, with $c \in C$ and $x^{rel,o}_{k+1}$ we sample a set of $|W|$ vectors $\loggamma^o_{k+1}$.

Often planning algorithms use Maximum Likelihood (ML) estimation to reduce computational effort compared to sampling; Note that in our case, taking the ML estimation of $C$ can be problematic because it only considers the most likely class realization, ignoring all possible others.

For the following time steps, we use the generated $\{\logGamma_{k+1}\}$ and $Z^g_{k+1}$ to infer $b[\Lambda_{k+1}, \poses_{k+1}]$ using \mh inference from Sec.~\ref{sec:MH_Multi}. Now using action $a_{k+1}$, we can generate $\{\logGamma_{k+2}\}$ and $Z^g_{k+2}$, then $b[\Lambda_{k+2}, \poses_{k+2}]$, and continue generating measurements and inferring corresponding belief until the end of planning horizon. %we reach action $a_{k+L}$ at the end of our horizon.

Alg.~\ref{alg:SamplingStrategy} presents the \mhbsp measurement generation algorithm, where the function \emph{PredictObs} predicts which objects are observed given sampled camera and object poses. 

\begin{algorithm}[t]
	\caption{\mhbsp Measurement Generation} 
	\label{alg:SamplingStrategy} 
	\flushleft
	
	\begin{algorithmic}[1]
		
		\Require{Belief $b[\Lambda_k,\poses_k]$, action $a_k$}
		\State{$b[\Lambda_k,\poses_{k+1}] \gets \motionmodel_k \cdot b[\Lambda_k,\poses_k]$}
		\State{$\Lambda_k, \poses_{k+1} \gets \text{Sample}(b[\Lambda_k,\poses_{k+1}])$}
		\State{$O_{k+1} \gets \text{PredictObs}(\poses_{k+1})$}
		\State{$\poses^{inv}_{k+1} \gets O_{k+1}, \poses_{k+1}$}
		\State{$Z^g_{k+1} \gets \emptyset$}
		\State{$\{\logGamma_{k+1}\} \gets \emptyset$}
		%------>
		\For{$o \in O_{k+1}$}
		\State{$x^{rel,o}_{k+1} \gets x^o, x_{k+1} \in \poses^{inv}_{k+1}$}
		\State{$c^o \gets C$}
		\State{$z^{g,o}_{k+1} \gets \text{Sample}(\prob{z^{g,o}_{k+1}|x^{rel,o}_{k+1}})$}
		\State{$Z^g_{k+1} \gets Z^g_{k+1} \cup z^g_{k+1}$}
		\State{$\{ \loggamma^o_{k+1} \} \gets \emptyset$}
		%------>
		\For{$w \in W$}
		\State{$\{ \loggamma^o_{k+1} \} \gets \{ \loggamma^o_{k+1} \} \cup \text{Sample(\prob{\loggamma_{k+1}^o | c^o,x^{rel,o}_{k+1}})}$}
		\EndFor
		%------<
		\State{$\{\logGamma_{k+1}\} \gets \{\logGamma_{k+1}\} \cup \{ \loggamma^o_{k+1} \}$}
		\EndFor \\
		%------<
		\Return{$Z^g_k,\{ \logGamma_{k+1} \}$}

	\end{algorithmic}
\end{algorithm}

We summarize our approach with the \mhbsp objective function computation Alg.~\ref{alg:PC}, where  the function \emph{UpdateHB} is the hybrid belief update approach presented in Sec.~\ref{sec:MH_Multi}, and \emph{InferDist} infers $b[\Lambda_{k+1},\poses_{k+1}]$ from measurement generated in Alg.~\ref{alg:SamplingStrategy}. Alg.~\ref{alg:PC} recursively calls itself until the action set only includes one action, allowing non-myopic planning.

\begin{algorithm}[t]
	\caption{\mhbsp Objective Function} 
	\label{alg:PC} 
	\flushleft
	
	\begin{algorithmic}[1]
		\Require{\mh Belief $b[\Lambda_k,\poses_k]$, a set of actions $a_{k:k+L}$}
		%--->
		\State{$J \gets 0$}
		\For{number of samples $N_s$}
		\State{$Z^g_{k+1}, \{ \logGamma_{k+1} \} \gets$}\par \hskip\algorithmicindent {$ \text{MH Measurement Generation}(b[\Lambda_k,\poses_k], a_k) (\text{Alg.~\ref{alg:SamplingStrategy}})$}
		
		\State{$b[\Lambda_{k+1},\poses_{k+1}] \gets \text{UpdateMH}(b[\Lambda_k,\poses_k],\{\logGamma_{k+1}\},Z^g_{k+1},a_k)$}
		\State{$r(b[\Lambda_{k+1},\poses_{k+1}]) \gets \text{Reward}(b[\Lambda_{k+1},\poses_{k+1}]$}
		
		\State{$J \gets J + r(b[\Lambda_{k+1},\poses_{k+1}]) / N_s$}
		
		\If{$L \neq 0$}
		\State{$J \gets J $}\par \hskip\algorithmicindent{$ + \text{MH Objective Function}(b[\Lambda_{k+1},\poses_{k+1}], a_{k+1:k+L}) / N_s$}
		\EndIf
		
		\EndFor \\
		%---<
		\Return $J$
		
	\end{algorithmic}
\end{algorithm}

As in inference, while accurate, \mhbsp can be computationally expensive. Subsequently, in the next section we propose the expansion of \jlp for planning. As in inference, \jlp is significantly computationally faster.

%--------------------------------------------------------------------
\subsection{Joint Lambda Pose Planning (\jlpbsp)}\label{sec:JLP_plan}

In this section we present \jlpbsp, an epistemic uncertainty aware semantic BSP framework that leverages \jlp from Section \ref{sec:JLP} as the inference engine. If the assumption in Lemma.~\ref{lemma:FGCondition} is exactly or approximately satisfied, we can utilize \jlp for planning. 

Similarly to \mhbsp, we should reason about the generation of new measurement. As described in Sec.~\ref{sec:MH_plan}, \mhbsp uses the classifier uncertainty model \eqref{eq:Classifier_Uncertainty_Model} parameters $h_i(x^{rel}_{k+1})$ and $\Sigma_i(x^{rel}_{k+1})$ to generate $\{\loggamma_{k+1}\}$ given class $c=i$ and $x^{rel}_{k+1}$; On the other hand, \jlpbsp doesn't require generating $\{\loggamma_{k+1}\}$ and elegantly uses $h_i(x^{rel}_{k+1})$ and $\Sigma_i(x^{rel}_{k+1})$ as generated measurements.

With this, the objective function takes the following form:
\begin{equation}\label{eq:JLP_plan_mo}
\begin{split}
	& J(b[\bar{\loglambda}_k,\poses_k],a_{k:k+L}) = \\ &\mathbb{E}_{\mathbb{E}(\logGamma_{k+1:k+L}), \Sigma(\logGamma_{k+1:k+L}), Z^g_{k+1:k+L}}
	(\sum_{i=1}^L r(b[\bar{\loglambda}_k,\poses_{k+i}], a_{k+i}))
\end{split}	
\end{equation}
where,
\begin{equation}
\begin{split}
	& b[\bar{\loglambda}_{k+i},\poses_{k+i}] = \\ & \prob{\bar{\loglambda}_{k+i},\poses_{k+i}|\mathbb{E}(\logGamma_{k+1:k+L}), \Sigma(\logGamma_{k+1:k+L}),Z^g_{k+1:k+i},I_{1:k},\his^g_k,D},
\end{split}
\end{equation}
where $\mathbb{E}(\logGamma_k) \triangleq \{ \mathbb{E}(\loggamma^o_k) \}_{o \in O_k}$ and similarly $\Sigma(\logGamma_k) \triangleq \{ \Sigma(\loggamma^o_k) \}_{o \in O_k}$

As in Sec.~\ref{sec:MH_plan}, we consider measurement generation for time $k+1$ from time $k$. 
This time, we present a sampling hierarchy that is described by the following marginalization scheme:
\begin{equation}
\begin{array}{c}
	\prob{\mathbb{E}(\logGamma_{k+1}),\Sigma(\logGamma_{k+1}),Z^g_{k+1}|\his_k,a_k} = \\ \int_{\poses_{k+1},\bar{\loglambda}_k} \prod_{o \in O_{k+1}}
	\prob{\mathbb{E}(\loggamma^o_{k+1}), \Sigma(\loggamma^o_{k+1})|\loglambda_k^o,\his_k,a_k} \cdot \\ \cdot
	\prob{z^{g,o}_{k+1}|\poses_{k+1}} \cdot b[\bar{\loglambda}_k, \poses_{k+1}] d\poses_{k+1}.
\end{array}
\end{equation}
By using the above equation, we can write the generative model that is used to generate measurements for every $o \in O_{k+1}$. First, we need to determine the set $O_{k+1}$, and sample the hypothesized object class $c^o$ from $\loglambda^o_k \in \bar{\loglambda}_k$. We do so by sampling $\poses_{k+1}$ and $\bar{\loglambda}_k$ using $b[\bar{\loglambda}_k,\poses_k]$ as follows:
\begin{equation}
	\bar{\loglambda}_k, \poses_{k+1} \sim \motionmodel_{k+1} \cdot b[\bar{\loglambda}_k,\poses_k].
\end{equation}
Similar to \mhbsp, $\bar{\loglambda}_k$ stays the same conditioned on $a_k$, thus not propagated.
Then we determine $O_{k+1}$, $\poses^{inv}_{k+1}$ and $x^{rel,o}_{k+1}$ per object as we did in Sec.~\ref{sec:MH_plan}. From there, for $o \in O_{k+1}$ we sample $c^o$ and afterwards generate the measurements:
\begin{eqnarray}\label{eq:Gen_Model_JLP}
	\mathbb{E}(\loggamma^o_{k+1}) &=& h_c(x^{rel,o}_{k+1}) \\
	\Sigma(\loggamma^o_{k+1}) &=& \Sigma_c(x^{rel,o}_{k+1}) \\
	z^g_{k+1} &\sim& \prob{z^g_k|x^{rel,o}_{k+1}} \\
	c^o &\sim& Cat(\lambda_k^o),
\end{eqnarray}
where $h_c(x^{rel,o}_{k+1})$ and $\Sigma_c(x^{rel,o}_{k+1})$ are the Gaussian parameters of $\prob{\loggamma^o_{k+1}|c,x^{rel,o}_{k+1}}$, as in Eq.~\eqref{eq:Classifier_Uncertainty_Model}.

Alg.~\ref{alg:SamplingStrategy_JLP} presents the \jlp measurement generation algorithm, where $\mathbb{E}(\logGamma_k) \triangleq \{ \mathbb{E}(\loggamma^o_k) \}_{o \in O_k}$, and similarly $\Sigma(\logGamma_k) \triangleq \{ \Sigma(\loggamma^o_k) \}_{o \in O_k}$. 

\begin{algorithm}[t]
	\caption{\jlpbsp Measurement Generation} 
	\label{alg:SamplingStrategy_JLP} 
	\flushleft
	
	\begin{algorithmic}[1]
		
		\Require{Belief $b[\bar{\loglambda_k},\poses_k]$, action $a_k$}
		\State{$b[\bar{\loglambda_k},\poses_{k+1}] \gets \motionmodel_k \cdot b[\bar{\loglambda_k},\poses_k]$}
		\State{$\bar{\loglambda_k}, \poses_{k+1} \gets \text{Sample}(b[\Lambda_k,\poses_{k+1}])$}
		\State{$O_{k+1} \gets \text{PredictObs}(\poses_{k+1})$}
		\State{$\poses^{inv}_{k+1} \gets O_{k+1}, \poses_{k+1}$}
		\State{$Z^g_{k+1} \gets \emptyset$}
		\State{$\mathbb{E}(\logGamma_{k+1}) \gets \emptyset$}
		\State{$\Sigma(\logGamma_{k+1}) \gets \emptyset$}
		%------>
		\For{$o \in O_{k+1}$}
		\State{$x^{rel}_{k+1} \gets x^o, x_{k+1} \in \poses^{inv}_{k+1}$}
		\State{$z^{g,o}_{k+1} \gets \text{Sample}(\prob{z^{g,o}_{k+1}|x^{rel}_{k+1}})$}
		\State{$Z^g_{k+1} \gets Z^g_{k+1} \cup z^g_{k+1}$}
		\State{$c^o \gets \text{Sample}(Cat(\lambda^o_k))$}
		\State{$\mathbb{E}(\loggamma^o_{k+1}) \gets h_c(x^{rel}_{k+1})$}
		\State{$\mathbb{E}(\logGamma_{k+1}) \gets \mathbb{E}(\logGamma_{k+1}) \cup \mathbb{E}(\loggamma^o_{k+1})$}
		\State{$\Sigma(\loggamma^o_{k+1}) \gets \Sigma_c(x^{rel}_{k+1})$}
		\State{$\Sigma(\logGamma_{k+1}) \gets \Sigma(\logGamma_{k+1}) \cup \Sigma(\loggamma^o_{k+1})$}
		\EndFor \\
		%------<
		\Return{$Z^g_k,\mathbb{E}(\logGamma_{k+1}), \Sigma(\logGamma_{k+1})$}

	\end{algorithmic}
\end{algorithm}

The objective function computation is presented in Alg.~\ref{alg:PC_JLP}. \emph{UpdateJLP} refers to updating $b[\bar{\loglambda}_k, \poses_k]$ as in Sec.~\ref{sec:JLP} given generated measurements. Alg.~\ref{alg:PC_JLP} calls itself recursively until there is only one action left in the set. The algorithm is similar to Alg.~\ref{alg:PC}, except for the measurement generation and update functions which are specific for JLP.

\begin{algorithm}[t]
	\caption{\jlpbsp Objective Function} 
	\label{alg:PC_JLP} 
	\flushleft
	
	\begin{algorithmic}[1]
		\Require{\jlp Belief $b[\bar{\loglambda}_k,\poses_k]$, a set of actions $a_{k:k+l}$}
		%--->
		\State{$J \gets 0$}
		\For{number of samples $N_s$}
		\State{$Z^g_{k+1}, \mathbb{E}(\logGamma_{k+1}), \Sigma(\logGamma_{k+1}) \gets $}\par \hskip\algorithmicindent {$ \text{JLP Measurement Generation}(b[\bar{\lambda_k},\poses_k], a_{k:k+l})
			(\text{Alg.~\ref{alg:SamplingStrategy_JLP}})$}
		%--->
		\State{$b[\bar{\loglambda}_{k+1},\poses_{k+1}] \gets \text{UpdateJLP}(b[\bar{\loglambda}_k,\poses_k]),$}\par \hskip\algorithmicindent {$ Z^g_{k+1}, \mathbb{E}(\logGamma_{k+1}), \Sigma(\logGamma_{k+1})$}
		%---<
		\State{$r(b[\bar{\loglambda}_{k+1},\poses_{k+1}]) \gets \text{Reward}(b[\bar{\loglambda}_{k+1},\poses_{k+1}]$}
		
		\State{$J \gets J + r(b[\bar{\loglambda}_{k+1},\poses_{k+1}]) / N_s$}
		
		\If{$l \neq 0$}
		\State{$J \gets J $}\par \hskip\algorithmicindent{$ + \text{JLP Objective Function}(b[\bar{\loglambda}_{k+1},\poses_{k+1}], a_{k+1:k+l}) / N_s$}
		\EndIf
		
		\EndFor \\
		%---<
		\Return $J$
		
	\end{algorithmic}
\end{algorithm}

%-----------------------------------------------
\subsection{Reward Functions Over $b[\lambda,\poses]$}\label{sec:Non_uncertainty_rewards}

Predicting future $b[\lambda_{k+i},\poses_{k+i}]$ at a future time $k+i$ allows us to consider multiple reward functions, all captured by the general formulation $r(b[\lambda,\poses])$. To the best of our knowledge, we are the first to consider reasoning about \emph{future posterior epistemic uncertainty} within a BSP setting. For rewards based on the poses $r(\poses)$ e.g. distance-to-goal, or rewards based on the belief over the poses $r(b[\poses])$ e.g. information-theoretic costs, we can compute the marginal $b[\poses_{k+i}]$ as in Eq.~\eqref{eq:MH_Pose_Marginal} for \mh, or by marginalizing out $\loglambda_{k+i}$ from $b[\lambda_{k+i},\poses_{k+i}]$ for \jlp. 

In addition, we may also consider a reward over the posterior class probability $r(\prob{c|\his})$ which can be extracted by computing $\mathbb{E}(\lambda_{k+i})$ from the marginal $b[\lambda_{k+i}]$:
\begin{equation}
\begin{array}{c}
		\prob{c \mid \{\loggamma_{k+1:k+i}\},z^g_{k+1,k+i},I_{1:k},\his^g_{k},D} = \\  \int_{\lambda_{k+i}} \prob{c|\lambda_{k+i}} \cdot b[\lambda_{k+i}] d\lambda_{k+i}
	= \mathbb{E}(\lambda_{k+i}),
\end{array}
\end{equation}
therefore we can write $r(\prob{c|\his})$ as $r(\mathbb{E}(\lambda))$. An example for such reward is the minus of Shannon Entropy, such that $r(\mathbb{E}(\lambda)) = \sum_c \lambda^c \log(\lambda^c)$. This reward favors class probability vectors when one of the candidates has probability close to one, and others close to zero.

Crucially, as $b[\Lambda_{k+i},\poses_{k+i}]$ for \mhbsp and $b[\bar{\loglambda}_{k+i},\poses_{k+i}]$ for \jlpbsp both reason about epistemic uncertainty, it affects implicitly every reward. Thus, we account for future posterior epistemic uncertainty indirectly in all the cases discussed in this section.

%---------------------------------------------
\subsection{Information-Theoretic Reward Over $b[\lambda]$}\label{sec:Uncertainty_rewards}

In Sec.~\ref{sec:Non_uncertainty_rewards} we discussed reward functions in the form of $r(b[\poses])$ and $r(\prob{c|\his})$. But crucially, maintaining $b[\lambda_{k+i},\poses_{k+i}]$ opens the possibility of planning directly over $b[\lambda]$. We consider info-theoretical rewards over $\lambda$ in the form of $r(b[\lambda])$. Specifically, we consider the differential entropy of $\lambda_{k+i}$, denoted $H(\lambda_{k+i})$, and is defined as:
\begin{equation}\label{eq:H_def}
	H(\lambda_{k+1}) \triangleq - \int_{\lambda_{k+1}} b[\lambda_{k+1}] \cdot \log b[\lambda_{k+1}] d\lambda_{k+1}.
\end{equation}
The reward considered is the minus of the entropy, i.e. $r(b[\lambda]) = -H(\lambda)$, which, as we will see in Sec.~\ref{sec:LG_Dist_Lambda} and \ref{sec:Dir_Dist_Lambda}, is dependent both on $\mathbb{E}(\lambda)$ and the epistemic model uncertainty. 

A possible alternative is a reward of the following general form for $\lambda$ (see e.g. \cite{Lutjens18arxiv}):
\begin{equation}
	r(b[\lambda]) = \omega_1 \cdot f_1(\mathbb{E}(\lambda)) + \omega \cdot f_2(\Sigma(\lambda)),
\end{equation}
where $\omega_1$ and $\omega_2$ are hyperparameters, and $f_1$ and $f_2$ are general functions. Here $\lambda$ can be interchangeable with its logit transformation $\loglambda$. This reward requires the tuning of $\omega_1$ and $\omega_2$ manually, as opposed to using $r(b[\lambda]) = -H(\lambda)$ which does not require parameter tuning at all. In particular, as we will see in Sec.~\ref{sec:LG_Dist_Lambda} and Sec.~\ref{sec:Dir_Dist_Lambda}, $H(\lambda)$ addresses both $\mathbb{E}(\lambda)$ and $\Sigma(\lambda)$ simultaneously; $H(\lambda)$ diminishes (i.e. $r(b[\lambda])$ grows) when $\mathbb{E}(\lambda)$ is closer to the simplex corners, i.e. when one category has its probability close to 1 and the rest close to 0. Also, $H(\lambda)$ diminishes the smaller $\Sigma(\lambda)$ becomes, which corresponds to smaller epistemic uncertainty.

However, computing $H(\lambda_{k+i})$ requires the PDF value of $b[\lambda]$, according to Eq.~\eqref{eq:H_def}, thus requiring us to model the distribution of $\lambda_{k+i}$. This distribution can be either parametric e.g. Dirichlet or LG, which we will discuss here, or non-parametric such as Kernel Density Estimation (KDE). \mh provides us with $\{ \lambda \}$, therefore any distribution that supports probability vectors can be chosen. On the other hand, \jlp limits $\lambda$ to be LG distributed per definition. Sec.~\ref{sec:LG_Dist_Lambda} and Sec.~\ref{sec:Dir_Dist_Lambda} detail Dirichlet and Logistical Gaussian distributions for $b[\lambda_{k}]$ respectively in the context of computing entropy. Sec.~\ref{sec:Dir_LG_discussion} discusses the differences between utilizing both distributions.
To simplify notations, all of the variables in these sections are considered at the same time step, so we drop the time step index. In addition, we use the single-object notation, i.e. $\lambda$ and $c$.

%\VI{[The below specifically focuses on entropy over lambda. But later in sections C and D you again present the ]}

\subsubsection{Logistic Gaussian For $b[\lambda]$}\label{sec:LG_Dist_Lambda}

One option is to model $b[\lambda]$ as Logistic Gaussian (LG) distributed. This option is supported by both \mh and \jlp, as illustrated in Fig.~\ref{fig:Plan_Diag}. This distribution (with PDF as in Eq.~\eqref{eq:Logistical_Gaussian_PDF}) supports probability vectors with conditions presented in Sec.~\ref{sec:Gamma_Lambda} for $\gamma$, thus samples from LG are probability vectors. This distribution does not have an analytical expression for expectation and covariance, and must be computed numerically or approximated, e.g. via bounds, as we will discuss later.

To compute the parameters from a point cloud of probability vectors, e.g. $\{ \lambda \}$, we apply the logit transformation for each $\lambda \in \{ \lambda \}$, and get $\{\loglambda\}$. Then, as $\loggamma$ is modeled Gaussian the LG parameters $\mathbb{E}(\loglambda)$ and $\Sigma(\loglambda)$ are inferred.

In addition to expectation and covariance, the LG distribution does not have a closed form solution for its differential entropy. However, LG variable is a transformation of a Gaussian variable with a known expression for entropy. As such, we can express the entropy using the following lemma.

\begin{lemma}\label{lemma:LG_entropy_exact}
	Let $\lambda = [\lambda^1,...,\lambda^m]^T$ be Logistical-Gaussian distributed, and $\loglambda$ its logit transformation as in Eq.~\eqref{eq:Logit_transformation}, thus $\loglambda$ is Gaussian with parameters $\mathbb{E}(\loglambda)$ and $\Sigma(\loglambda)$. As such, the differential entropy $H(\lambda)$ is described by:
	\begin{equation}\label{eq:Entropy_LG_total}
		\begin{split}
		H(\lambda) = & H(\loglambda) + \sum_{i=1}^{m-1} \mathbb{E}[\loglambda^i] \\
		& - \int_{\loglambda} \log \left( 1 + \sum_{i=1}^{m-1} e^{\loglambda^i} \right) \prob{\loglambda} d\loglambda.
	\end{split}
	\end{equation}
\end{lemma}

The complete proof is shown at appendix \ref{sec:Proof2}.

As $\loglambda$ is Gaussian, $H(\loglambda) = 0.5 \cdot \log (2\pi e |Cov(\loglambda)|)$. The integral in Eq.~\eqref{eq:Entropy_LG_total} to the best of our knowledge does not have an analytical solution. One approach is to compute the entropy numerically from $\{\lambda\}$ that we already have, but it is computationally expensive to do so for a large number of candidate classes. Another option is to compute bounds for the entropy, which are presented in the following lemma.

\begin{lemma}\label{lemma:LG_entropy_bounds}
	Let $\lambda = [\lambda^1,...,\lambda^m]^T$ be Logistical-Gaussian distributed, and $\loglambda$ its logit transformation as in Eq.~\eqref{eq:Logit_transformation}, thus $\loglambda$ is Gaussian with parameters $\mathbb{E}(\loglambda)$ and $\Sigma(\loglambda)$. As such, an upper bound for $H(\lambda)$ is given by:
	\begin{equation}\label{eq:LG_entropy_upper_bound}
		H(\lambda) \leq H(\loglambda) + \sum_{i=1}^{m-1} \mathbb{E}(\loggamma^i) - m \cdot \max_i \{ 0 , \mathbb{E}(\loggamma^i) \},
	\end{equation}
	and similarly a lower bound is given by:
	\begin{equation}\label{eq:LG_entropy_lower_bound}
		\begin{split}
		H(\lambda) \geq & H(\loglambda) + \sum_{i=1}^{m-1} \mathbb{E}(\loggamma^i) \\ & - m \cdot \max_i \{ 0 , \mathbb{E}(\loggamma^i) \} - m \log m - \sqrt{\frac{\sigma_{ii}^{max}}{2\pi}},
	\end{split}
	\end{equation}
where $\sigma_{ii}^{max} \triangleq \max_i \Sigma_{ii}(\loglambda)$ is the largest value element in the covariance of $\loglambda$.
\end{lemma}
The complete proof is shown at appendix \ref{sec:Proof3}.

One can observe from the upper bound that $H(\loglambda)$ is necessarily larger than $H(\lambda)$ as $\loggamma$ is not subjected to the probability vector constraints, thus $\mathbb{E}(\loggamma^i)$ can be negative for every $i$ and $\sum_{i=1}^{m-1} \mathbb{E}(\loggamma^i) - m \cdot \max_i \{ 0 , \mathbb{E}(\loggamma^i) \}$ is necessarily non-positive.

Fig.~\ref{fig:LG_Entropy} presents the entropy values of $b[\lambda]$ as a function of its LG parameters $\mathbb{E}(\loglambda)$ and $Var(\loglambda)$ in the case of two candidate classes. As it has a single degree of freedom, two parameters can fully describe the distribution. The figure shows that the farther $\mathbb{E}[\loglambda]$ is from zero, i.e. the closer $\mathbb{E}[\lambda^1]$ to either one or zero, the smaller the entropy gets in general. The effect is more pronounced in the case where $Var(\loglambda)$ is small. If we aim to minimize  entropy during planning, the robot will aim to reach regions where $\mathbb{E}[\lambda]$ is close to the edges of the simplex, and have smaller posterior epistemic uncertainty. 

The scenarios presented in Fig.~\ref{fig:simplex_figures} correspond to the following cases in Fig.~\ref{fig:LG_Entropy}:

\begin{itemize}
	
	\item The unknown-unknown case (Fig.~\ref{fig:Fig_Unknown_Unknown}) corresponds to $\mathbb{E}(\loglambda)$ close to 0, and large $Var(\loglambda)$ , i.e. the upper central part of Fig.~\ref{fig:LG_Entropy}.
	
	\item The known-unknown case (Fig.~\ref{fig:Fig_Known_Unknown}) corresponds to $\mathbb{E}(\loglambda)$ close to 0, and small $Var(\loglambda)$, i.e. the lower central part of Fig.~\ref{fig:LG_Entropy}.
	
	\item The known-known case (Fig.~\ref{fig:Fig_Known_Known}) corresponds to $\mathbb{E}(\loglambda)$ with large absolute value, and small $Var(\loglambda)$, i.e. the lower areas at the sides.
	
	\item The uncertain classification case (Fig.~\ref{fig:Fig_Uncertain_Edge}) corresponds to $\mathbb{E}(\loglambda)$ with large absolute value, and large $Var(\loglambda)$, i.e. the upper areas at the sides.
	
\end{itemize}

\begin{figure}[!htbp]
	\centering
	\includegraphics[width=0.5\textwidth]{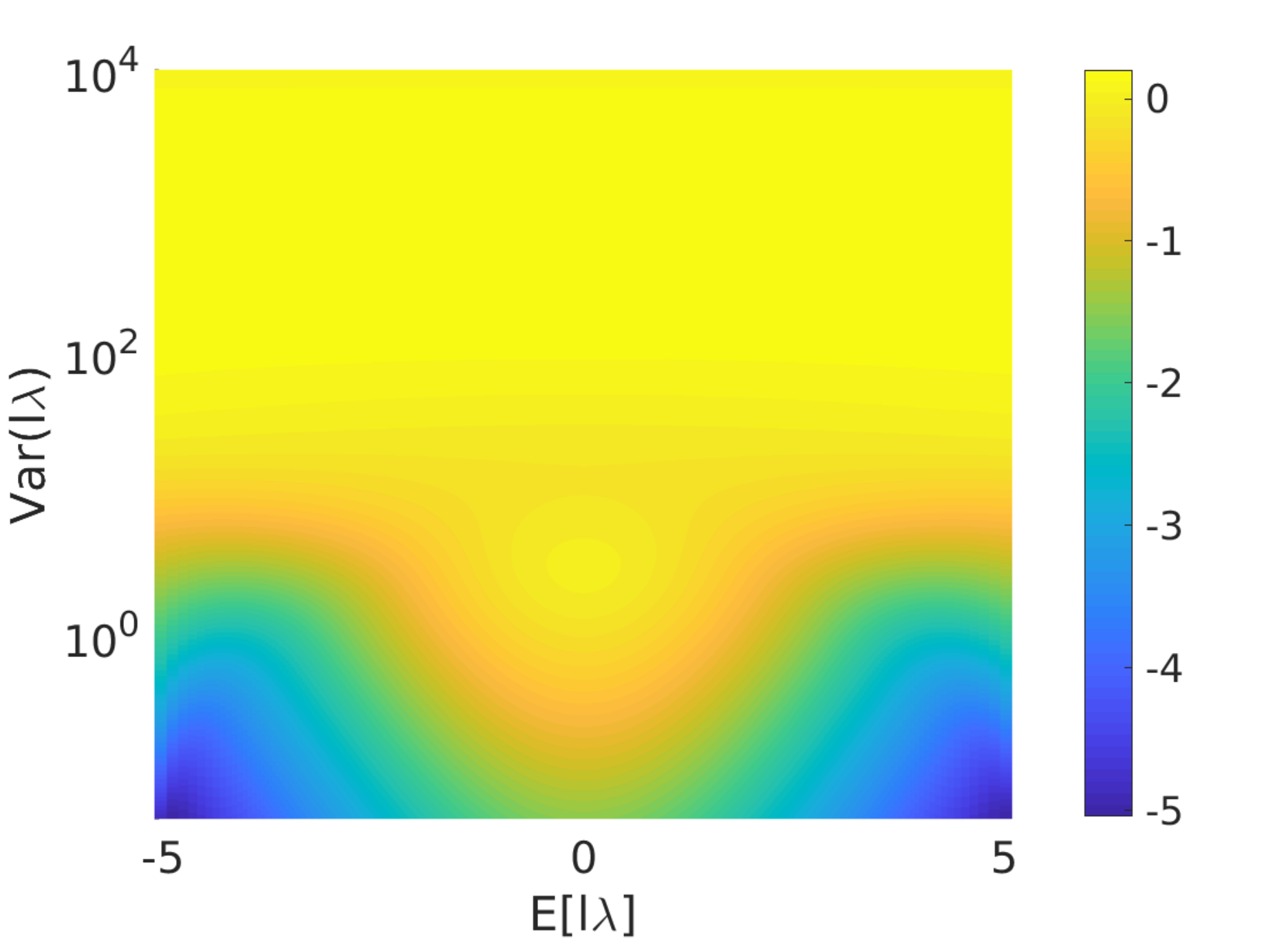}
	\caption{Entropy of a one dimensional Logistical Gaussian that corresponds to two dimensional probability vector $\gamma$. The $x$ and $y$ axis represent $\mathbb{E}(\loggamma)$ and $Var(\loggamma)$ respectively. Blue to yellow colors correspond to low to high entropy.}
	
	\label{fig:LG_Entropy}
\end{figure}

\subsubsection{Dirichlet Distribution For $b[\lambda]$}\label{sec:Dir_Dist_Lambda}

The other option assumes $b[\lambda]$ is Dirichlet distributed, which is supported only by \mh, as illustrated  in Fig.~\ref{fig:Plan_Diag}. This distribution is a natural representation of distribution over probability vectors in which samples necessarily satisfy all the conditions of probability vectors presented in Sec.~\ref{sec:Gamma_Lambda}. 

The Dirichlet distribution is parametrized by a parameter set $\alpha \triangleq \{ \alpha_1,...,\alpha_m \}$, and the PDF is:
\begin{equation}
Dir(\lambda;\alpha) = \frac{1}{B(\alpha)} \prod_{i=1}^m (\lambda^i)^{\alpha_i - 1},
\end{equation}
with $\lambda^i$ being the $i$-th class probability. $B(\alpha)$ is a normalization constant  defined as
\begin{equation}
B(\alpha) \triangleq \frac{\prod_{i=1}^m \Gamma(\alpha_i)}{\Gamma\left( \alpha_0 \right)},
\end{equation}
where $\Gamma(\cdot)$ is the Gamma function and $\alpha_0 \triangleq \sum_{i=1}^m \alpha_i$ for shorthand. 

Recall that in \mh $b[\lambda]$ is maintained via maintaining each $\lambda_w \in \{\lambda_w\}_{w \in W}$ as in Eq.~\ref{eq:MH_lambda_update_single} and Eq.~\ref{eq:MH_Multi_Obj_Lambda_Upd}.
Dirichlet's distribution parameters, given $\{ \lambda \}$, can be estimated in an iterative manner as follows \cite{Minka03}: 
\begin{equation}
\psi(\alpha_i^{\text{new}}) = \psi(\sum_{j=1}^m \alpha_j^{old}) + \log \hat{\lambda}^i,
\end{equation}
where $\log\hat{\lambda}^i \triangleq \frac{1}{|W|} \log\lambda^i_w$, and $\psi(\cdot)$ is the digamma function. The following expression shows the entropy of the Dirichlet distribution given $\alpha$ parameters:
\begin{equation}\label{eq:Dir_Entropy}
H(\lambda) = \log B(a) + (\alpha_0 - m) \psi(\alpha_0) - \sum_{i=1}^m (\alpha_i - 1) \psi(\alpha_i).
\end{equation}
The term $B(\alpha)$ needs to be numerically computed. While it is not an analytical solution, the computation is significantly faster than computing  differential entropy using samples.

This entropy takes the maximal value when $\alpha_i = 1,\; \forall i$, and at the "edges" of the distribution, where a single parameter is much larger than the others, the entropy is the lowest. If one of the parameters is zero, then $H(\lambda) = -\infty$, as $\psi(0) = -\infty$. This behavior of entropy can be observed in Fig.~\ref{fig:2D_Entropy_Graph} that shows an example for a two dimensional distribution.

The scenarios presented in Fig.~\ref{fig:simplex_figures} correspond to the following cases in Fig.~\ref{fig:2D_Entropy_Graph}:

\begin{itemize}
	
	\item The unknown-unknown case (Fig.~\ref{fig:Fig_Unknown_Unknown}) corresponds to $\alpha_1$ and $\alpha_2$ that are close to 1, i.e. the central part of Fig.~\ref{fig:2D_Entropy_Graph}.
	
	\item The known-unknown case (Fig.~\ref{fig:Fig_Known_Unknown}) corresponds to $\alpha$'s with large and similar values, i.e. the upper right part of Fig.~\ref{fig:2D_Entropy_Graph}.
	
	\item The known-known case (Fig.~\ref{fig:Fig_Known_Known}) corresponds to the case where one $\alpha$ is significantly larger than the other, and larger than 1, i.e. left or bottom areas of Fig.~\ref{fig:2D_Entropy_Graph}.
	
	\item The uncertain classification case (Fig.~\ref{fig:Fig_Uncertain_Edge}) corresponds to the case where one $\alpha$ is not significantly larger than the other, i.e. the areas between the high and low entropy in Fig.~\ref{fig:2D_Entropy_Graph}.
	
\end{itemize}

\begin{figure}[!htbp]
	\centering
	\includegraphics[width=0.5\textwidth]{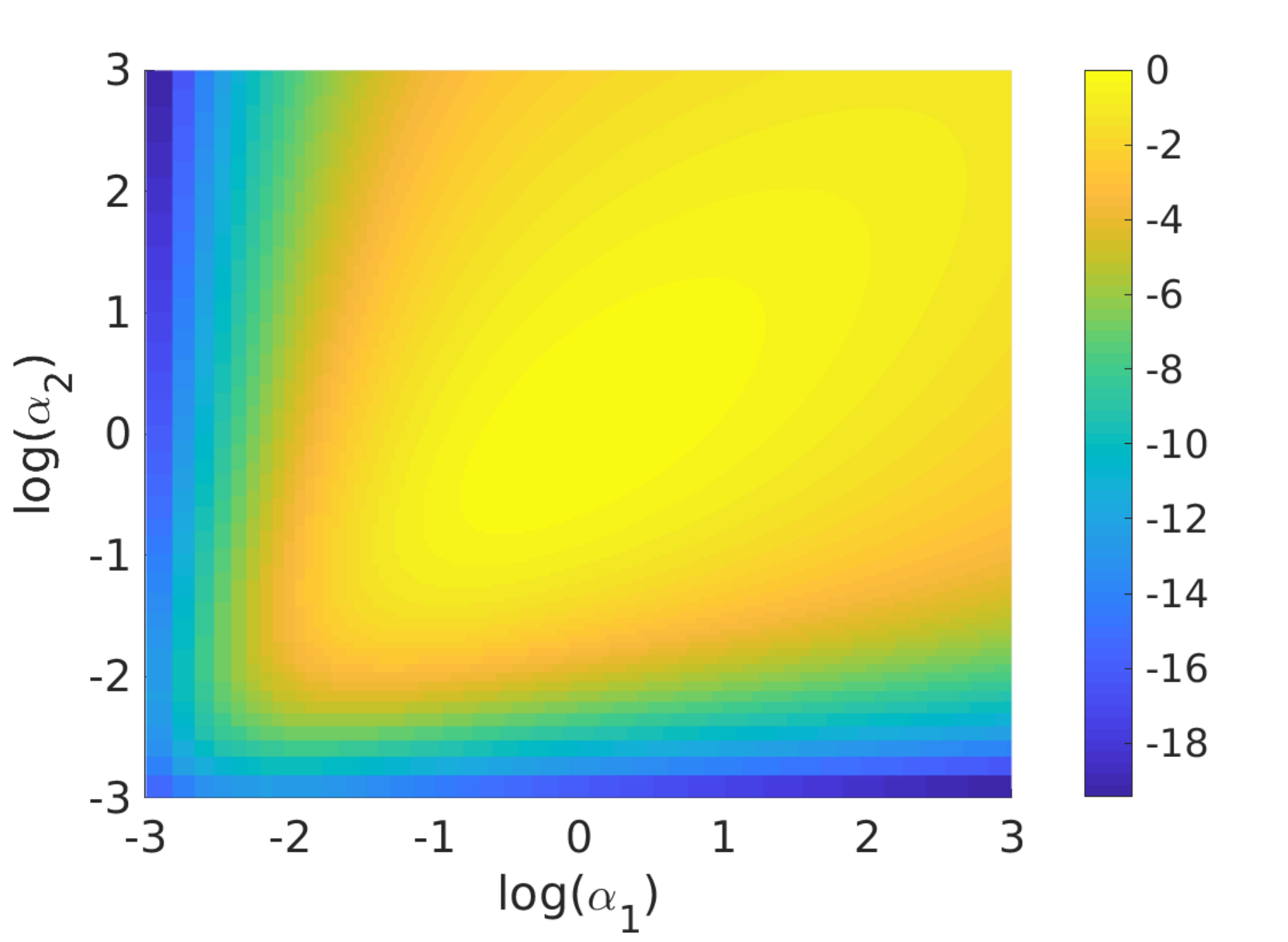}
	\caption{Entropy of a two dimensional Dirichlet distribution as a function of log of parameters. Blue to yellow colors correspond to low to high entropy values.}
	
	\label{fig:2D_Entropy_Graph}
\end{figure}

%----------------------------------------------------------
\subsubsection{Comparison Between Dirichlet and Logistic Gaussian}\label{sec:Dir_LG_discussion}
\begin{figure*}[!htbp]
	
	\begin{subfigure}[b]{0.18\textwidth}
		\includegraphics[width=\textwidth]{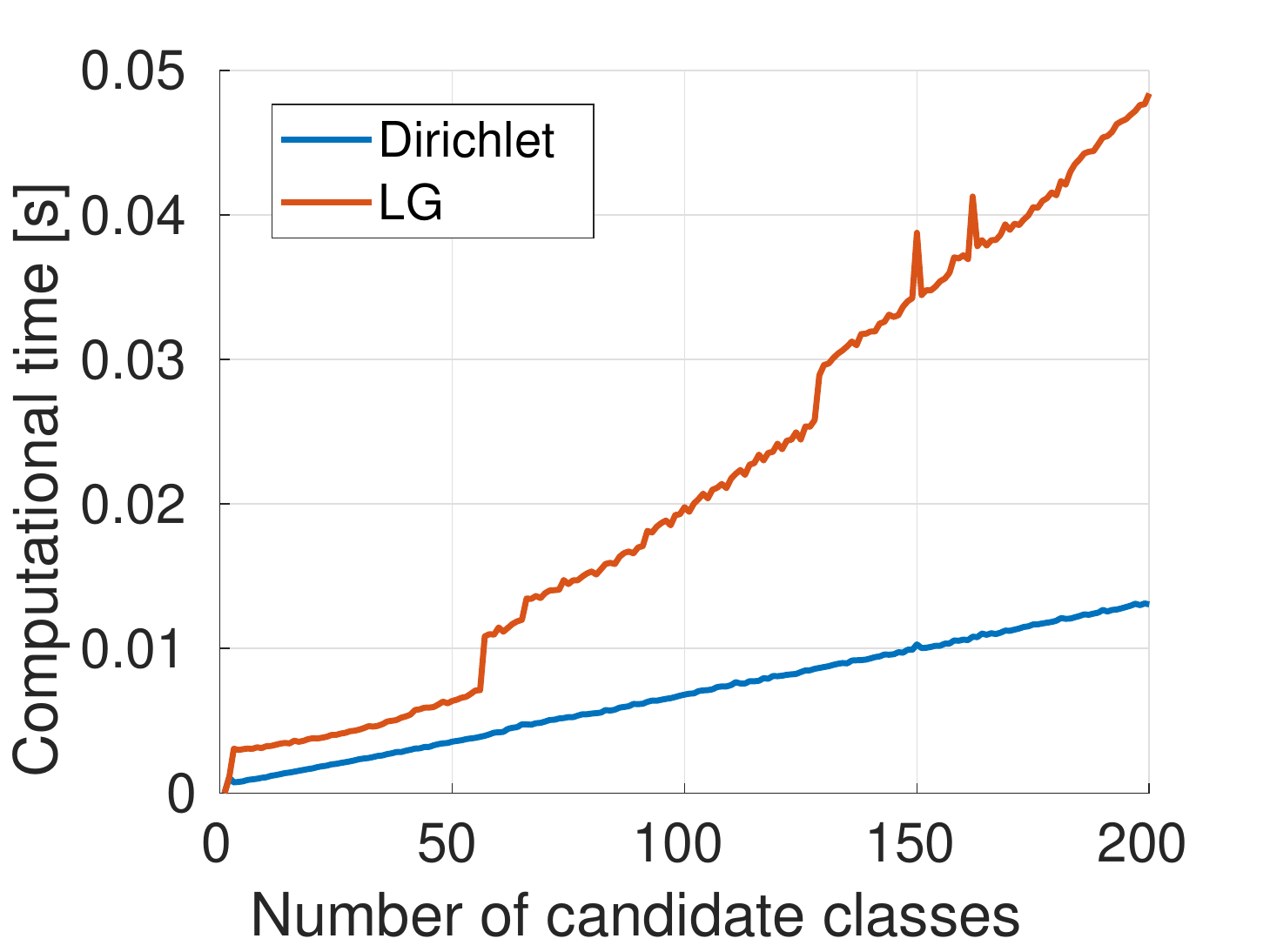}
		\caption{Inference time}\label{fig:Fig_time_inf}
	\end{subfigure}
	\begin{subfigure}[b]{0.18\textwidth}
		\includegraphics[width=\textwidth]{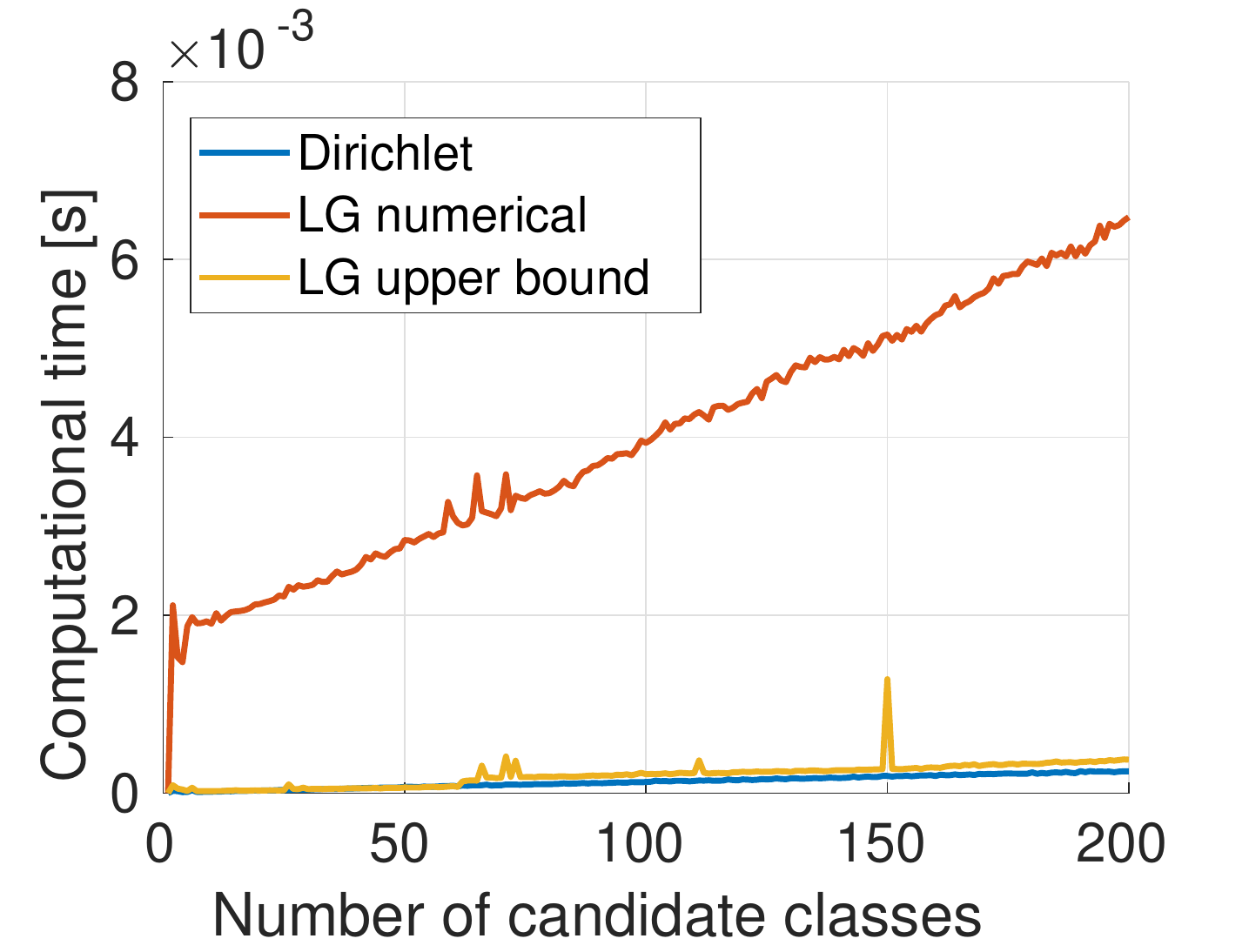}
		\caption{Entropy time}\label{fig:Fig_time_ent}
	\end{subfigure}
	\begin{subfigure}[b]{0.18\textwidth}
		\includegraphics[width=\textwidth]{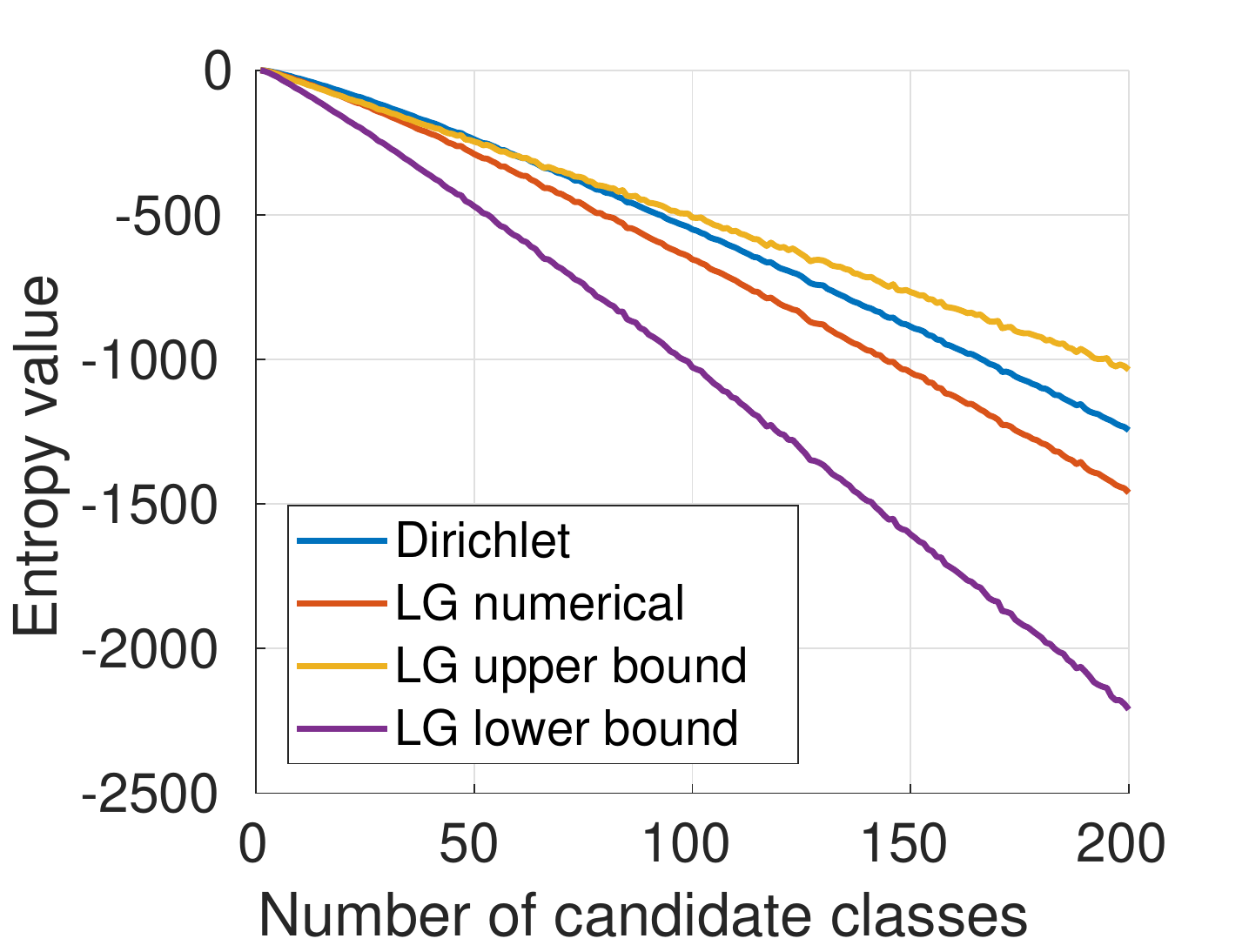}
		\caption{Entropy value}\label{fig:Val_fig_ent}
	\end{subfigure}
	\begin{subfigure}[b]{0.18\textwidth}
		\includegraphics[width=\textwidth]{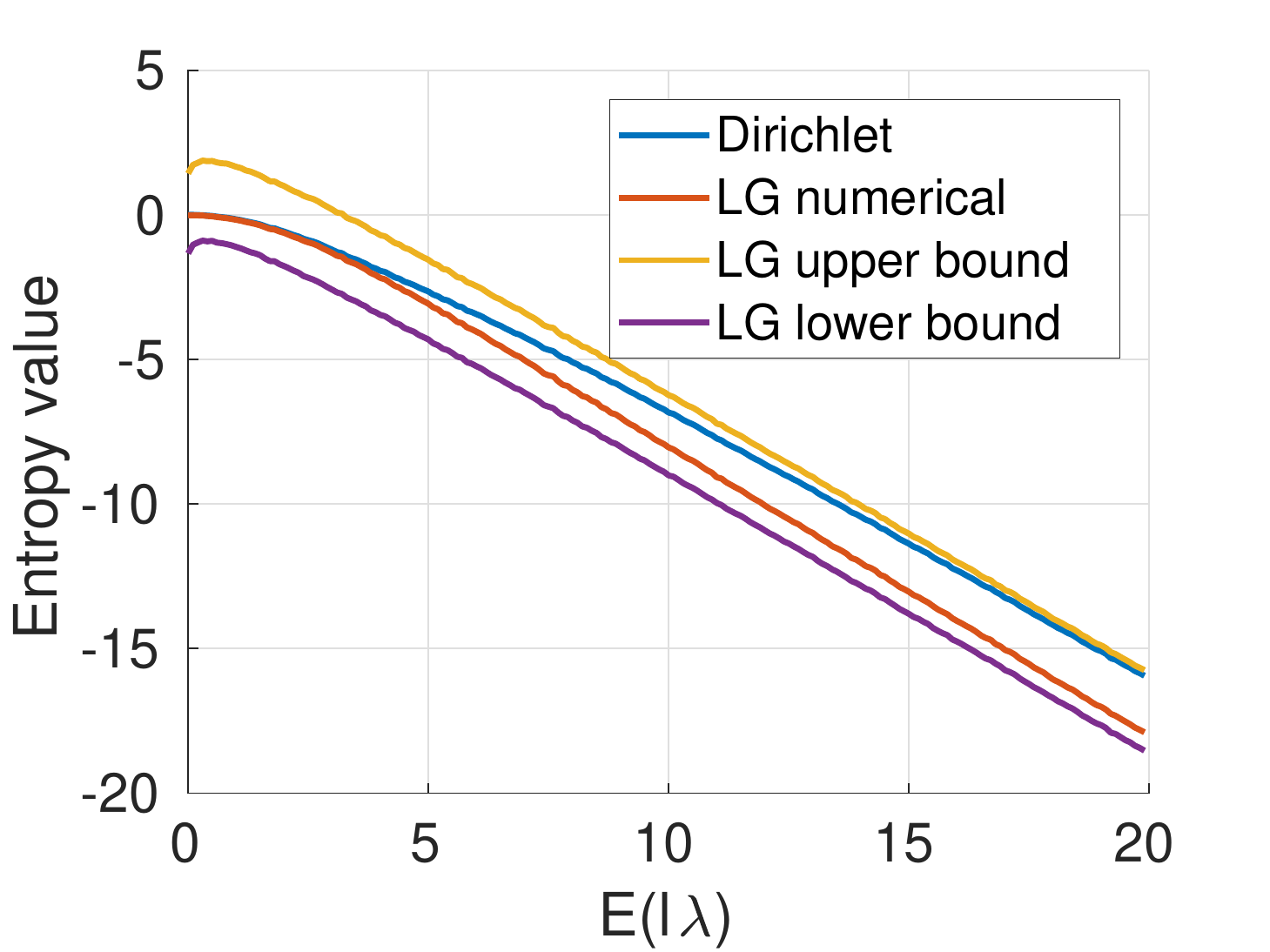}
		\caption{Entropy value}\label{fig:Fig_exp_ent}
	\end{subfigure}
	\begin{subfigure}[b]{0.18\textwidth}
		\includegraphics[width=\textwidth]{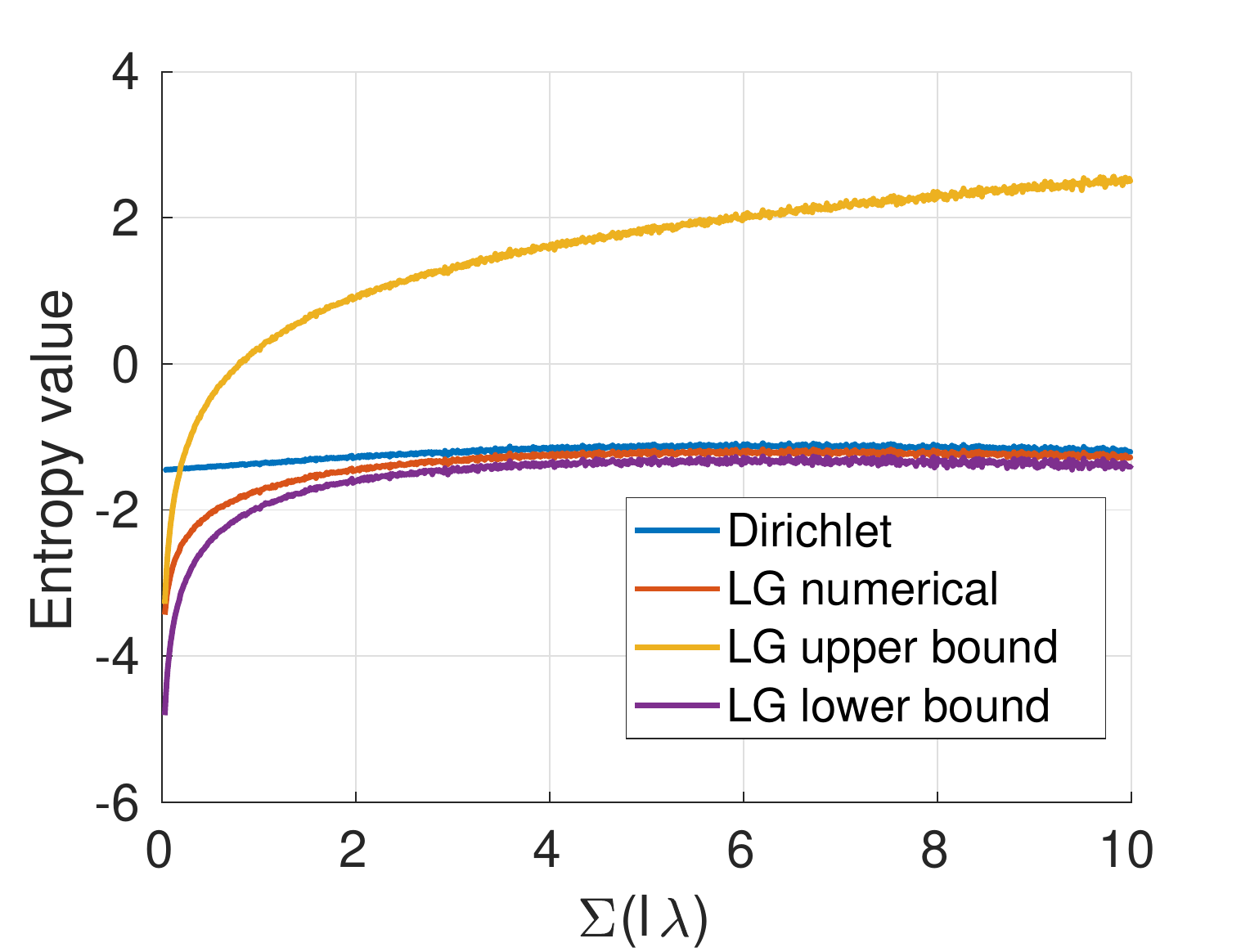}
		\caption{Entropy value}\label{fig:Fig_cov_ent}
	\end{subfigure}
	%
%	\begin{subfigure}[b]{0.23\textwidth}
%		\includegraphics[width=\textwidth]{Log_likelihood_graph.pdf}
%		\caption{Entropy value}\label{fig:Log_likelihood_graph}
%	\end{subfigure}
	
	\caption{This figure presents a comparison between Dirichlet and Logistical Gaussian (denoted LG) in terms of computational time, entropy, and log-likelihood values. \textbf{(a)} presents a computational time comparison between Dirichlet and LG for inference as a function of probability vector dimension. Similarly \textbf{(b)} presents a computational time for entropy computation, both for numerical and upper bound. \textbf{(c)} presents a value comparison between the different entropy computations. Note that in \textbf{(b)} and \textbf{(c)} the plots are given the parameters calculated for \textbf{(a)}. \textbf{(d)} presents an entropy value comparison between distributions with different expectations with a fixed covariance; Similarly, \textbf{(e)} presents an entropy value comparison between distributions with different covariance value with a fixed expectation. }
	\label{fig:Dir_LG_comp}
\end{figure*}

When considering the reward $H(\lambda)$ we have to consider two steps:
\begin{enumerate}
	
	\item Computation of $b[\lambda]$ parameters; With \mh we maintain separately $\lambda_w \in \{ \lambda_w \}_{w \in W}$ and subsequently describe $b[\lambda]$ using $\{\{ \lambda_w \}_{w \in W}\}$. As such, to compute $H(\lambda)$ we must assume a distribution for $b[\lambda]$ and infer its parameters. \jlp on the other hand limits $b[\lambda]$ to be LG distributed.
	
	\item Calculation of $H(\lambda)$ to use as a reward function for planning, either via numerical computation, or by using bounds.
	
\end{enumerate}
We remind that this discussion is relevant for $H(\lambda)$ computation for \mh, as in \jlp $\lambda$ is LG distributed by definition.

Parameter inference for Dirichlet is faster than for LG although the process is numeric. On the other hand, LG is more expressive; For $m$ classes, Dirichlet distribution has $m$ parameters, while LG has $m-1$ parameters for $\mathbb{E}(\loglambda)$, and $\frac{m \cdot (m-1)}{2}$ parameters for $\Sigma(\loglambda)$, totaling in $(1 + \frac{m}{2}) \cdot (m-1)$ parameters.

LG does not have an analytical solution for computing entropy, and its numeric computation is slower than Dirichlet's. On the other hand, computing bounds of entropy for LG is comparable in terms of computational effort to computing Dirichlet entropy. One must note that while Dirichlet is less computationally expensive than LG, when \mh with Dirichlet and \jlp are compared, \jlp is still computationally much more efficient.

Fig.~\ref{fig:Dir_LG_comp} presents a comparison between the two distributions in terms of computational effort and value of the entropy. In all figures, the x-axis is number of candidate classes, and for each, a dataset of 1000 class probability vectors was sampled. Fig.~\ref{fig:Fig_time_inf} presents the measured time of parameter computation, clearly showing an advantage for Dirichlet distribution for high dimensional probability vectors despite the parameter computation process for Dirichlet distribution containing functions that must be numerically computed. 

Fig.~\ref{fig:Fig_time_ent} presents the computational time of the entropy, for numerical computation for both LG and Dirichlet, and the bounds for LG(computation time is identical both for upper and lower bounds; thus only upper bound is shown). Here entropy computation for Dirichlet holds a significant advantage over numerical computation of entropy for LG, and the bound computation time is comparable to Dirichlet.

Fig.~\ref{fig:Val_fig_ent} presents  entropy values for Dirichlet, numerical LG, lower and upper bounds for LG. In general, the upper bound tends to be close to the numerical solution for fewer candidate classes. In addition, entropy for Dirichlet distribution tends to be higher.

In Fig.~\ref{fig:Fig_exp_ent} the number of candidate classes is fixed to two, i.e. the dimension of $\loglambda$ is $\mathbb{R}^1$; The covariance $\Sigma(\loglambda)$ is fixed at 3, and $\mathbb{E}(\loglambda)$ goes from 0 to 20. In this figure the entropy value monotonically decreases when increasing $\mathbb{E}(\loglambda)$. The lower bound is tighter between the two bounds as $\mathbb{E}(\loglambda)$ increases. 

In Fig.~\ref{fig:Fig_cov_ent} $\mathbb{E}(\loglambda)$ is fixed instead at $\mathbb{E}(\loglambda)=3$ and $\Sigma(\loglambda)$ varies between 0 and 10. We can see that the entropy value increases with the increase in $\Sigma(\loglambda)$, but the bigger effect is for LG compared to Dirichlet distribution. The upper bound is tighter at lower $\mathbb{E}(\loglambda)$ values, while the lower bound is tighter for higher values.

One may ask: which distribution should be used? For JLP, as mentioned previously, we are limited to LG. For \mh, the tradeoff is between distribution expressiveness and computational effort; While Logistical Gaussian is more expressive because of a larger number of parameters, the computation effort is significantly higher than for Dirichlet distribution. Also, Dirichlet distribution, unlike the Logistical Gaussian, can manage a very small number of probability vector samples.

%\begin{table*}
%	\centering
%	\begin{tabular}{p{0.22\textwidth}p{0.35\textwidth}p{0.35\textwidth}}
%		
%		\hline
%		& Dirichlet & LG \\
%		\hline
%		
%		Inference and number of parameters &
%		Less parameters, shorter inference although numerical &
%		More parameters, slower inference although analytical \\[0.2cm] 
%		
%		Entropy computation &
%		Faster computation of exact entropy &
%		Slower computation of exact entropy, 
%		computation of bounds as fast as Dirichlet. \\[0.2cm] 
%		
%		Small numbers of 
%		samples &
%		No problems &
%		Covariance matrix may get singular \\[0.3cm] 
%		
%		Possible use for $b[\lambda]$&
%		Can be used in MH, but not JLP &
%		Can be used in MH, must be used in JLP \\
%		
%		\hline
%		
%	\end{tabular}
%\caption{Table that summarizes the main points of Sec.~\ref{sec:Dir_LG_discussion}.}
%\label{tab:Summary_Table}
%\end{table*}

	\section{Experiments}
	\label{sec:experiments}
	% !TeX root = Paper Main.tex

We evaluate our approaches for semantic SLAM inference and planning in simulation (Sec.~\ref{sec:Sim}) and an experiment (Sec.~\ref{sec:Exp}) over the Active Vision Dataset scenario Home-3-01 \cite{Ammirato17icra}, with viewpoint dependent classifier uncertainty models trained using the BigBIRD dataset \cite{Singh14icra}. We considered environments with multiple spatially scattered objects, and the robot's task is to accurately classify them while localizing. Our implementation uses the GTSAM library \cite{Dellaert12tr} with a Python wrapper. The hardware used is an Intel i7-7700 processor running at 2.8GHz and 16GB RAM, with GeForce GTX 1050Ti with 4GB RAM.

\subsection{Compared Approaches and Metrics}\label{sec:Sim_Exp_Approaches_Metrics}

We consider three  approaches for inference and planning:  our \mh and \jlp methods with the corresponding \mhbsp and \jlpbsp, and an approach that does not consider model uncertainty, denoted as Without Epistemic Uncertainty (\weu). In this approach we maintain a single hybrid belief and use it for inference and planning, similar to approaches presented in \cite{Tchuiev19iros,Patten18arj}.

% Shared

We require a metric to evaluate classification where a completely incorrect classification would not result in infinite error, unlike cross entropy loss. Therefore, our approach is evaluated for classification accuracy using the Mean Square Detection Error metric (MSDE, also used by Teacy et al. \cite{Teacy15aamas} and Feldman \& Indelman \cite{Feldman18icra}). Given $b[\lambda_k]$,  MSDE is defined as follows:
\begin{equation}\label{eq:MSDE}
MSDE \triangleq \frac{1}{m} \sum_{i=1}^m \left( \lambda_{gt}^i - \mathbb{E}(\lambda^i_k) \right)^2,
\end{equation}
where $\lambda_{gt}^i$ is the ground truth probability of the object being of class $c=i$, and is equal to 1 if the object is class $i$ and 0 otherwise. For a completely incorrect classification $MSDE \leq 1$, ideal classification $MSDE = 0$, and for classification results where all class probabilities are equal, $MSDE = \frac{m-1}{m^2}$.

%---------------------------------------
\subsection{Simulation}\label{sec:Sim}

\subsubsection{Simulation Setting}\label{sec:Inf_Stat_Setting}

% Simulation setting

We consider a closed set setting and assume, for simplicity, that the number of classes $m=2$, i.e., each object can be one of the two classes. The camera senses objects up to 10 meters distance, with an opening angle of $120^\circ$. We choose two sets of models for the simulation; The first is a model that satisfies Lemma \ref{lemma:FGCondition}, and the second does not to show the effect of using \jlp with such models. 
%In both cases, we assume the models are accurate \VI{for measurement simulation [clarify]}.
The baseline MSDE score for $m=2$ where all class probabilities are equal is $MSDE = 0.25$.

% Shared
%--------------------------------------------------------
\subsubsection{Inference: Single Run}\label{sec:Inf_Fixed}

The setting for this comparison is an environment with 5 objects; These object are placed within the environment, which is presented in Fig.~\ref{fig:Inf_Fixed_GT} along with the ground truth trajectory. The robot passes through an area in which the objects have classification scores with high degree of epistemic uncertainty. Normally, with methods that do not consider epistemic uncertainty, classification results will have a high chance of being incorrect, but our approach provide more accurate results as it considers epistemic uncertainty.
Denote $\psi$ as the relative orientation between the object's orientation (chosen during the classifier uncertainty model training) and the camera's pose. We simulate a classifier model that considers the following cases:
\begin{enumerate}
	
	\item The classifier differentiates well between classes with low epistemic uncertainty, $\psi = 0^\circ$.
	
	\item The classifier does not differentiate well between the two classes, $\psi = 90^\circ, 270^\circ$.
	
	\item The classifier differentiates between classes well, but with high epistemic uncertainty, $\psi = 180^\circ$.
	
\end{enumerate}
As such, $\psi = 0^\circ$ is the relative orientation where the best classification with the lowest uncertainty is expected (corresponding to the blue cone in Fig.~\ref{fig:Inf_Fixed_GT} that represent this relative orientation), and $\psi = 180^\circ$ is the relative orientation that most prone to classification errors when not considering epistemic uncertainty. Considering the specific ground truth trajectory for the presented scenario, objects 1 and 2 represent case 3; as such, we expect our approaches to infer the correct class within a large number of steps because of the uncertainty. Object 3 represents case 1, and as such when it is observed the classification will be accurate on the first view. Objects 4 and 5 represent case 2, where classification is difficult as the model doesn't differentiate well between the classes of those objects. The object ground truth classes are $c=1$ for objects 1, 2 and 5, and $c=2$ for objects 3 and 4.

\begin{figure}[!htbp]
	
	\includegraphics[width=0.6\textwidth]{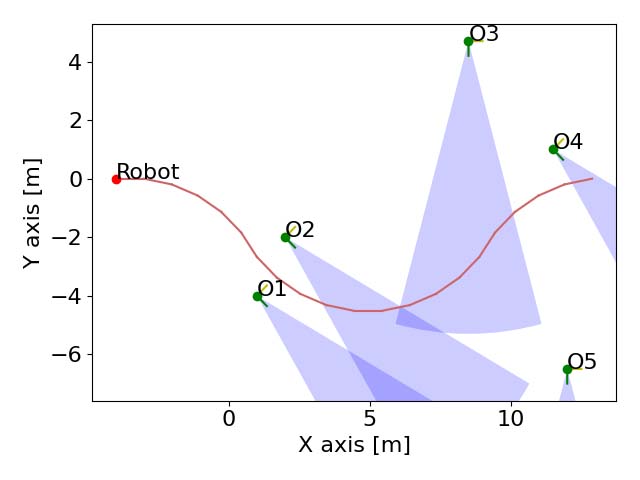}
	\caption{The ground truth of the scenario in Sec.~\ref{sec:Inf_Fixed}. The red dot represents the robot's starting point, with the red curve being the path. The green dots represent the objects' location with the corresponding object labels. The green line represents the object orientation, with the yellow line present $90^\circ$ of that orientation. The blue cones represent the observation viewpoints in which the classifier identifies the object class well with low uncertainty, i.e. case 1.}%
	\label{fig:Inf_Fixed_GT} 
	
\end{figure}

A visualization of the models presented can be seen in Fig.~\ref{fig:cls_model_1}.

\begin{figure}[!htbp]
	
	\begin{subfigure}[b]{0.35\textwidth}
		\includegraphics[width=\textwidth]{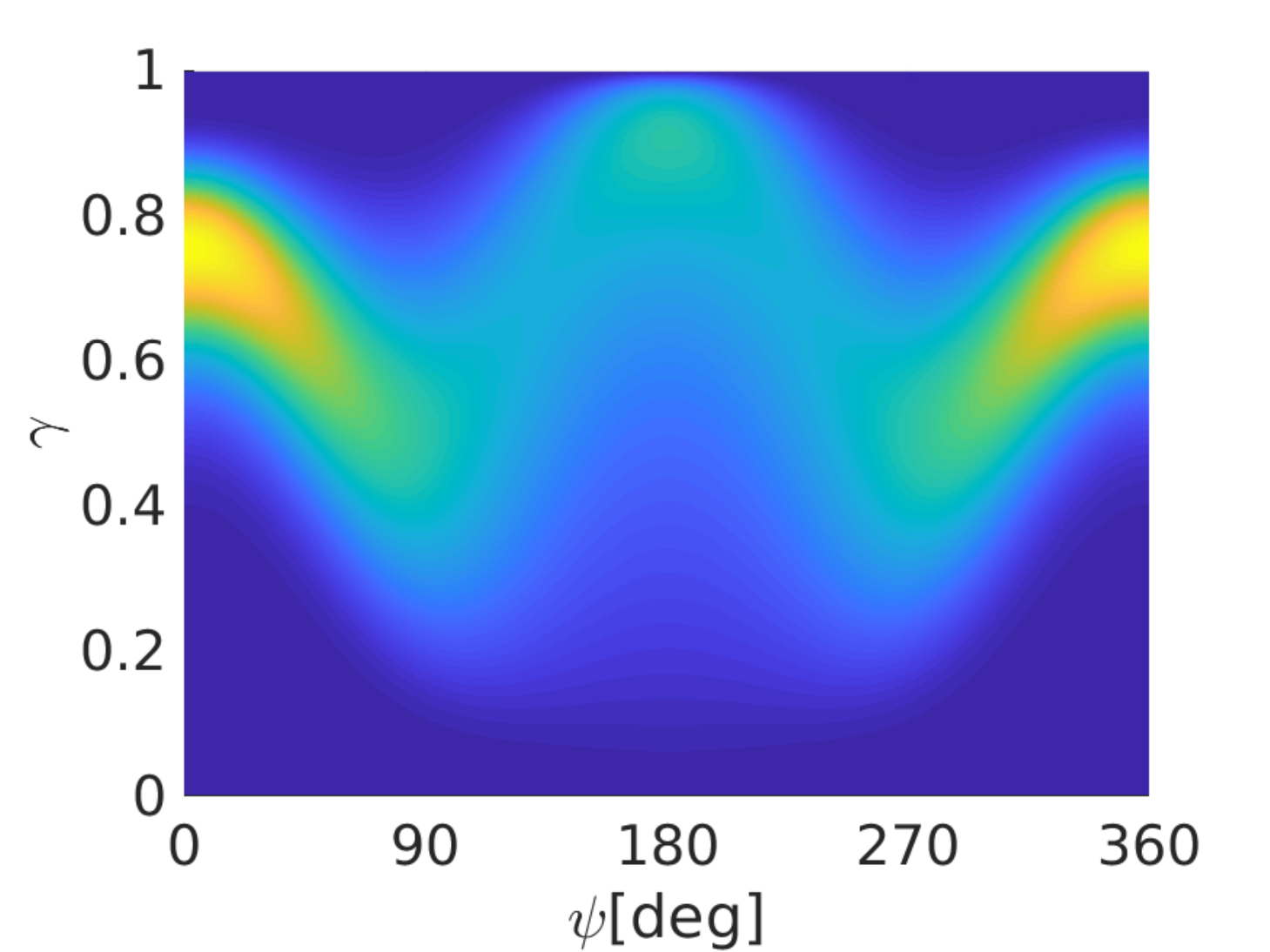}
		\caption{$\prob{\gamma^{c=1}|c=1,\psi}$}\label{fig:cls_model_1_c1}
	\end{subfigure}
	\begin{subfigure}[b]{0.35\textwidth}
		\includegraphics[width=\textwidth]{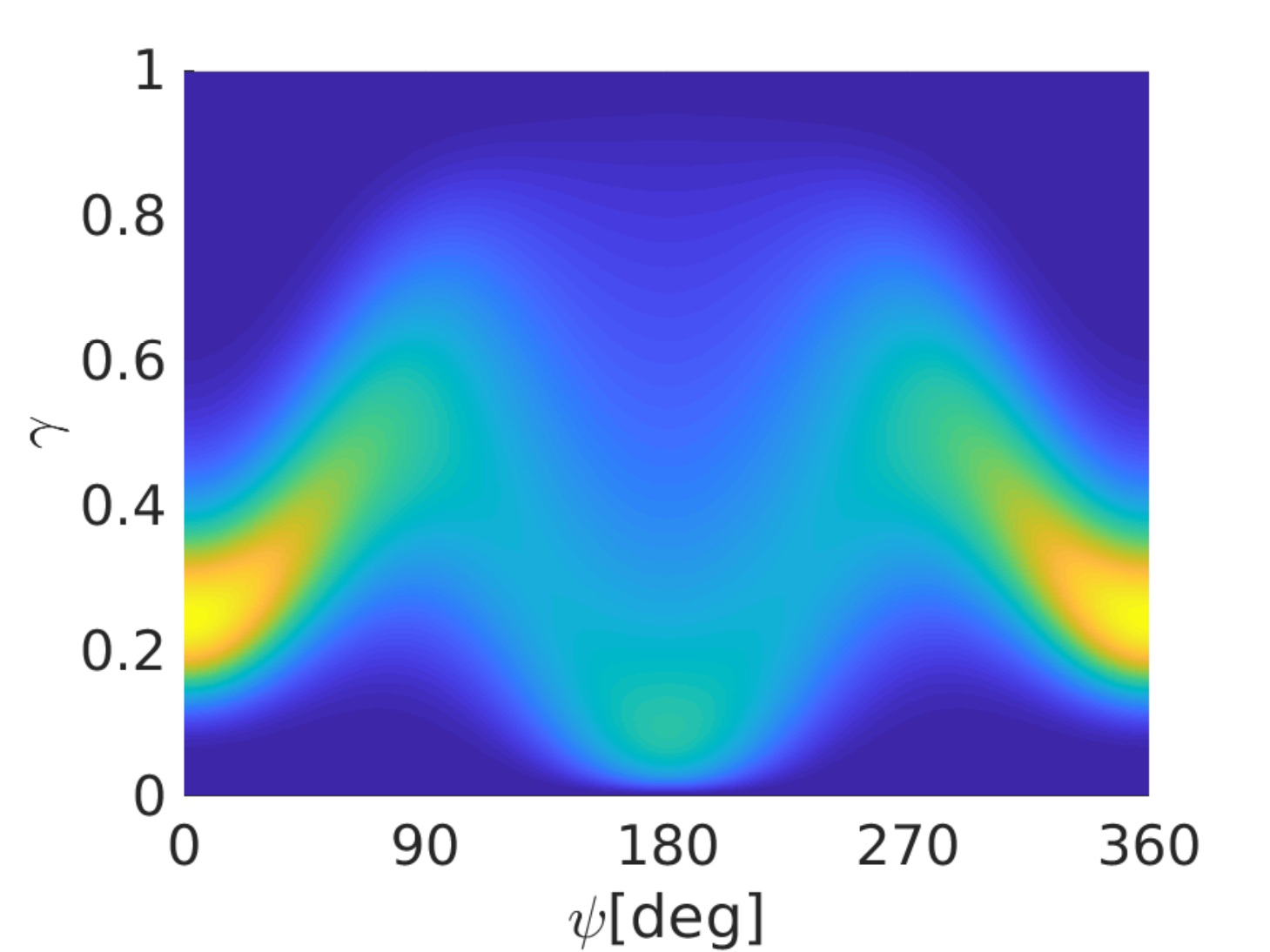}
		\caption{$\prob{\gamma^{c=1}|c=2,\psi}$}\label{fig:cls_model_1_c2}
	\end{subfigure}
	\caption{A visualization of the classifier uncertainty model used in Sec.~\ref{sec:Inf_Fixed}. We present the value of $\prob{\gamma^{c=1}|c,x^{rel}} \equiv \prob{\gamma^{c=1}|c,\psi}$ as a function of relative orientation $\psi$ and $\gamma^{c=1}$ value, for classes $c=1$ and $c=2$ in \textbf{(a)} and \textbf{(b)} respectively. Blue and yellow colors correspond to low and high PDF values respectively.}% in green frames.}
	\label{fig:cls_model_1}
\end{figure}

We consider noisy geometric measurements of relative pose, and cloud point semantic measurements, i.e. the classifier gives $\{\gamma\}$ per each object, sampled from the classifier uncertainty model.
We use a classifier uncertainty model with the following function for expectation (see Eq.~\eqref{eq:Classifier_Uncertainty_Model}):
\begin{equation}\label{eq:Model_1_exp}
\begin{split}
h_{c=1}(x^{rel}) & = \frac{1}{2} \cos(2 \cdot \psi) + \frac{1}{2} \\
h_{c=2}(x^{rel}) & = - \frac{1}{2} \cos(2 \cdot \psi) - \frac{1}{2},
\end{split}
\end{equation}
and the following parameter for root-information:
\begin{equation}\label{eq:Model_1_cov}
R_{c=1}(x^{rel}) = R_{c=2}(x^{rel}) = 1.4 + 0.6 \cdot \cos(\psi),
\end{equation}
Subsequently, the covariance parameter from Eq.~\eqref{eq:Classifier_Uncertainty_Model} in the two class case is computed as follows:
\begin{equation}
\Sigma_c(x^{rel}) = \sqrt{\frac{1}{R_{c}(x^{rel})}}.
\end{equation}
With the covariance parameter being equal, the presented model satisfies the assumption of Lemma \ref{lemma:FGCondition}, allowing us to use the \jlp approach.

Fig.~\ref{fig:Inf_Fixed_MSDE_Ind} presents MSDE results for each object separately. We perform inference with \mh with a different number of hybrid beliefs, and compare it to \jlp and \weu. With \mh and \jlp, the class of objects 1 and 2 is inferred using multiple observations, eventually inferring the correct class. The class of object 3, once seen, is quickly and accurately inferred. The class of objects 4 and 5 remain ambiguous (MSDE of approximately 0.25) because they are observed from viewpoints that correspond to case 2. In general, \mh in Fig.~\ref{fig:Inf_Fixed_MSDE_Ind_MH05}-\ref{fig:Inf_Fixed_MSDE_Ind_MH100} tends to present smoother results the more hybrid beliefs are used, and also compared to \jlp in Fig.~\ref{fig:Inf_Fixed_MSDE_Ind_JLP} where for each time step the entropy must be computed numerically from new $\lambda_k$ samples. \weu in Fig.~\ref{fig:Inf_Fixed_MSDE_Ind_WEU} shows that objects can be classified incorrectly if not considering epistemic uncertainty, such as object 4, as shown in the figure.

As a summary, Fig.~\ref{fig:Inf_Fixed_MSDE} presents average MSDE results for all the objects combined, showing that epistemic uncertainty aware approaches outperform \weu, while \mh with 10 beliefs and \jlp perform similarly. Fig \ref{fig:Inf_Fixed_Time} presents a computation time comparison between \weu, \jlp and \mh for different number of hybrid beliefs. From this figure, we can see that \jlp is comparable to \weu, with \mh being significantly more computationally intensive as the number of the simultaneous beliefs increase.

\begin{figure*}[!htbp]
	
	\begin{subfigure}[b]{0.32\textwidth}
		\includegraphics[width=\textwidth]{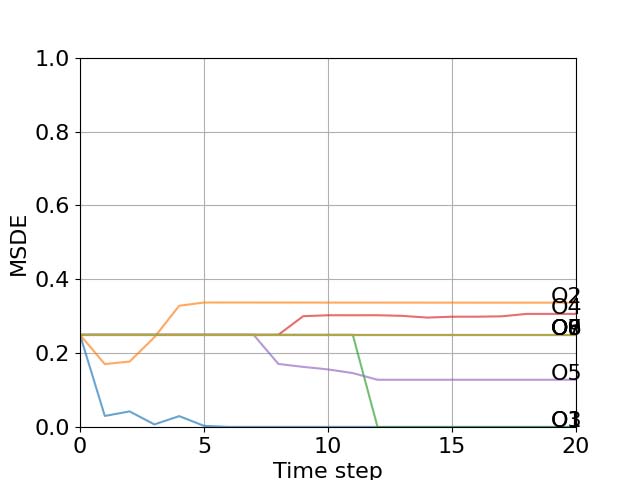}
		\caption{\mh 5 hybrid beliefs}\label{fig:Inf_Fixed_MSDE_Ind_MH05}
	\end{subfigure}
	\begin{subfigure}[b]{0.32\textwidth}
		\includegraphics[width=\textwidth]{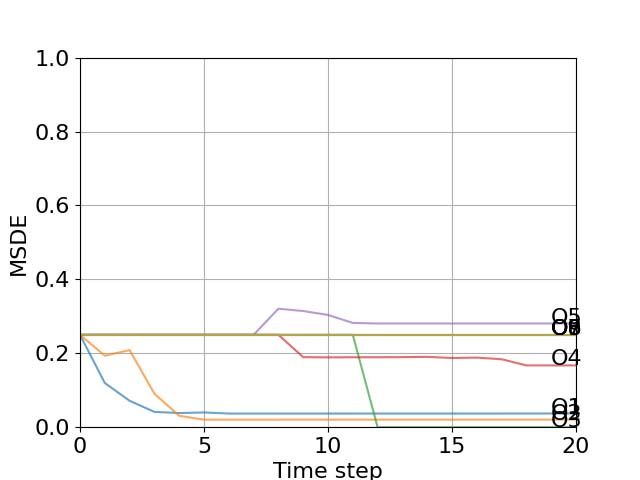}
		\caption{\mh 10 hybrid belief}\label{fig:Inf_Fixed_MSDE_Ind_MH10}
	\end{subfigure}
	\begin{subfigure}[b]{0.32\textwidth}
		\includegraphics[width=\textwidth]{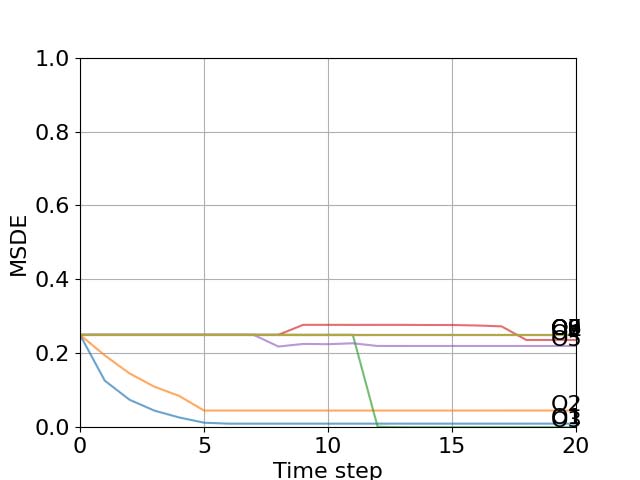}
		\caption{\mh 25 hybrid beliefs}\label{fig:Inf_Fixed_MSDE_Ind_MH25}
	\end{subfigure}
	
	\begin{subfigure}[b]{0.32\textwidth}
		\includegraphics[width=\textwidth]{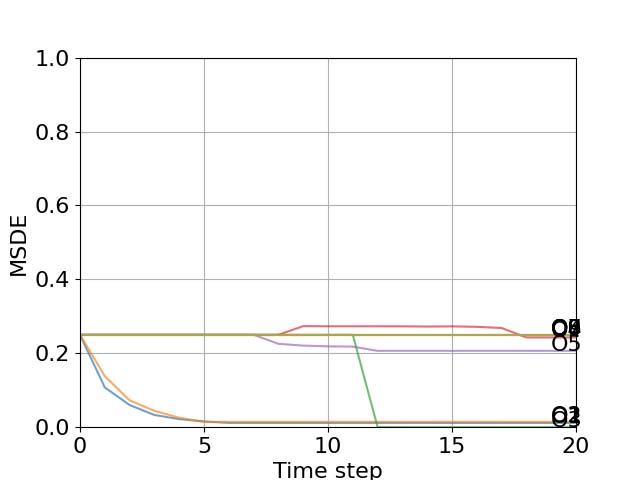}
		\caption{\mh 100 hybrid beliefs}\label{fig:Inf_Fixed_MSDE_Ind_MH100}
	\end{subfigure}
	\begin{subfigure}[b]{0.32\textwidth}
		\includegraphics[width=\textwidth]{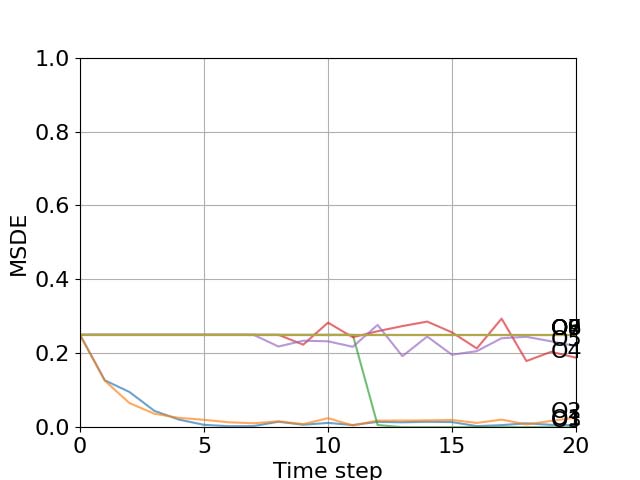}
		\caption{\jlp}\label{fig:Inf_Fixed_MSDE_Ind_JLP}
	\end{subfigure}
	\begin{subfigure}[b]{0.32\textwidth}
		\includegraphics[width=\textwidth]{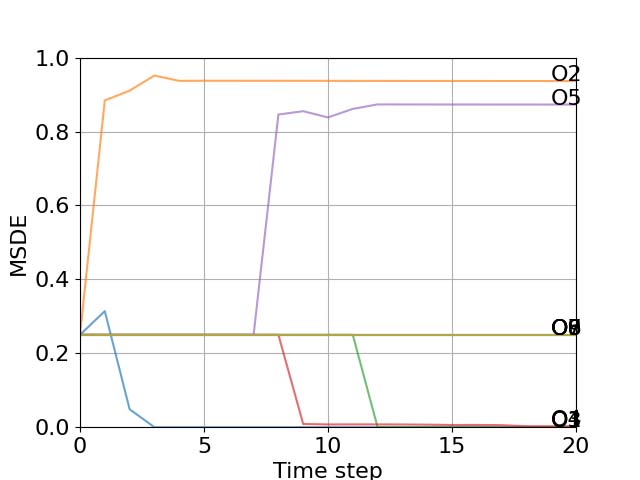}
		\caption{\weu}\label{fig:Inf_Fixed_MSDE_Ind_WEU}
	\end{subfigure}
	\caption{\textbf{(a)}, \textbf{(b)}, \textbf{(c)}, and \textbf{(d)} show MSDE results per time step for \mh per object, each in a different color, for 5, 10, 25, and 100 respectively. \textbf{(e)} shows MSDE results for \jlp. \textbf{(f)} shows MSDE results for \weu}
	\label{fig:Inf_Fixed_MSDE_Ind}
\end{figure*}

\begin{figure}[!htbp]
	
	\begin{subfigure}[b]{0.35\textwidth}
		\includegraphics[width=\textwidth]{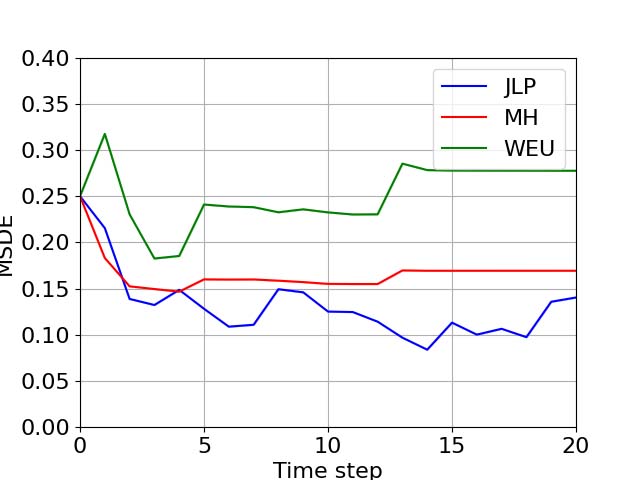}
		\caption{}\label{fig:Inf_Fixed_MSDE}
	\end{subfigure}
	\begin{subfigure}[b]{0.35\textwidth}
		\includegraphics[width=\textwidth]{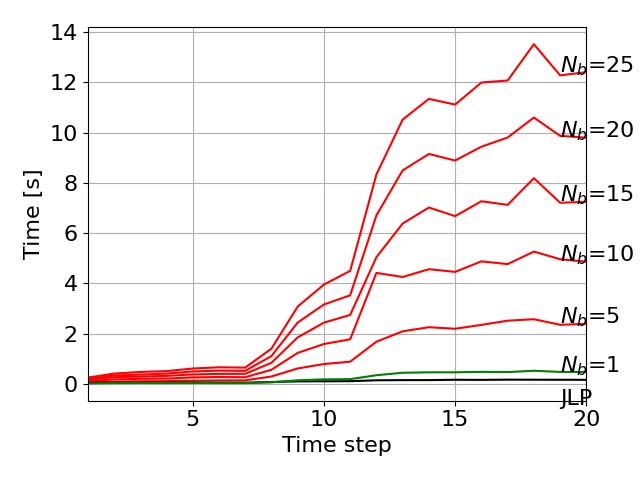}
		\caption{}\label{fig:Inf_Fixed_Time}
	\end{subfigure}
	\caption{\textbf{(a)} compares MSDE to time step between \mh in red, \jlp in blue, and without uncertainty in green.
	\textbf{(b)} compares run-time per inference step between realizations of \mh with different number of hybrid beliefs in red,  \jlp in black, and \weu in green.}% in green frames.}
	\label{fig:Inf_Fixed}
\end{figure}

%-------------------------------------------------------
\subsubsection{Inference: Statistical Study}\label{sec:Inf_Stat_Study}

In this section we perform a Monte-Carlo study to compare between \mh, \jlp, and \weu. We run the simulation 10 times and present results for MSDE and computational time.
The setting for this comparison is an environment with 5 objects with randomized poses, with examples presented in Fig.~\ref{fig:Inf_Scenario_Comp}. Otherwise, we use the same setting and classifier uncertainty model as in Sec.~\ref{sec:Inf_Fixed}.

\begin{figure}[!htbp]
	
	\begin{subfigure}[b]{0.35\textwidth}
		\includegraphics[width=\textwidth]{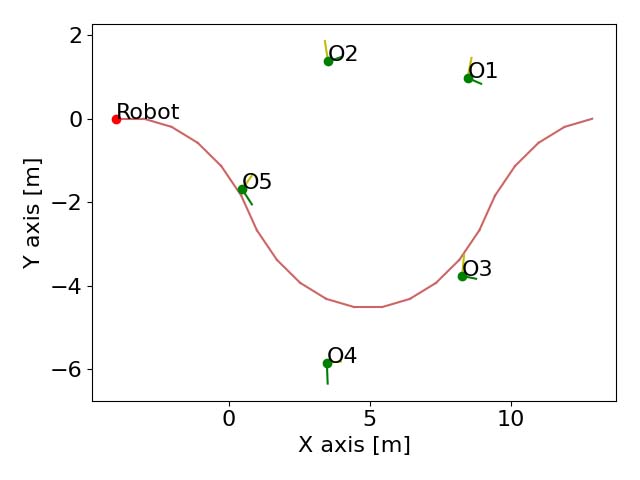}
		\caption{}\label{fig:Inf_Scenario}
	\end{subfigure}
	\begin{subfigure}[b]{0.35\textwidth}
		\includegraphics[width=\textwidth]{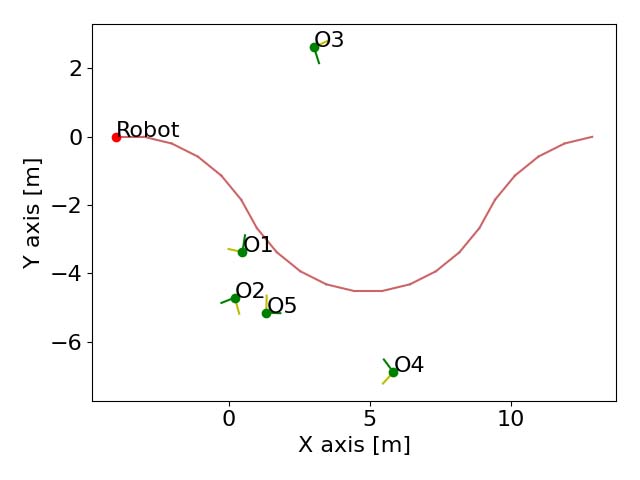}
		\caption{}\label{fig:Inf_Scenario_2}
	\end{subfigure}
	\caption{Examples of a sampled environment in which the inference is performed. The red trajectory is the robot path. The green dots denote the objects, numbered O1 to O5. The green and yellow lines represent their orientation, $0^\circ$ and $90^\circ$ respectively.}% in green frames.}
	\label{fig:Inf_Scenario_Comp}
\end{figure}

We present MSDE statistical results in Fig.~\ref{fig:Inf_Stat_MSDE}, with one $\sigma$ uncertainty. While \mh and \jlp perform similarly, both outperform the approach that does not consider epistemic uncertainty, especially in cases where the camera goes through areas that correspond to $\psi=180^\circ$. Fig.~\ref{fig:Inf_Comp_Time} presents run-time results for the algorithms. Expectedly, as the number of simultaneous beliefs increase for \mh, the algorithm runs slower. \jlp is comparable to maintaining a single hybrid belief in this case, demonstrating that it is more practical when the conditions of Lemma \ref{lemma:FGCondition} are satisfied.

\begin{figure}[!htbp]
	
	\begin{subfigure}[b]{0.35\textwidth}
		\includegraphics[width=\textwidth]{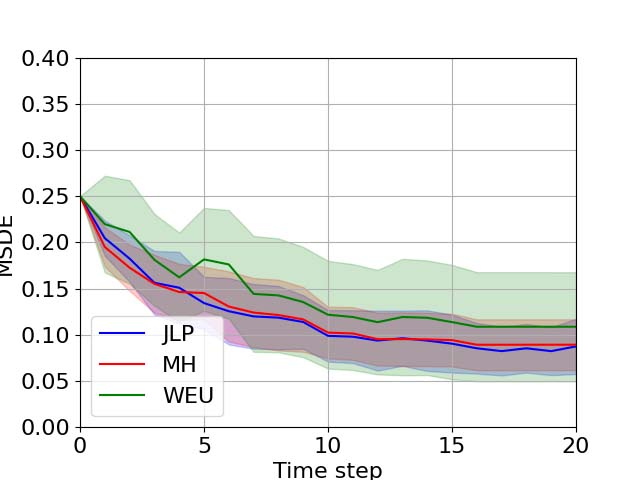}
		\caption{}\label{fig:Inf_Stat_MSDE}
	\end{subfigure}
	\begin{subfigure}[b]{0.35\textwidth}
		\includegraphics[width=\textwidth]{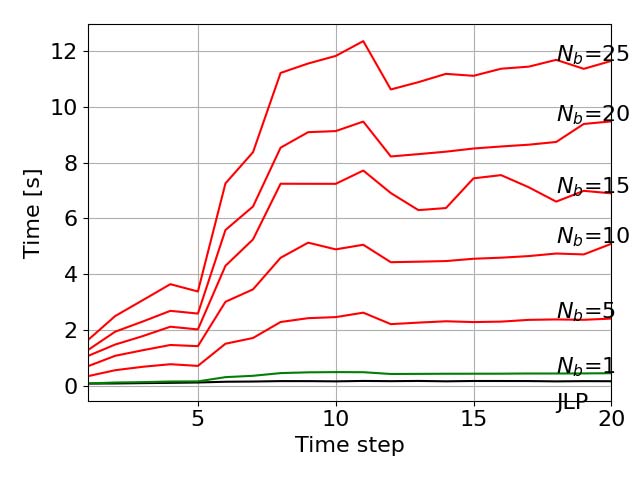}
		\caption{}\label{fig:Inf_Comp_Time}
	\end{subfigure}
	\caption{\textbf{(a)} compares MSDE to time step between \mh in red, \jlp in blue, and without uncertainty in green. The transparent colors correspond to the respective plot in one $\sigma$ value. \textbf{(b)} compares run-time per inference step between realizations of \mh with different number of hybrid beliefs in red,  \jlp in black, and \weu in green.}% in green frames.}
	\label{fig:Inf_Comp}
\end{figure}

\subsubsection{Inference: Joint Lambda Pose Assumption}\label{sec:Inf_Stat_JLP_Conf}

One may consider the ramifications of using \jlp with models that don't satisfy Lemma \ref{lemma:FGCondition}; The most straightforward result is that \mh and \jlp results don't coincide with each other, and that may result in either erroneous or overconfident classification (i.e. large values of $\mathbb{E}(\loglambda)$ and/or too small values of $\Sigma(\loglambda)$).

This time, the classifier model uses the following parameters; For expectation:
\begin{equation}\label{eq:Model_2_exp}
\begin{split}
	h_{c=1}(x^{rel}) & = 0.3 \cdot \cos(2 \cdot \psi) + 0.3 \\
	h_{c=2}(x^{rel}) & = - 0.3 \cdot \cos(2 \cdot \psi) - 0.3,
\end{split}
\end{equation}
and the following parameters for root-information:
\begin{equation}\label{eq:Model_2_cov}
\begin{split}
	R_{c=1}(x^{rel}) & = 1.4 + 0.6 \cdot \cos(\psi)\\
	R_{c=2}(x^{rel}) & = 1.4 - 0.6 \cdot \cos(\psi)	.
\end{split}
\end{equation}
The goal is creating opposing $\Sigma_c(\psi)$ for the two classes, such that Lemma \ref{lemma:FGCondition} does not hold and may result in inaccurate classification while using \jlp. Fig.~\ref{fig:cls_model_2} presents a visualization of the model. Specifically, the problematic areas are around $\psi = 0^\circ$ where the models actually predict that when $\gamma^{c=1}>0.8$ the likelihood is actually higher for $c=2$, and vice-versa when $\psi = 180^\circ$ when $\gamma^{c=2}>0.8$ predicts a higher likelihood for $c=1$.

\begin{figure}[!htbp]
	
	\begin{subfigure}[b]{0.35\textwidth}
		\includegraphics[width=\textwidth]{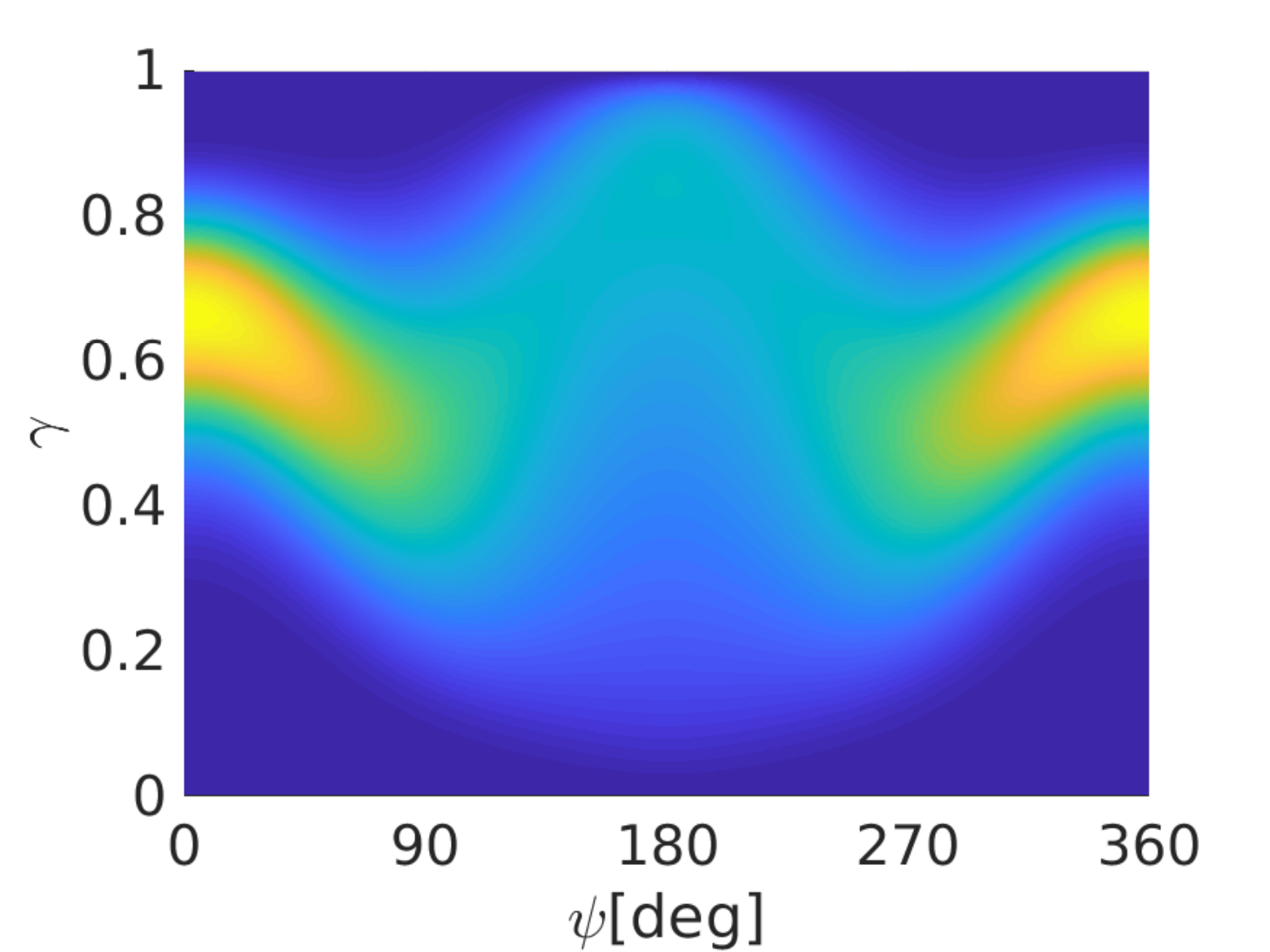}
		\caption{$\prob{\gamma^{c=1}|c=1,\psi}$}\label{fig:cls_model_2_c1}
	\end{subfigure}
	\begin{subfigure}[b]{0.35\textwidth}
		\includegraphics[width=\textwidth]{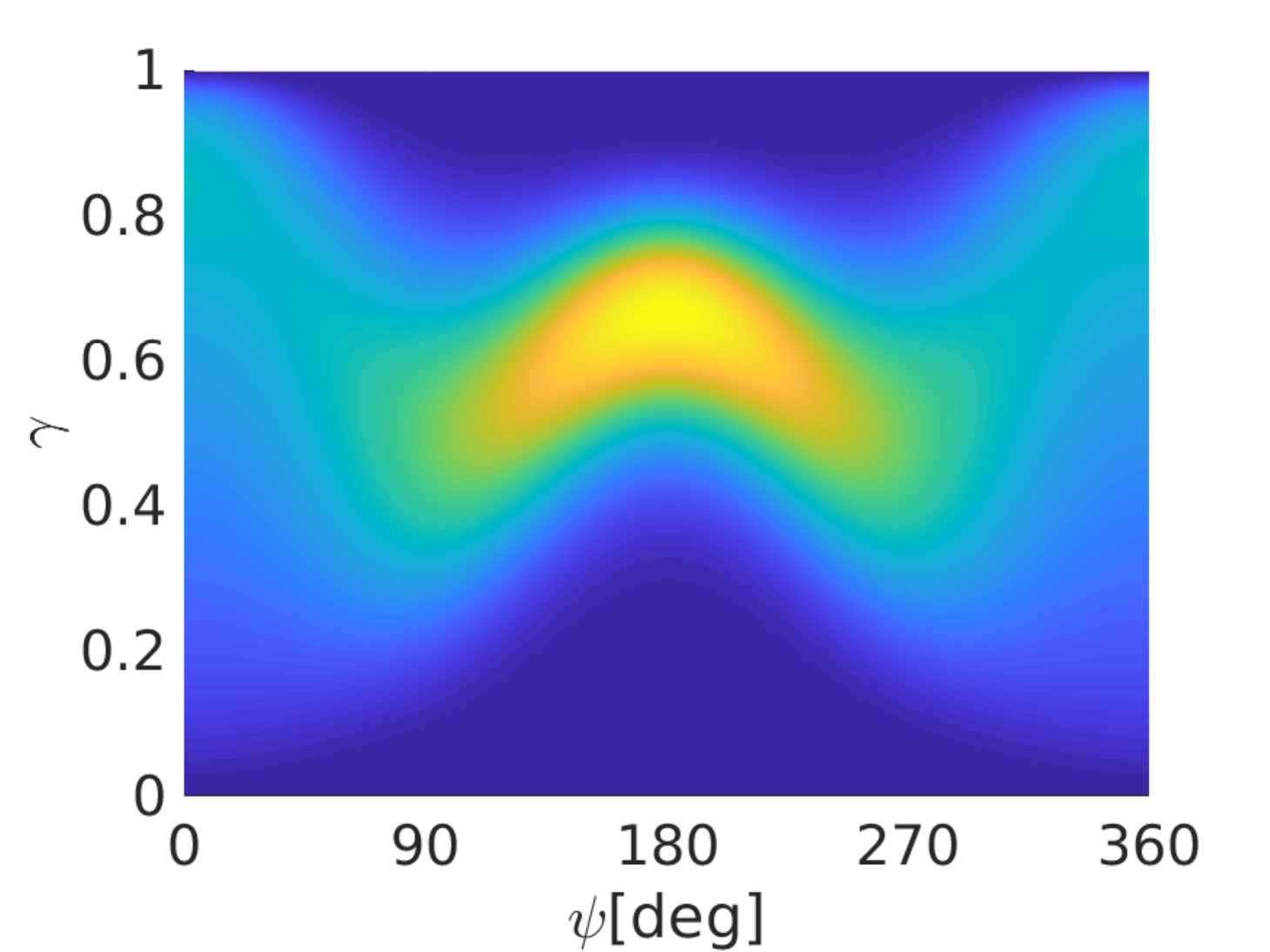}
		\caption{$\prob{\gamma^{c=1}|c=2,\psi}$}\label{fig:cls_model_2_c2}
	\end{subfigure}
	\caption{A visualization of the classifier uncertainty model used in Sec.~\ref{sec:Inf_Stat_JLP_Conf}. We present the value of $\prob{\gamma^{c=1}|c,\psi}$ as a function of relative orientation $\psi$ and $\gamma^{c=1}$ value, for classes $c=1$ and $c=2$ in \textbf{(a)} and \textbf{(b)} respectively. Blue and yellow colors correspond to low and high PDF values respectively.}% in green frames.}
	\label{fig:cls_model_2}
\end{figure}

Fig.~\ref{fig:JLP_Assumption_Histogram} presents the PDF values of $\loggamma_k$, and $\mathcal{L}^s_k$ with and without \jlp assumption at $\psi=0^\circ$. Here $\loggamma_k \sim \mathcal{N}(0.6, 0.25)$, coinciding with the classifier uncertainty model parameters of class $c=1$. In this figure the difference between the approximation and real $\mathcal{L}^s_k$ are evident, with the approximated $\mathcal{L}^s_k$ risking a larger chance of incorrect classification as the area below 0 is larger than that for the real $\mathcal{L}^s_k$.

\begin{figure}[!htbp]
	\includegraphics[width=0.5\textwidth]{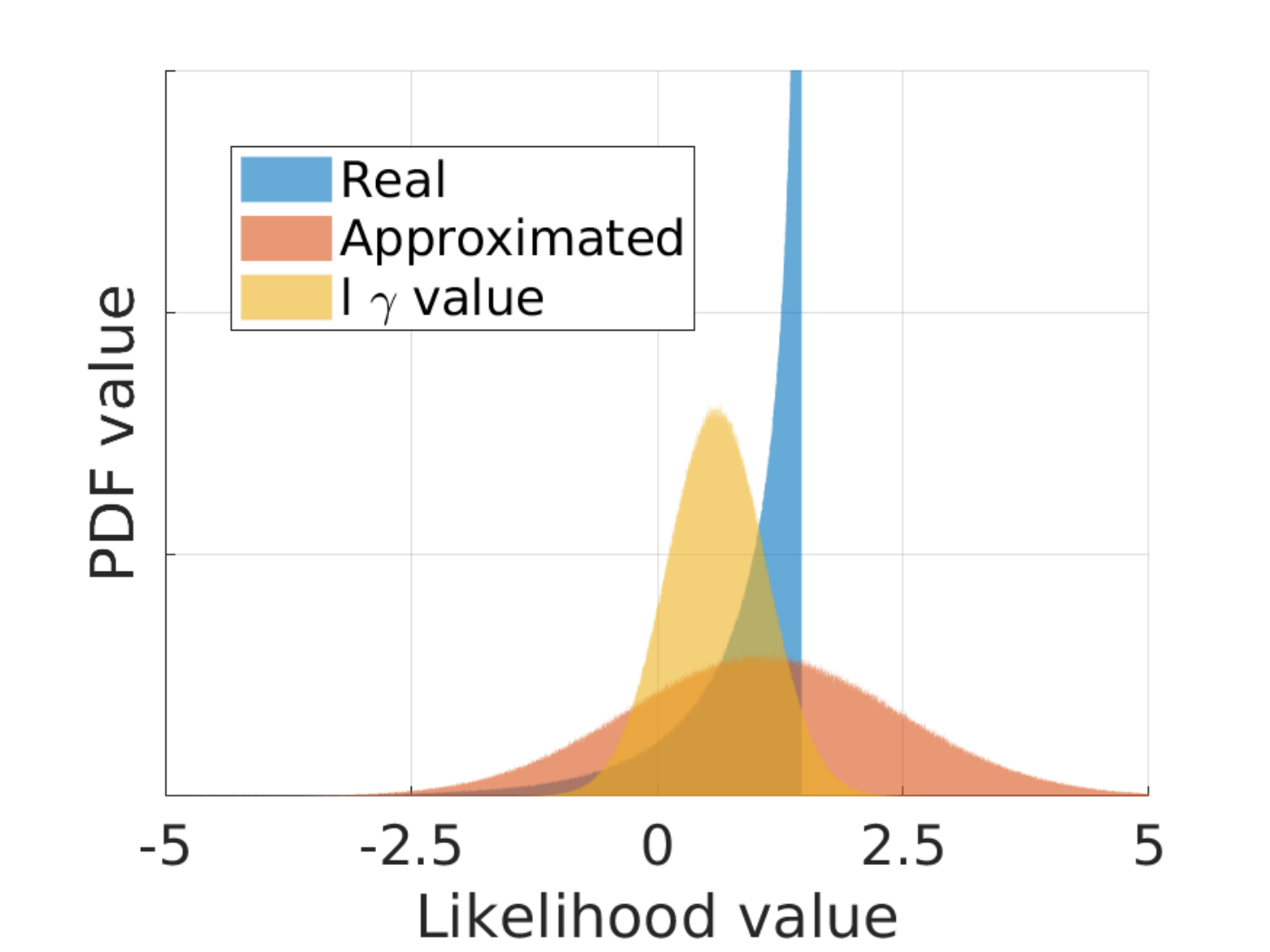}
	\caption{An approximate PDF value graph for \jlp assumption where $\psi = 0^\circ$ The yellow area represents the distribution of $\loggamma_k$, the red area represents $\mathcal{L}^s_k$ PDF when \jlp assumption is used, and the blue area represents the real PDF of $\mathcal{L}^s_k$.}% 
	\label{fig:JLP_Assumption_Histogram}
\end{figure}

In this scenario we use the same setting we used at Sec.~\ref{sec:Inf_Stat_Setting}, \ref{sec:Inf_Fixed}, and \ref{sec:Inf_Stat_Study}. and compare the MSDE scores in Fig.~\ref{fig:Inf_Conf_MSDE}. For Fig.~\ref{fig:Inf_Conf_MSDE_Ind_MH100} and \ref{fig:Inf_Conf_MSDE_Ind_JLP} we use the scenario from Sec.~\ref{sec:Inf_Fixed}.
In Fig.~\ref{fig:Inf_Conf_MSDE_Ind_MH100} MSDE results for \mh with 100 hybrid beliefs are shown as the most accurate, where objects 1 and 2 have more accurate classification than the rest. Fig.~\ref{fig:Inf_Conf_MSDE_Ind_JLP} presents MSDE results for \jlp, where we see significantly less accurate results. Finally, we perform a statistical study with 5 random object locations of 10 runs as in Sec.~\ref{sec:Inf_Stat_Study}, comparing between \mh with 10 hybrid belief, \jlp, and WEO, and see that statistically the difference between \mh and \jlp is not large even without Lemma \ref{lemma:FGCondition} holding, as opposed to the specific run from Figs.~\ref{fig:Inf_Conf_MSDE_Ind_MH100} and \ref{fig:Inf_Conf_MSDE_Ind_JLP}.

\begin{figure*}[!htbp]
	
	\begin{subfigure}[b]{0.32\textwidth}
		\includegraphics[width=\textwidth]{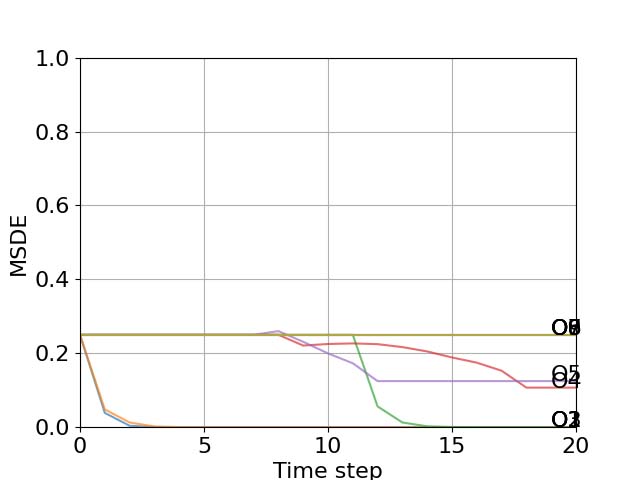}
		\caption{\mh 100 hybrid beliefs}\label{fig:Inf_Conf_MSDE_Ind_MH100}
	\end{subfigure}
	\begin{subfigure}[b]{0.32\textwidth}
		\includegraphics[width=\textwidth]{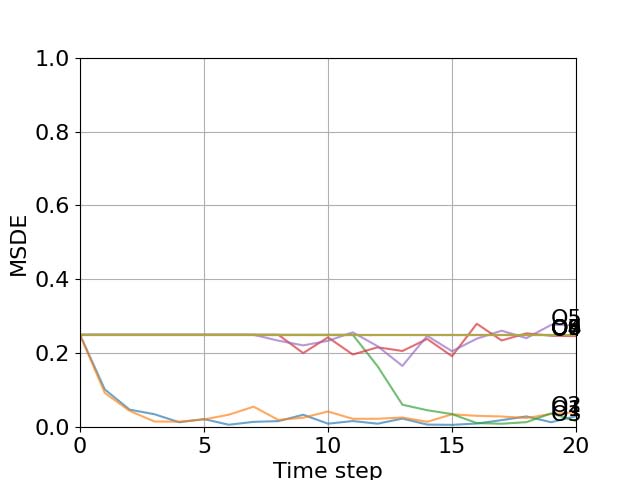}
		\caption{\jlp}\label{fig:Inf_Conf_MSDE_Ind_JLP}
	\end{subfigure}
	\begin{subfigure}[b]{0.32\textwidth}
		\includegraphics[width=\textwidth]{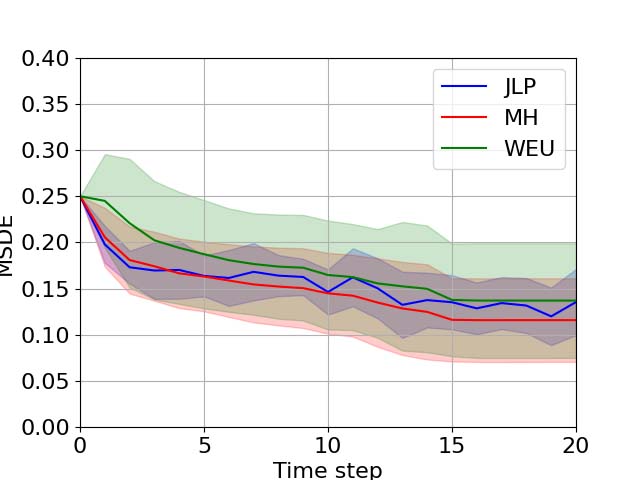}
		\caption{Average MSDE}\label{fig:Inf_Stat_MSDE_Conf}
	\end{subfigure}
	
	\caption{\textbf{(a)} shows MSDE results to time step per object for \mh with 100 hybrid beliefs which we consider most accurate, while \textbf{(b)} shows \jlp results for MSDE. \textbf{(c)} is a statistical study that compares between average MSDE over all objects for \weu in green, \jlp in blue, and \mh in red. The line corresponds to MSDE expectation and the colored area to one $\sigma$ range.}
	\label{fig:Inf_Conf_MSDE}
\end{figure*}

%------------------------------------------------------------------------
\subsubsection{Planning: Single Object}\label{sec:Plan_Single}

Next, we simulate a planning scenario of a single object using \sembsp. Relative to the object, there is an area with low epistemic uncertainty and high separation between classes, represented as a blue cone in Fig.~\ref{fig:Plan_Single_GT}.  We compare between two reward functions for planning. $R_1$ is the negative of the entropy of $\lambda$ as defined in Eq.~\eqref{eq:H_def}, while $R_2$ is the entropy of $\mathcal{E}(\lambda)$ as defined in Sec.~\ref{sec:Non_uncertainty_rewards}. For a future belief $b[\lambda_{k+1}]$:
\begin{equation}
\begin{split}
	R_1 & = - H(\lambda_{k+1})\\ 
	R_2 & = - H(\mathbb{E}(\lambda^o_{k+1}))\\ 
\end{split}
\end{equation}
For both reward functions we use \mhbsp and \jlpbsp. We use only $R_2$ for \weu as $R_1$ is not applicable because it does not consider epistemic uncertainty, while $R_2$ can use the posterior class probability as $\mathbb{E}(\lambda^o_{k+1})$.
Optimally, the robot would plan to go through the high separation low uncertainty zone. We have five possible motion primitives, as represented in Fig.~\ref{fig:Plan_Motion_Primitives_1} with a vision cone of $120^\circ$ emanating from the camera. We explore the planning decision tree using Monte Carlo Tree Search with a horizon length $L=10$ at each step, then perform the action with the highest reward. The setting for the classifier model, viewing radius and angle, motion, and geometric noise are the same as in the inference simulation. We use 10 hybrid beliefs for $MH$. The trajectory length is 20 time steps. 

\begin{figure}[!htbp]
	
	\begin{subfigure}[b]{0.35\textwidth}
		\includegraphics[width=\textwidth]{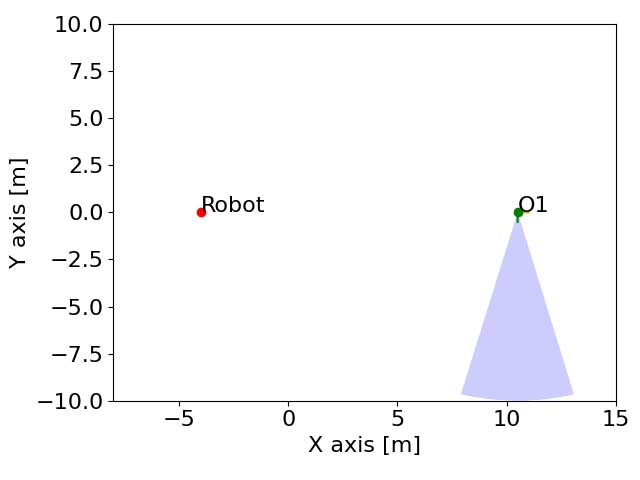}
		\caption{}
		\label{fig:Plan_Single_GT} 
	\end{subfigure}
	\begin{subfigure}[b]{0.35\textwidth}
		\includegraphics[width=\textwidth]{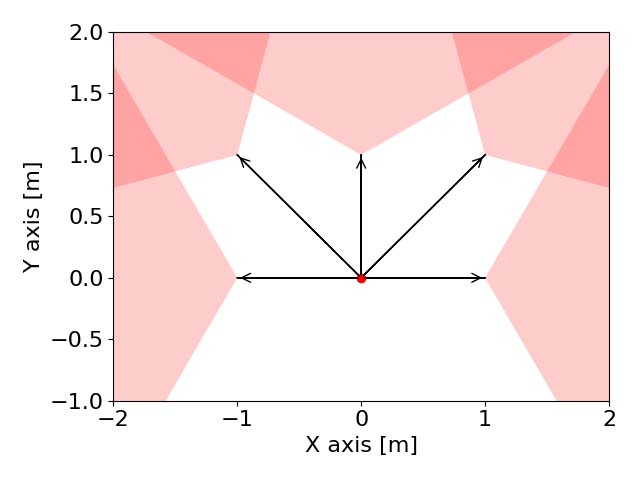}
		\caption{}
		\label{fig:Plan_Motion_Primitives_1} 
	\end{subfigure}
	\caption{
		\textbf{(a)} is the ground truth of the scenario in Sec.~\ref{sec:Plan_Fixed}. The red dot represents the robot's starting point. The green dots represent the objects' location with the corresponding object labels. The green line represents the object orientation, with the yellow line present $90^\circ$ of that orientation. The blue cones represent the observation angles in which the object are classified most accurately with the lowest epistemic uncertainty.
		\textbf{(b)} presents the five motion primitives in the scenario. The red dot represents the origin point, the black arrows the possible actions, and the blue cone is the field of view after the action.}%
	\label{fig:Plan_Single_GT_and_Prim}
\end{figure}

Fig \ref{fig:Plan_Single_Trajectories} presents the ground truth trajectories calculated by performing planning over $R_1$ and $R_2$ both for \jlpbsp and \mhbsp, and planning for $R_2$ for \weu. It is evident that the epistemic-uncertainty-aware methods seek to pass near the blue-cone area for more accurate classification with lower epistemic uncertainty. Methods that plan over $R_1$ tend to pass through the cone.

\begin{figure}[!htbp]
	
	\begin{subfigure}[b]{0.32\textwidth}
		\includegraphics[width=\textwidth]{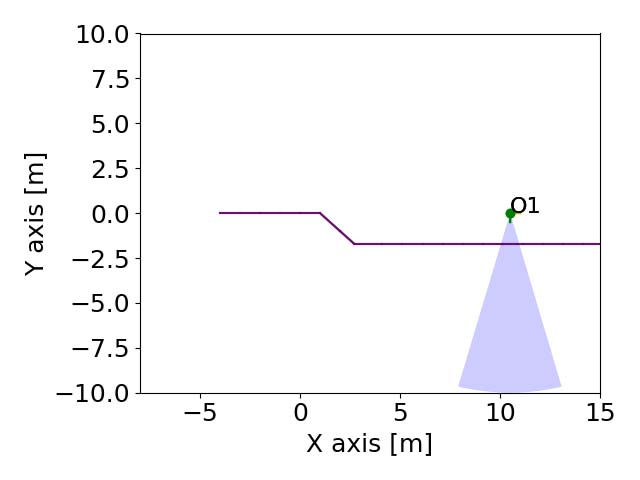}
		\caption{$R_1$}
		\label{fig:Plan_Single_Red} 
	\end{subfigure}
	\begin{subfigure}[b]{0.32\textwidth}
		\includegraphics[width=\textwidth]{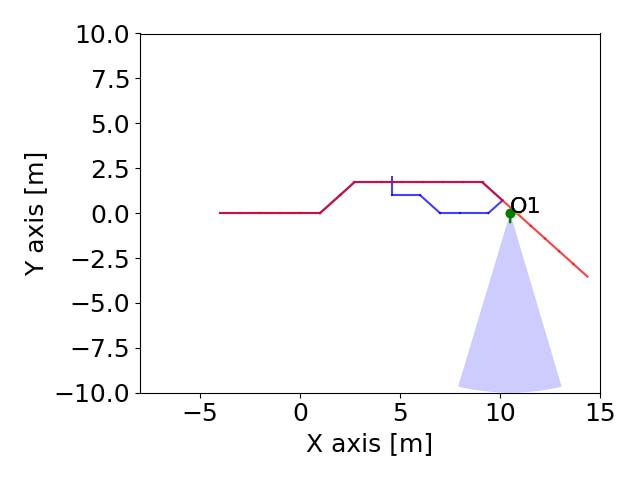}
		\caption{$R_2$}
		\label{fig:Plan_Single_Blue} 
	\end{subfigure}
	\begin{subfigure}[b]{0.32\textwidth}
		\includegraphics[width=\textwidth]{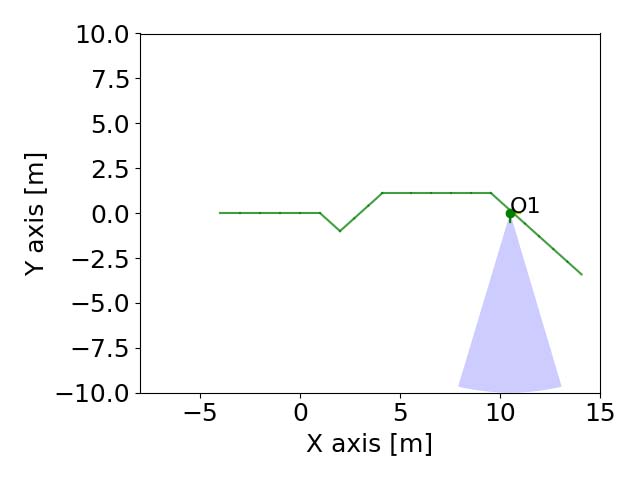}
		\caption{\weu}
		\label{fig:Plan_Single_Green} 
	\end{subfigure}
	
	\caption{This figure presents the ground truth of a planned trajectories. \textbf{(a)} for planning over $R_1$ for \mhbsp (purple) and \jlpbsp (black). \textbf{(b)} for planning over $R_2$ for \mhbsp (red) and \jlpbsp (blue). \textbf{(b)} for \weu (green). All are for the multiple object scenario. The object is shown in a green dot, with the green line representing the object orientation, with the yellow line present $90^\circ$ of that orientation. The blue cones represent the areas where observations have the lowest epistemic uncertainty.}%
	\label{fig:Plan_Single_Trajectories}
\end{figure}

The behavior presented in Fig~\ref{fig:Plan_Single_Trajectories} is reflected in Fig.~\ref{fig:Plan_Stat_Rewards} where the values of $H(\lambda_k)$ are shown as a function of time during inference after the corresponding action has been performed. The values of $H(\lambda_k)$ correlate to the epistemic uncertainty. Evidently, planning over $R_1$ yields lower epistemic uncertainty for both \mhbsp (Fig.~\ref{fig:Plan_Single_MH_R1}) and \jlpbsp (Fig.~\ref{fig:Plan_Single_JLP_R1}).

\begin{figure}[!htbp]
	
	\begin{subfigure}[b]{0.35\textwidth}
		\includegraphics[width=\textwidth]{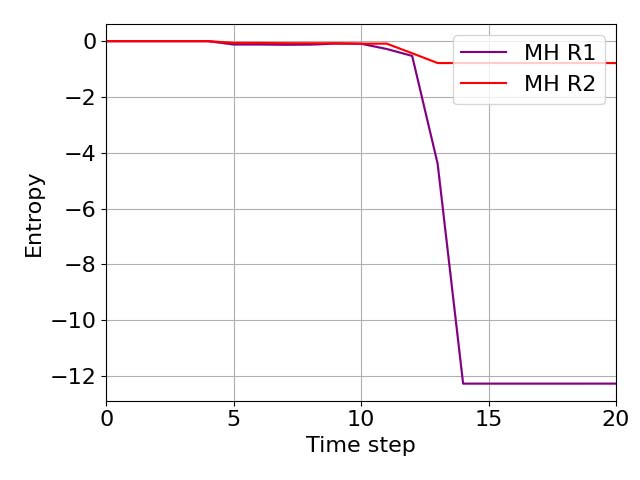}
		\caption{}
		\label{fig:Plan_Single_MH_R1} 
	\end{subfigure}
	\begin{subfigure}[b]{0.35\textwidth}
		\includegraphics[width=\textwidth]{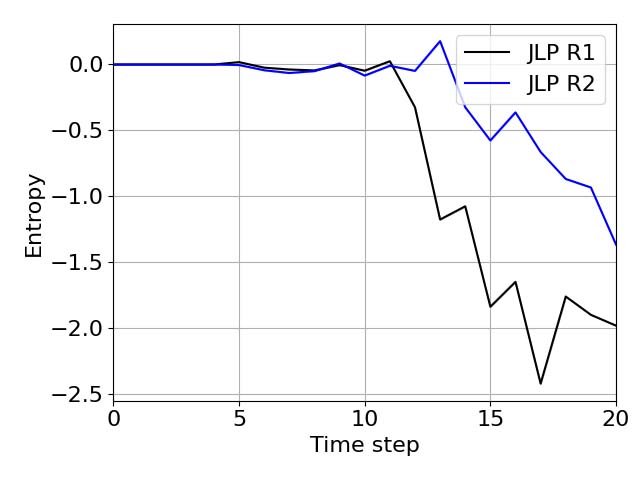}
		\caption{}
		\label{fig:Plan_Single_JLP_R1} 
	\end{subfigure}
	
	\caption{$H(\lambda_{20})$ values for \mh \textbf{(a)} and \jlp \textbf{(b)} as a function of the time step. In \textbf{(a)}, the purple and red plots represent $R_1$ and $R_2$ respectively, and similarly in $\textbf{(a)}$ , the black and blue plots represent $R_1$ and $R_2$ respectively.}%
	\label{fig:Plan_Stat_Rewards}
\end{figure}

Fig.~\ref{fig:Plan_Stat_MSDE_Single} presents MSDE results for all the methods, split into results for \mh in Fig.~\ref{fig:Plan_Stat_MSDE_R1_Single} and for \jlp in Fig.~\ref{fig:Plan_Stat_MSDE_R2_Single}, both showing comparison to \weu in the green plot. \weu performs significantly worse in this setting than all the other methods. When comparing planning over $R_1$ and $R_2$, the first presents better results than the latter for both \jlp and \mh. 

\begin{figure}[!htbp]
	
	\begin{subfigure}[b]{0.35\textwidth}
		\includegraphics[width=\textwidth]{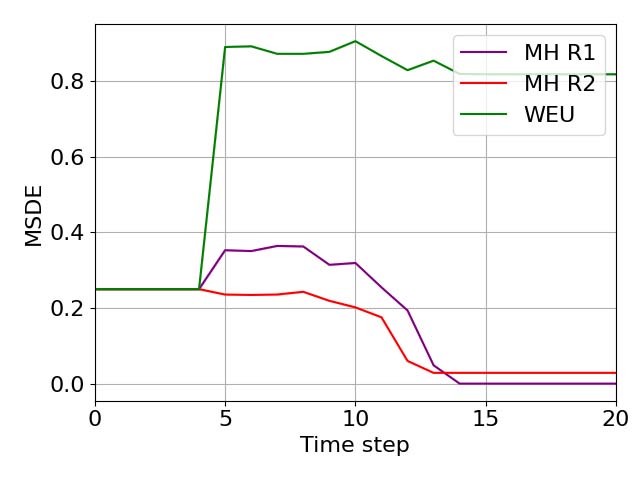}
		\caption{}
		\label{fig:Plan_Stat_MSDE_R1_Single} 
	\end{subfigure}
	\begin{subfigure}[b]{0.35\textwidth}
		\includegraphics[width=\textwidth]{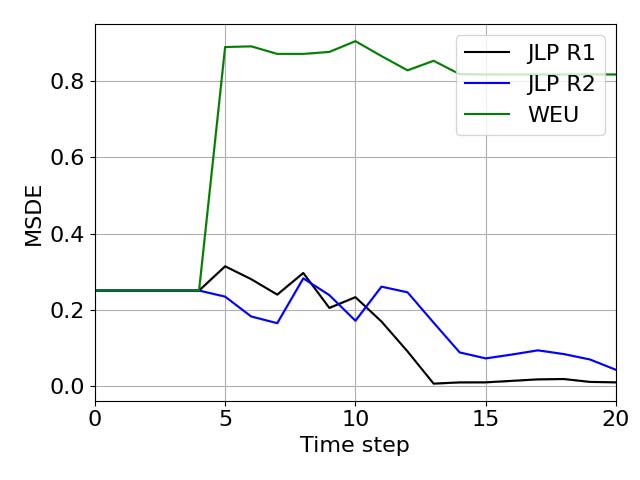}
		\caption{}
		\label{fig:Plan_Stat_MSDE_R2_Single} 
	\end{subfigure}
	\caption{MSDE values as a function of time step for \mh \textbf{(a)} and \jlp \textbf{(b)}, compared to \weu. In \textbf{(a)}, the purple and red plots represent $R_1$ and $R_2$ respectively, and similarly in $\textbf{(a)}$ , the black and blue plots represent $R_1$ and $R_2$ respectively. \weu is represented in both figure with a green plot.}%
	\label{fig:Plan_Stat_MSDE_Single}
\end{figure}

Fig.~\ref{fig:Bars_Plan_Single_Combined} presents the results at time $k=20$ for all methods as a bar graph with error margins for the ground truth class. We can compare the entropy from Fig.~\ref{fig:Plan_Single_Trajectories} and MSDE from Fig.~\ref{fig:Plan_Stat_MSDE_Single} with the bar graphs, with lower entropy values resulting in smaller posterior epistemic uncertainty. Similarly, lower MSDE values result in a more "certain" result in the bar graph, as we can see for methods that plan over $R_1$.

\begin{figure}[!htbp]
	\includegraphics[width=0.45\textwidth]{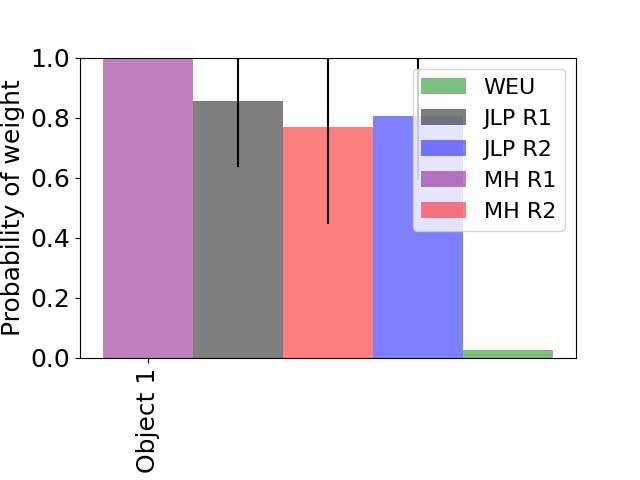}
	\caption{Probability of the object being class $c=2$ (ground truth) for our methods at time $k=20$. We compare planning over $R_1$ and $R_2$, \jlpbsp, \mhbsp, and \weu. Purple and red for $R_1$ and $R_2$ respectively using \mhbsp, black and blue for using  for $R_1$ and $R_2$ respectively using \jlpbsp, and green for \weu. The one $\sigma$ deviation is represented via the black line at each relevant bar, and represents the posterior model uncertainty.}% 
	\label{fig:Bars_Plan_Single_Combined}
\end{figure}

In Fig.~\ref{fig:Plan_Time_Single} we perform computation time comparisons between \weu, \mhbsp and \jlpbsp. The significant advantage in computational time for \jlpbsp is evident against \mhbsp, and while \weu is lower still, \jlpbsp also opens the possibility of reasoning about epistemic uncertainty.

\begin{figure}[!htbp]
	\includegraphics[width=0.5\textwidth]{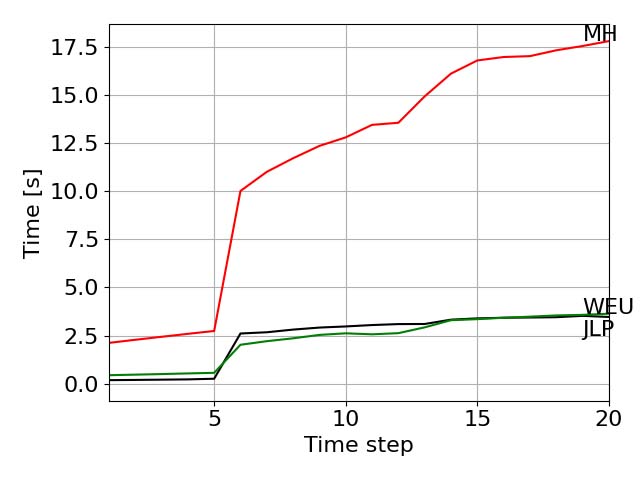}
	\caption{This figure compares run-time per inference step between realizations of \mh with 5 hybrid beliefs in red,  \jlp in black, and \weu in green.}% 
	\label{fig:Plan_Time_Single}
\end{figure}

%------------------------------------------------------------------------
\subsubsection{Planning: Single Run, Multiple Objects}\label{sec:Plan_Fixed}

We simulate a planning scenario of 9 objects, where they formed in a way that there are 3 zones of low uncertainty high expected classification scores, as shown in Fig.~\ref{fig:Plan_Fixed_GT}. Reward function $R_1$ is now modified to include a cap of $R_{max}=5$ per object to to encourage exploration and classification of all objects in the scene. We modify $R_1$ and $R_2$ to include all objects by summing the entropy of each marginal $\lambda_{k+1}$ per object. All in all, the explicit expression for the cost functions for a future $b[\lambda_k]$ is:
\begin{equation}\label{eq:Reward_Functions_Sim}
\begin{split}
	R_1 & = \sum_o \min(-H(\lambda^o_{k+1}), R_{max})\\ 
	R_2 & = - \sum_o H(\mathbb{E}(\lambda^o_{k+1}))\\ 
\end{split}
\end{equation}

Optimally, the robot would plan to go through all three zones to achieve accurate classification of all objects. As in Sec.~\ref{sec:Plan_Single}, We have five possible motion primitive, as presented in Fig.~\ref{fig:Plan_Motion_Primitives} with a cone of vision of $120^\circ$ emanating from the camera. We use MCTS for a horizon $L=10$. We use 10 hybrid beliefs for \mhbsp. The trajectory length is 20 time steps. As in the previous section, we plan for $R_1$ with \mhbsp and \jlpbsp, and for $R_2$ with \mhbsp, \jlpbsp, and \weu.

\begin{figure}[!htbp]
	
	\begin{subfigure}[b]{0.35\textwidth}
		\includegraphics[width=\textwidth]{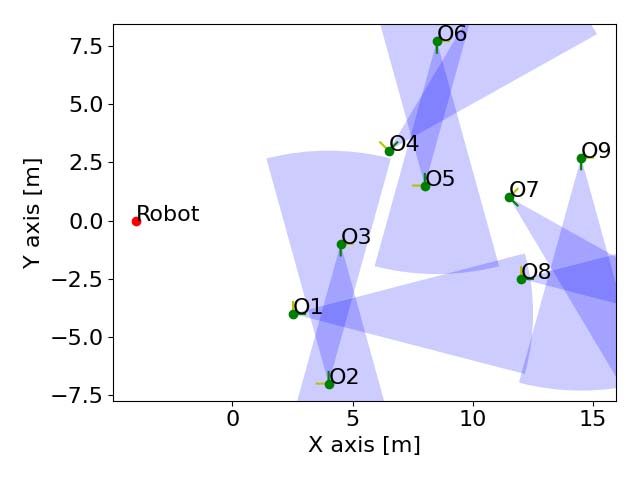}
		\caption{}
		\label{fig:Plan_Fixed_GT} 
	\end{subfigure}
	\begin{subfigure}[b]{0.35\textwidth}
		\includegraphics[width=\textwidth]{Plan_Motion_Primitives.jpg}
		\caption{}
		\label{fig:Plan_Motion_Primitives} 
	\end{subfigure}
	\caption{
	\textbf{(a)} is the ground truth of the scenario in Sec.~\ref{sec:Plan_Fixed}. The red dot represents the robot's starting point. The green dots represent the objects' location with the corresponding object labels. The green line represents the object orientation, with the yellow line present $90^\circ$ of that orientation. The blue cones represent the observation angles in which the objects are classified most accurately with the lowest epistemic uncertainty, with 3 overlapping areas as low epistemic uncertainty areas.
	\textbf{(b)} presents the five motion primitives in the scenario. The red dot represents the origin point, the black arrows the possible actions, and the blue cone is the field of view after the action.}%
	\label{fig:Plan_GT_and_Prim}
\end{figure}

Fig.~\ref{fig:Plan_Fixed_Trajectories} presents the trajectories created for all the methods. The ones that plan over $R_1$ create trajectories pass closer to the overlapping low uncertainty areas from Fig.~\ref{fig:Plan_Fixed_GT}, resulting eventually in more accurate classification compared to planning over $R_2$ for all methods, especially \weu.

\begin{figure}[!htbp]
	
	\begin{subfigure}[b]{0.32\textwidth}
		\includegraphics[width=\textwidth]{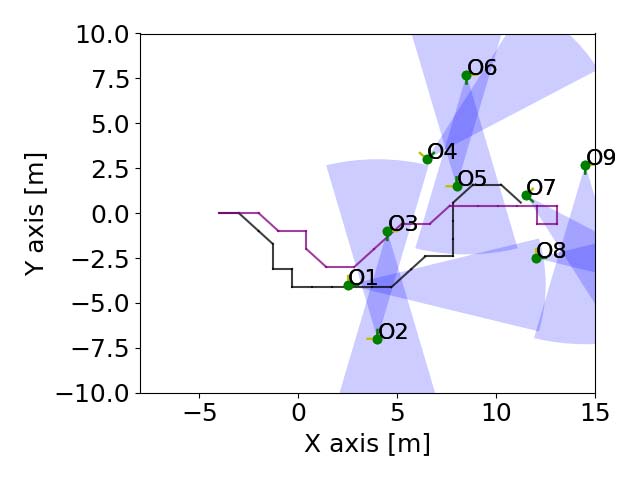}
		\caption{$R_1$}
		\label{fig:Plan_Fixed_Purple_Black} 
	\end{subfigure}
	\begin{subfigure}[b]{0.32\textwidth}
		\includegraphics[width=\textwidth]{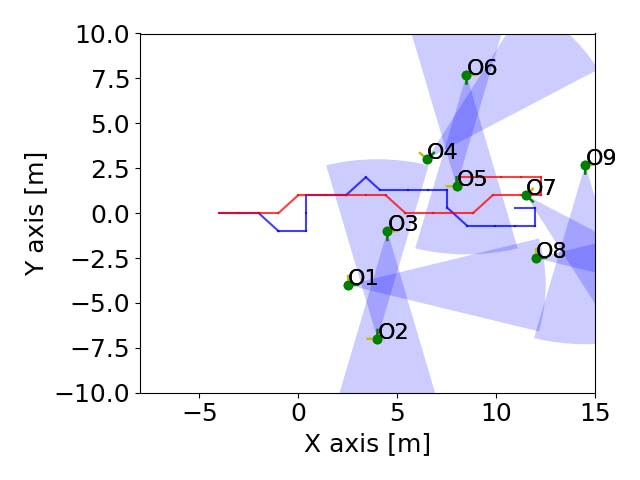}
		\caption{$R_2$}
		\label{fig:Plan_Fixed_Red_Blue} 
	\end{subfigure}
	\begin{subfigure}[b]{0.32\textwidth}
		\includegraphics[width=\textwidth]{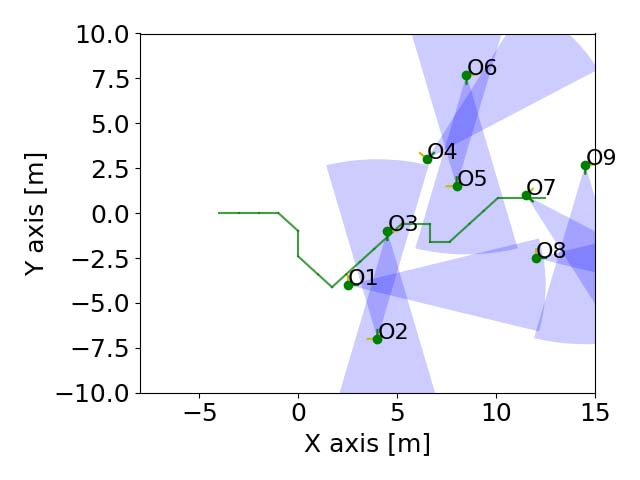}
		\caption{\weu}
		\label{fig:Plan_Fixed_Green} 
	\end{subfigure}
	
	\caption{This figure presents the ground truth of a planned trajectories. \textbf{(a)} for planning over $R_1$ for \mhbsp (purple) and \jlpbsp (black). \textbf{(b)} for planning over $R_2$ for \mhbsp (red) and \jlpbsp (blue). \textbf{(b)} for \weu (green). All are for the multiple object scenario. The object is shown in a green dot, with the green line representing the object orientation, with the yellow line present $90^\circ$ of that orientation. The blue cones represent the areas where observations have the lowest epistemic uncertainty.}%
	\label{fig:Plan_Fixed_Trajectories}
\end{figure}

Fig.~\ref{fig:Plan_Fixed_Rewards}  presents a comparison for $H(\lambda_k)$ at the inference phase, when comparing planning over $R_1$, and $R_2$ for \mhbsp in Fig.~\ref{fig:Plan_Fixed_MH_R1} and \jlpbsp in Fig.~\ref{fig:Plan_Fixed_JLP_R1}. In both figures planning over $R_1$ yields lower entropy, correlating to lower epistemic uncertainty. The effect is more noticeable for \mhbsp than \jlpbsp.

\begin{figure}[!htbp]
	
	\begin{subfigure}[b]{0.35\textwidth}
		\includegraphics[width=\textwidth]{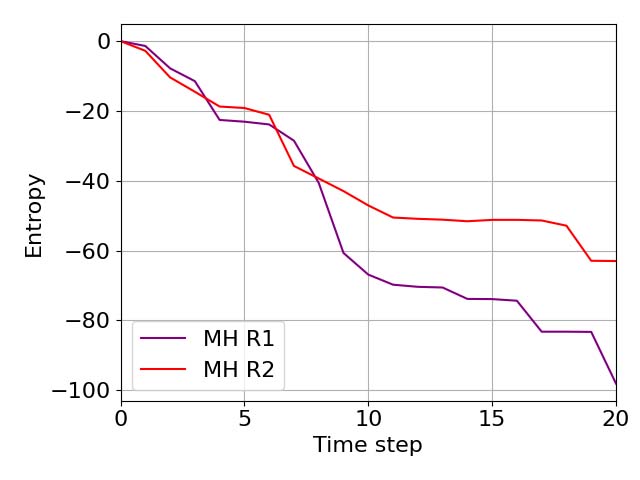}
		\caption{}
		\label{fig:Plan_Fixed_MH_R1} 
	\end{subfigure}
	\begin{subfigure}[b]{0.35\textwidth}
		\includegraphics[width=\textwidth]{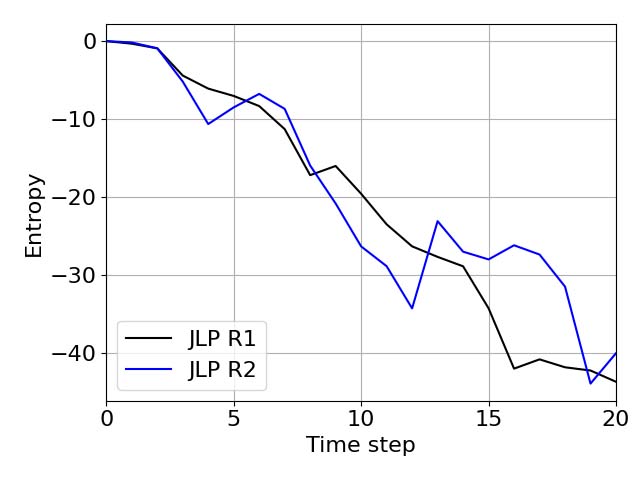}
		\caption{}
		\label{fig:Plan_Fixed_JLP_R1} 
	\end{subfigure}
	
	\caption{$H(\lambda_{20})$ values for \mhbsp \textbf{(a)} and \jlpbsp \textbf{(b)} as a function of the time step. In \textbf{(a)}, the purple and red plots represent $R_1$ and $R_2$ respectively, and similarly in $\textbf{(a)}$ , the black and blue plots represent $R_1$ and $R_2$ respectively.}%
	\label{fig:Plan_Fixed_Rewards}
\end{figure}

Fig.~\ref{fig:Plan_Fixed_MSDE} presents MSDE results for all the methods, split into results for \mhbsp in Fig.~\ref{fig:Plan_Fixed_MSDE_R1}  and for \jlpbsp in Fig.~\ref{fig:Plan_Fixed_MSDE_R2}, both showing comparison to \weu in the green plot. As in Sec~\ref{sec:Plan_Single}, planning over $R_1$ slightly outperforms planning over $R_2$, with \weu lagging far behind.

\begin{figure}[!htbp]
	
	\begin{subfigure}[b]{0.35\textwidth}
		\includegraphics[width=\textwidth]{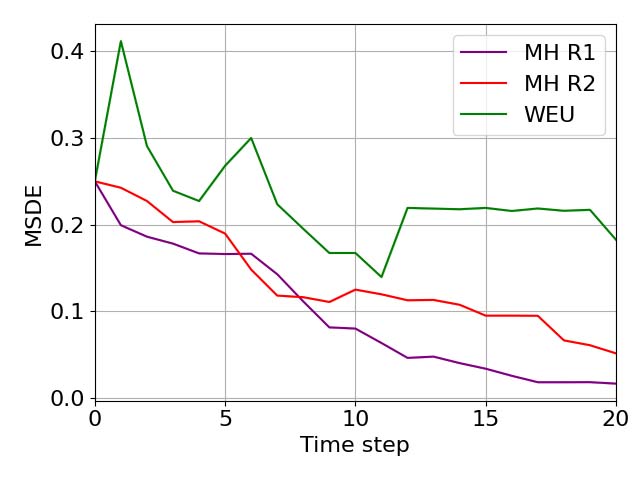}
		\caption{}
		\label{fig:Plan_Fixed_MSDE_R1} 
	\end{subfigure}
	\begin{subfigure}[b]{0.35\textwidth}
		\includegraphics[width=\textwidth]{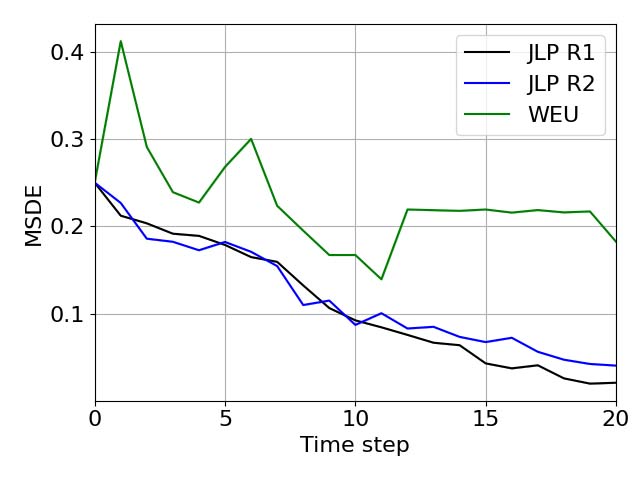}
		\caption{}
		\label{fig:Plan_Fixed_MSDE_R2} 
	\end{subfigure}
	\caption{Single-run study for multiple object scenario study for MSDE comparing planning over $R_1$ and $R_2$ for \mhbsp (\textbf{(a)}) and \jlpbsp (\textbf{(b)}), and \weu. }%
	\label{fig:Plan_Fixed_MSDE}
\end{figure}

Fig.~\ref{fig:Bars_Plan_Fixed_Combined} presents a bar-graph with error representation of the classification results at time $k=20$ for all objects. In general, planning over $R_1$ tend to have more accurate classification compared to planning over $R_2$ with lower uncertainty. On the other hand, \weu tends to go towards extremes of class probabilities 0 or 1, whether it is the correct class or not.

\begin{figure*}[!htbp]
	\includegraphics[width=0.9\textwidth]{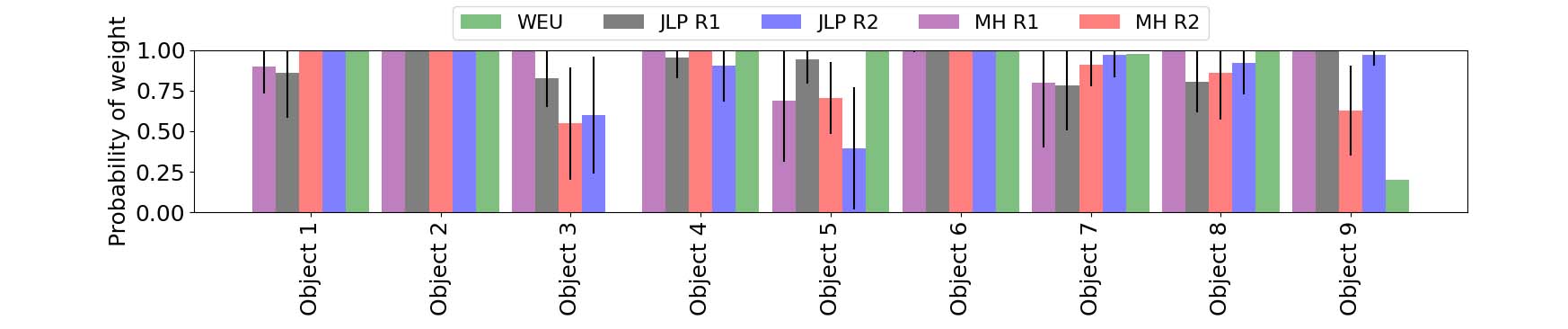}
	\caption{Probability of the objects being ground truth class for our methods at time $k=20$ for all objects. We compare planning over $R_1$ and $R_2$, \jlpbsp, \mhbsp, and \weu. Purple and red for $R_1$ and $R_2$ respectively using \mh, black and blue for using for $R_1$ and $R_2$ respectively using \jlpbsp, and green for \weu. The one $\sigma$ deviation is represented via the black line at each relevant bar, and represents the posterior model uncertainty.}% 
	\label{fig:Bars_Plan_Fixed_Combined}
\end{figure*}

In Fig.~\ref{fig:Plan_Fixed_Time} we present the computational time per step for all our approaches using $R_2$ reward function. For \mhbsp, we used 10 hybrid beliefs. This figure shows that \jlpbsp is slightly faster than \weu while also reasoning about posterior epistemic uncertainty, because the number of states in \jlpbsp scales linearly with the number of objects and candidate classes, as opposed to exponentially with \weu and \mhbsp. As in Sec.~\ref{sec:Plan_Single}, \jlpbsp is significantly more computationally efficient than \mhbsp.

\begin{figure}[!htbp]
	\includegraphics[width=0.45\textwidth]{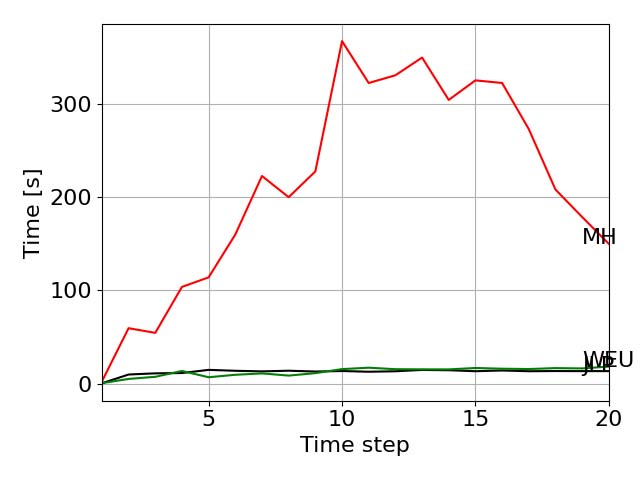}
	\caption{This figure compares run-time per inference step between realizations of \mhbsp with 5 hybrid beliefs in red,  \jlpbsp in black, and \weu in green.}% 
	\label{fig:Plan_Fixed_Time}
\end{figure}

\subsubsection{Planning: Statistical Study}\label{sec:Plan_Stat}

For the statistical study, we randomly corrupt geometric and semantic measurements with noise. We use the scenario from Sec.~\ref{sec:Plan_Fixed}, using $R_1$, and $R_2$ with \jlpbsp, and compare it to \weu. We perform 10 iteration, each with a planning horizon $L=10$, and present results for entropy and MSDE.
%For \mhbsp, 10 hybrid beliefs were used. 
Each run was performed to 20 time-steps.

Fig.~\ref{fig:Plan_Stat_Graphs} presents the statistical results for the sum of the entropy in Fig.~\ref{fig:Plan_Stat_JLP_R1}, and the MSDE results in Fig.~\ref{fig:Plan_Stat_MSDE_JLP}, with the colored areas representing one $\sigma$ deviation. All in all, planning over $R_1$ performs better over planning over $R_2$ for \jlpbsp, with lower entropy and MSDE. In addition, MSDE results compared to \weu are vastily superior for epistemic-uncertainty-aware methods.

\begin{figure}[!htbp]
	
%	\begin{subfigure}[b]{0.32\textwidth}
%		\includegraphics[width=\textwidth]{Plan_Stat_JLP_R1.jpg}
%		\caption{}
%		\label{fig:Plan_Stat_JLP_R1} 
%	\end{subfigure}
	%
	\begin{subfigure}[b]{0.35\textwidth}
		\includegraphics[width=\textwidth]{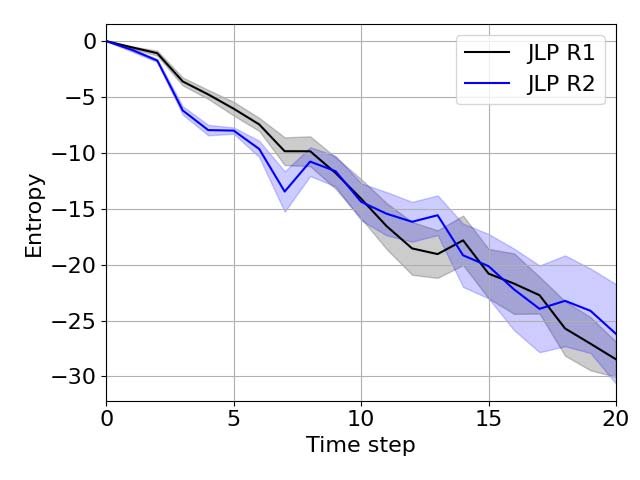}
		\caption{}
		\label{fig:Plan_Stat_JLP_R1} 
	\end{subfigure}
	\begin{subfigure}[b]{0.35\textwidth}
		\includegraphics[width=\textwidth]{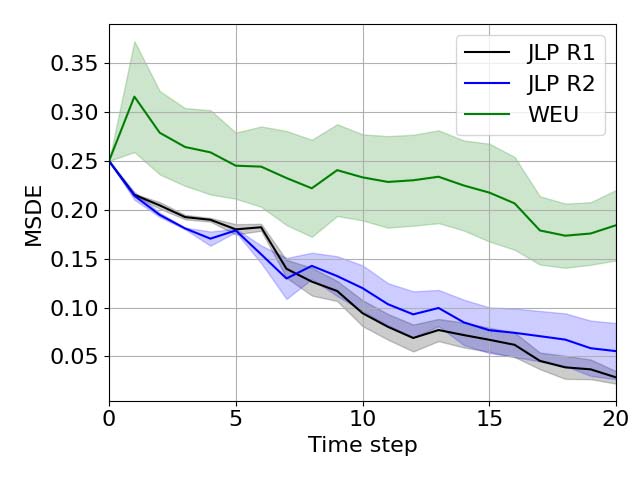}
		\caption{}
		\label{fig:Plan_Stat_MSDE_JLP} 
	\end{subfigure}
	\caption{Statistical study for the scenario in Sec.~\ref{sec:Plan_Stat}
		\textbf{(a)} Presents an comparison for the sum of entropy over all objects between trajectories for $R_1$ and $R_2$ as a function of time step for \mhbsp.
		\textbf{(b)} presents an MSDE comparison between \mhbsp, \jlpbsp, and \weu as a function of time step. In both, the line represent the statistical expectation, while the colored area represents a one $\sigma$ deviation.}%
	\label{fig:Plan_Stat_Graphs}
\end{figure}

%------------------------------------------------------------------
%-------------------------------------------------------------------
\subsection{Experiment}\label{sec:Exp}

\subsubsection{Setup}\label{sec:Exp_setup}

For the experiment, we consider a myopic planning scenario in a semantic SLAM setting, using Active Vision Dataset (AVD) \cite{Ammirato17icra} Home 005 with example images presented in Fig.~\ref{fig:AVD_set}. In this scenario, the objects are grouped to two groups, one on a table near the window back-lit by sunlight as seen in Fig.~\ref{fig:AVD_1}, and another on the kitchen counter seen in Fig.~\ref{fig:AVD_2}. We perform planning for a 20 time step trajectory, at each step performing myopic planning. We aim to compare between \jlpbsp and \weu for classification accuracy using MSDE \eqref{eq:MSDE}, differential entropy representing epistemic uncertainty, and computational time. The reward functions $R_1$ and $R_2$ are identical to those presented in Eq.~\eqref{eq:Reward_Functions_Sim}.

\begin{figure}[!htbp]
	
	\begin{subfigure}[b]{0.32\textwidth}
		\includegraphics[width=\textwidth]{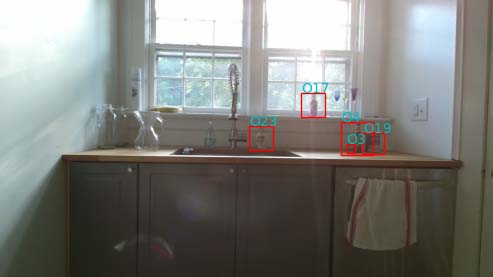}
		\caption{}
		\label{fig:AVD_1} 
	\end{subfigure}
	\begin{subfigure}[b]{0.32\textwidth}
		\includegraphics[width=\textwidth]{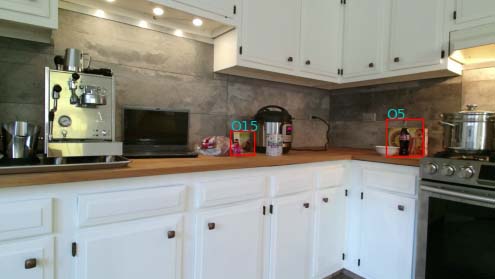}
		\caption{}
		\label{fig:AVD_2} 
	\end{subfigure}
	\begin{subfigure}[b]{0.32\textwidth}
		\includegraphics[width=\textwidth]{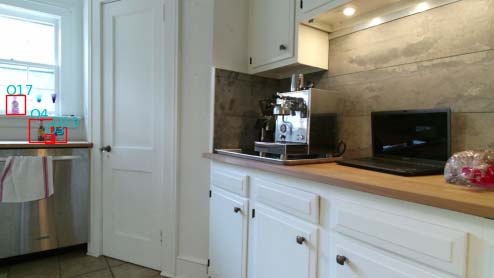}
		\caption{}
		\label{fig:AVD_3} 
	\end{subfigure}
	\caption{Example images of the Active Vision Dataset, home 005. The red boxes represent the bounding boxes for the objects, and the notation $Ox$ represent the $x$'th object.}%
	\label{fig:AVD_set}
\end{figure}

We consider five candidate classes: "Packet", "Book Jacket", "Pop Bottle", "Digital Clock", and "Soap Dispenser". For each class, we trained classifier uncertainty models using images from BigBIRD dataset \cite{Singh14icra}, with example images presented in Fig.~\ref{fig:BB_set}. For classification, we used VGG convolutional neural network \cite{Simonyan14arxiv} with dropout activated during test time. The $R_1$ upper limit $R_{max}$ per object is 500, as the increase number of objects increases the scale of $R_1$ values; Recall Lemma \ref{lemma:LG_entropy_exact}, the entropy depends on the covariance of $\loggamma$ via $H(\loglambda)$.

\begin{figure}[!htbp]
	
	\centering
	
	\begin{subfigure}[b]{0.12\textwidth}
		\includegraphics[width=\textwidth]{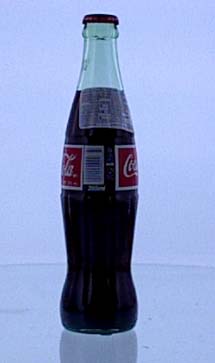}
		\caption{}
		\label{fig:BB_1} 
	\end{subfigure}
	\begin{subfigure}[b]{0.12\textwidth}
		\includegraphics[width=\textwidth]{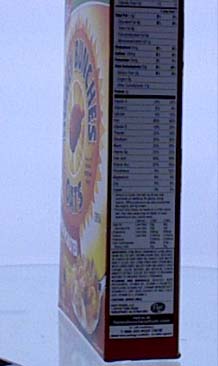}
		\caption{}
		\label{fig:BB_2} 
	\end{subfigure}
	\begin{subfigure}[b]{0.12\textwidth}
		\includegraphics[width=\textwidth]{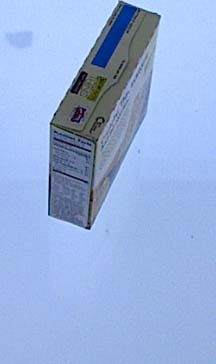}
		\caption{}
		\label{fig:BB_3} 
	\end{subfigure}
	\begin{subfigure}[b]{0.12\textwidth}
		\includegraphics[width=\textwidth]{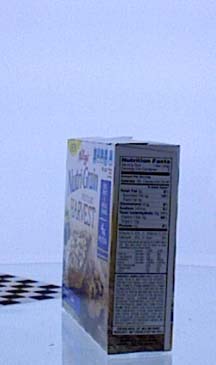}
		\caption{}
		\label{fig:BB_4} 
	\end{subfigure}
	\caption{Example images of the BigBIRD dataset for training the classifier models. \textbf{(a)} is an example for "pop bottle" class, while the rest are examples for "packet".}%
	\label{fig:BB_set}
\end{figure}

The classifier models were trained via PyTorch on fully connected networks. Recall Eq.~\eqref{eq:Classifier_Uncertainty_Model}, we train $h_c(x^{rel})$ and $\Sigma_c(x^{rel})$ from a dataset $D_c=\{ x^{rel} , \{ \loggamma \} \}$ per object, where $x^{rel} = [ \psi , \theta ]$ is parametrized by relative yaw angle $\psi$ and relative pitch angle $\theta$. $h_c$ and $\Sigma_c$ are represented by separate neural networks, up to a total of $2m$ networks. As seen in Sec.~\ref{sec:JLP}, all $\Sigma_{c=i} \equiv \Sigma_{c=j}$ for $i,j=1,...,m-1$ for the JLP factor to be Gaussian. This constraint limits the expressibility of $\Sigma_c$, thus not accurately representing the epistemic uncertainty from certain viewpoints of objects. As such, instead of enforcing a hard constraint on all $\Sigma_c$, we train the classifier uncertainty model with a loss function that imposes a penalty if $\Sigma_c$ for different $c$ are not similar, enforcing a soft constraint.

The loss function $L_h$ for the $h_c$ network is mean square error (MSE):
\begin{equation}
	L_h(h_c, \{\loggamma\}) = MSE(h_c, \{\loggamma\}) = \sum_{i=1}^{m} \left( h_c^i - \mathbb{E}(\loggamma^i) \right)^2,
\end{equation}
where $h_c^i$ is the $i$'th element of $h_c$. 
The loss function $L_{\Sigma}$ for the $\Sigma_c$ uses MSE over the covariance matrix elements, and adds a Forbenius norm term that acts as the soft constraint that makes the values of $\Sigma_c$ closer:
\begin{equation}
	L_{\Sigma}(h_c, \Sigma_c, \{\loggamma\}) = MSE(\Sigma_c, \Sigma(\loggamma)) 
	+ \kappa \cdot F_N(h_c, \Sigma_c) 
\end{equation}
where the MSE for the above loss function is defined:
\begin{equation}
	MSE(\Sigma_c, \Sigma(\loggamma)) = \frac{1}{(m-1)^2} \sum_{i=1}^{m} \sum_{j=1}^{m} ([\Sigma_c]_{ij} - [\Sigma(\loggamma)]_{ij})^2,
\end{equation}
$F_N(\cdot)$ is the Forbenius Norm, defined:
\begin{equation}
	F_N(\Sigma_c)  = Tr \left( (\Sigma^{-1}_{c=i} - \Sigma^{-1}_{c=m}) \cdot
	(\Sigma^{-1}_{c=i} - \Sigma^{-1}_{c=m})^T \right),
\end{equation}
and $\kappa$ is a positive constant. In our case, $\kappa = 0.005$.

\vspace{0.3cm}
\subsubsection{Results}

Fig.~\ref{fig:AVD_Paths} presents the paths created by the planning session. The path for planning over $R_1$ focuses on the object group on the kitchen counter, while the others focus more on the object on the table by the window. This can be explained by poorer visibility of the objects near the window, induced by the sunlight, therefore inducing higher epistemic uncertainty than the objects on the counter.

\begin{figure}[!htbp]
	
	\begin{subfigure}[b]{0.32\textwidth}
		\includegraphics[width=\textwidth]{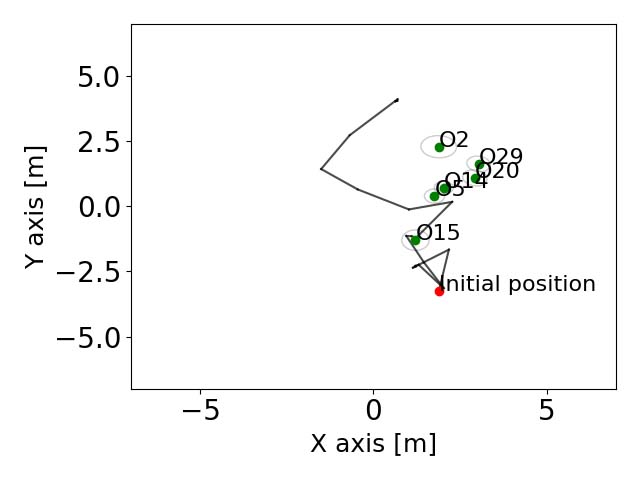}
		\caption{$R_1$}
		\label{fig:AVD_Black} 
	\end{subfigure}
	\begin{subfigure}[b]{0.32\textwidth}
		\includegraphics[width=\textwidth]{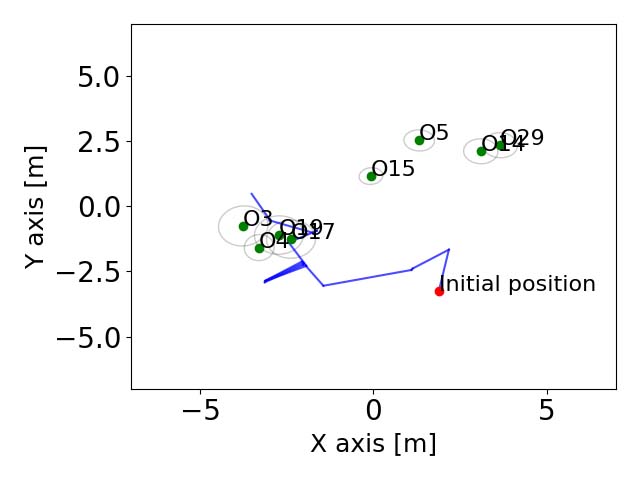}
		\caption{$R_2$}
		\label{fig:AVD_Blue} 
	\end{subfigure}
	\begin{subfigure}[b]{0.32\textwidth}
		\includegraphics[width=\textwidth]{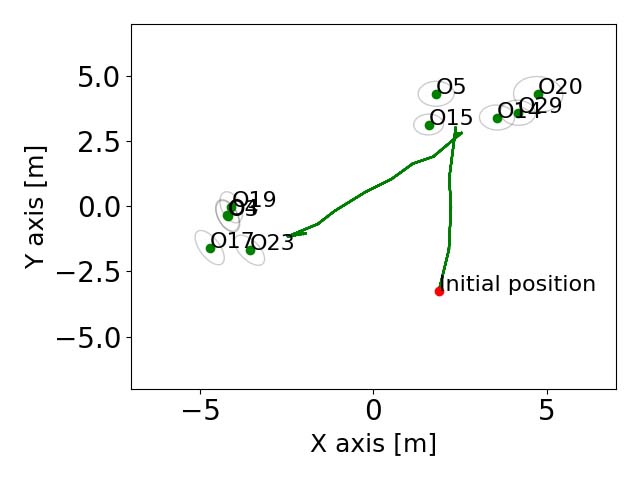}
		\caption{\weu}
		\label{fig:AVD_Green} 
	\end{subfigure}
	
	\caption{This figure presents the ground truth of a planned trajectories with object pose estimations. \textbf{(a)} for planning over $R_1$ for \jlpbsp in black. \textbf{(b)} for planning over $R_2$ for \jlpbsp in blue. \textbf{(c)} for \weu (green). All for the AVD scenario. The object estimation is shown in a green dot with corresponding estimation covariance of $3\sigma$ in gray. The red dots represent the starting position of each trajectory.}%
	\label{fig:AVD_Paths}
\end{figure}

The results of those trajectories chosen can be seen in Fig.~\ref{fig:AVD_Graphs}, where the entropy and MSDE results are presented. In Fig.~\ref{fig:AVD_JLP_R1} the lower epistemic uncertainty for planning with $R_2$ can be evident. In addition, the MSDE comparison in Fig.~\ref{fig:AVD_MSDE_JLP} significantly favors planning over $R_1$ over $R_2$ and especially compared to \weu, with epistemic-uncertainty-aware planning outperforms both.

\begin{figure}[!htbp]

	\begin{subfigure}[b]{0.35\textwidth}
		\includegraphics[width=\textwidth]{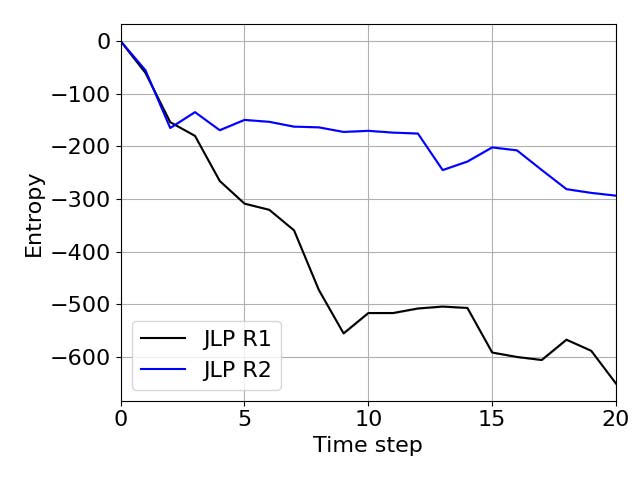}
		\caption{}
		\label{fig:AVD_JLP_R1} 
	\end{subfigure}
	\begin{subfigure}[b]{0.35\textwidth}
		\includegraphics[width=\textwidth]{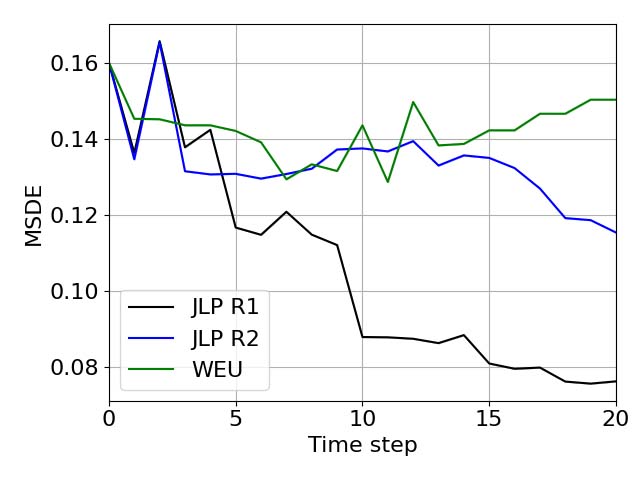}
		\caption{}
		\label{fig:AVD_MSDE_JLP} 
	\end{subfigure}
	\caption{Experimental results for the scenario in Sec.~\ref{sec:Exp_setup}
		\textbf{(a)} Presents an comparison for the sum of entropy over all objects between trajectories for $R_1$ and $R_2$ as a function of time step for \jlpbsp.
		\textbf{(b)} presents an MSDE comparison between \jlpbsp, and \weu as a function of time step.}%
	\label{fig:AVD_Graphs}
\end{figure}

Fig.~\ref{fig:AVD_Bars} shows the class probability of the ground truth class for all the objects for time-step $k=20$. While both \jlpbsp with $R_2$ and \weu observe an object more as the group near the window contains more objects, the objects that \jlpbsp with $R_1$ observes are classified more accurately.

\begin{figure*}[!htbp]
	\includegraphics[width=0.9\textwidth]{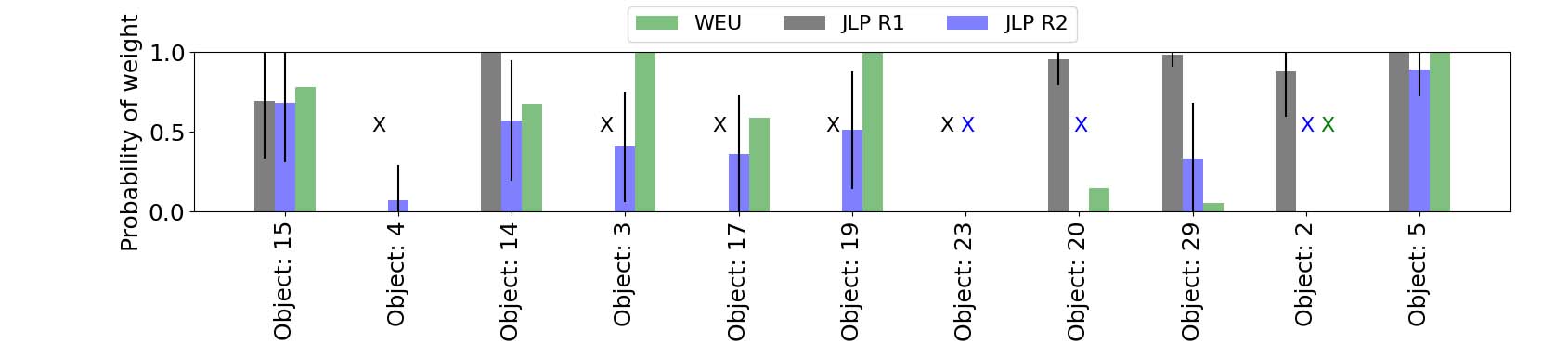}
	\caption{Probability of the objects being ground truth class for our methods at time $k=20$ for all objects. We compare planning over $R_1$ and $R_2$, \jlpbsp,and \weu. Black and blue for using for $R_1$ and $R_2$ respectively using \jlpbsp, and green for \weu. The one $\sigma$ deviation is represented via the black line at each relevant bar, and represents the posterior model uncertainty. The colored X marks represent that object wasn't observed by the corresponding method.}% 
	\label{fig:AVD_Bars}
\end{figure*}

Fig.~\ref{fig:AVD_Time} presents a computational time comparison between \jlpbsp and \weu. The figure shows a significant advantage for \jlpbsp over \weu, as this time the number of candidate classes is 5, instead of 2 in the simulation. \weu computational time per step drops with time steps as some class realization are pruned. As evident from the figure, \jlpbsp offers computational efficiency greater than \weu, while also opening access to model uncertainty, both for inference and planning.

\begin{figure}[!htbp]
	\includegraphics[width=0.45\textwidth]{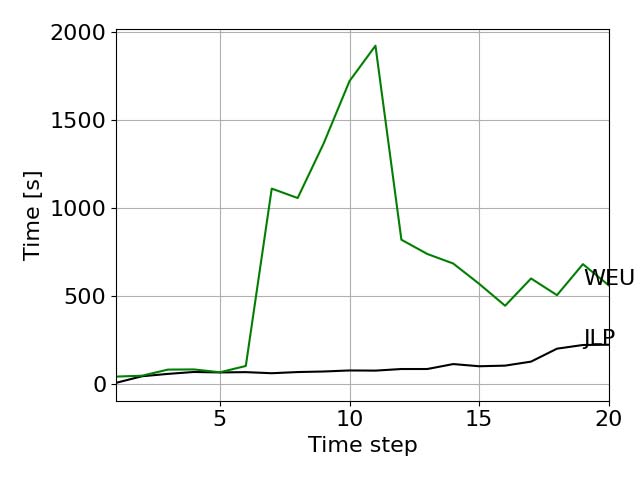}
	\caption{This figure compares run-time per inference step between realizations of \jlpbsp in black, and \weu in green for the AVD scenario.}% 
	\label{fig:AVD_Time}
\end{figure}

	%This work was partially supported by the Technion Autonomous Systems Program (TASP).
	
	% ======================
	\section{Conclusions}
	\label{sec:conclusions}
	% !TeX root = Paper Main.tex

We presented a unified semantic SLAM framework for inference and BSP that maintains a joint belief over robot and objects' poses and posterior class probability, addressing viewpoint-based classification aliasing and reasoning about epistemic uncertainty of the classifier. In particular, two approaches were introduced; Firstly, we introduced \mh which maintains simultaneously multiple hybrid beliefs over poses and object classes, with semantic class probability vector measurements varying with different predetermined weights. Secondly, we introduced \jlp, which is a more computationally efficient alternative that uses the novel \jlp factor. Furthermore, we introduced \mhbsp and \jlpbsp as the formulation of both approaches to a BSP framework, and introduced a novel information-theoretic reward to plan over future posterior epistemic uncertainty, improving classification performance over methods and reward functions that do not consider epistemic uncertainty. Both approaches leverage the coupling between relative poses and object classes via a viewpoint dependent classifier uncertainty model, which also allows us to predict future epistemic uncertainty for planning. In simulation and experiment we showed that reasoning about epistemic uncertainty improves classification performance both in inference and planning.

	\appendix
	\section{Appendix}
	\label{sec:Appendix} 
	\subsection{Proof of Lemma \ref{lemma:LG_entropy_exact}}\label{sec:Proof2}

The reverse logit transformation from $\loglambda$ to $\lambda$ is given by:
\begin{equation}\label{eq:app_reverse_transform}
\lambda = \left[ \frac{e^{\loglambda^1}}{1 + \sum_{i=1}^{m-1} e^{\loglambda^i}},...,
\frac{e^{\loglambda^{m-1}}}{1 + \sum_{i=1}^{m-1} e^{\loglambda^i}}, 
\frac{e^{1}}{1 + \sum_{i=1}^{m-1} e^{\loglambda^i}}	 \right]^T.
\end{equation}
Thus, $\lambda$ is LG distributed, and the probability density function is given as:
\begin{equation}
\prob{\lambda} = \frac{1}{\sqrt{2\pi |\Sigma|}} \cdot
\frac{1}{\prod_{i=1}^m \lambda^i} \cdot
e^{-\frac{1}{2} || \loglambda -  \mu||^2_{\Sigma} },
\end{equation}
with $\mu \in \mathbb{R}^{m-1}$ and $\Sigma \in \mathbb{R}^{(m-1) \times (m-1)}$ being the LG parameters. The term $\frac{1}{\prod_{i=1}^m \lambda^i}$ is the determinant of the transformation Jacobian, and is denoted as $|J(\lambda)|$. Thus we write $\prob{\lambda}$ as:
\begin{equation}
\prob{\lambda} = \prob{\loglambda} |J(\lambda)|,
\end{equation}
with $\prob{\loglambda} \triangleq \prob{\loglambda} = \prob{\loglambda}(\mu,\Sigma)$,
and write $H(\lambda)$ as:
\begin{equation}\label{eq:app_ent_basic}
H(\lambda) = - \int_\lambda \prob{\lambda} \cdot \log \prob{\lambda} d\lambda
\end{equation}
Then, we transform the integral variable back to $\loglambda$, as we have a closed form expression for $H(\loglambda)$. As $|J(\lambda)|$ is the transformation Jacobian, $|J(\lambda)| d\lambda = d\loglambda$. From there we can write the integral in Eq.~\eqref{eq:app_ent_basic} as a function of $\loglambda$:
\begin{equation}\label{eq:app_entropy_seperation}
\begin{split}
H(\lambda) = & - \int_\lambda \prob{\loglambda} \cdot |J(\lambda)| \cdot \log(\prob{\loglambda} \cdot |J(\lambda)|) d\lambda = \\
& - \int_\loglambda \prob{\loglambda} \cdot \log (\prob{\loglambda} \cdot |J(\lambda)| ) dy = \\
& - \int_\loglambda \prob{\loglambda} \cdot \log (\prob{\loglambda}) d\loglambda - \int_\loglambda \prob{\loglambda} \cdot \log (|J(\lambda)|) d\loglambda = \\ 
& H(\loglambda) - \int_\loglambda \prob{\loglambda} \cdot \log (|J(\lambda)|) d\loglambda.
\end{split}
\end{equation}
The term $\int_\loglambda \prob{\loglambda} \cdot \log (|J(\lambda)|) d\loglambda$ is positive, as $\prob{\loglambda}$ is always positive and $J(\lambda) > 1$, therefore $H(\lambda) < H(\loglambda)$. Next, we describe $\log (|J(\lambda)|)$ in as a function of $\loglambda$:
\begin{equation}
\begin{split}
\log|J(\lambda)| = & \log \left( \frac{1}{\prod_{i=1}^m \lambda^i} \right) = - \sum_{i=1}^{m-1} \log \lambda^i - \log \lambda^m = \\
& - \sum_{i=1}^{m-1} \left[ \log(e^{\loglambda^i}) + \log\left(1+\sum_{j=1}^{m-1} e^{\loglambda^j}\right) \right] \\ & - \log(1) + \log\left(1+\sum_{j=1}^{m-1} e^{\loglambda^j}\right) = \\
& - \sum_{i=1}^{m-1} \loglambda^i + m\cdot\log\left(1+\sum_{j=1}^{m-1} e^{\loglambda^j}\right).
\end{split}
\end{equation}
Now we plug the above expression for $\log|J(\lambda)|$ into Eq.~\eqref{eq:app_entropy_seperation} and express $H(\lambda)$ as a function of $\loglambda$:
\begin{equation}
\begin{split}
H(\lambda) = & H(\loglambda) \\ & + \int_\loglambda \prob{\loglambda} \cdot \left[ \sum_{i=1}^{m-1} \loglambda^i - m\cdot\log\left(1+\sum_{j=1}^{m-1} e^{\loglambda^j}\right) \right] d\loglambda.
\end{split}
\end{equation}
As $\int_\loglambda \prob{\loglambda} \loglambda^i d\loglambda = \mathbb{E}[\loglambda^i]$, we can simplify the above equation into the form shown in Lemma \ref{lemma:LG_entropy_exact}:
\begin{equation}\label{eq:LG_entropy_exact_2}
H(\lambda) = H(\loglambda) + \sum_{i=1}^{m-1} \mathbb{E}(\loglambda^i) - m\int_\loglambda \log\left(1+\sum_{j=1}^{m-1} e^{\loglambda^j}\right) \cdot \prob{\loglambda} d\loglambda.
\end{equation}

%----------------------------------------------
\subsection{Proof of Lemma \ref{lemma:LG_entropy_bounds}}\label{sec:Proof3}

\subsubsection{Upper Bound}

Let us look at the integral in Eq.~\eqref{eq:LG_entropy_exact_2}. The term $\log\left(1+\sum_{j=1}^{m-1} e^{\loglambda^j}\right)$ can be bounded from below by:
\begin{equation}
	\log\left(1+\sum_{j=1}^{m-1} e^{\loglambda^j}\right) \geq
	\log(\max_i \{1,e^{\loglambda^i}\}) = \max_i \{0,\loglambda^i\}. 
\end{equation}
Substituting the above equation to Eq.~\eqref{eq:LG_entropy_exact_2} yields the following inequality:
\begin{equation}\label{eq:app_pre_final_upper}
	H(\lambda) \leq H(\loglambda) + \sum_{i=1}^{m-1} \mathbb{E}(\loglambda^i) - m\int_\loglambda \max_i \{0,\loglambda^i\} \cdot \prob{\loglambda} d\loglambda.
\end{equation}
The integral term is similar to the expectation definition for $\loglambda_i$, except that it considers only positive $\loglambda_i$, making the resulting value from the integral larger than $\mathbb{E}(\loglambda_i)$.
For the next step, we consider the case where there is at least a single $\mathbb{E}[\loglambda^i] \geq 0$, and the case where for all $i$, $\mathbb{E}[\loglambda^i] < 0$. Considering both cases we can write:
\[
\begin{cases}
	\int_\loglambda \max_i \{0,\loglambda^i\} \cdot \prob{\loglambda} d\loglambda \geq 0 &
	\mathbb{E}(\loglambda^i) < 0: \forall i \\
	\int_\loglambda \max_i \{0,\loglambda^i\} \cdot \prob{\loglambda} d\loglambda \geq \max \mathbb{E}(\loglambda^i) & \exists \mathbb{E}(\loglambda^i) \geq 0.
\end{cases}
\]
Considering both cases:
\begin{equation}
	\int_\loglambda \max_i \{0,\loglambda^i\} \cdot \prob{\loglambda} d\loglambda \geq 
	\max \{ 0, \mathbb{E}\loglambda^i \}.
\end{equation}
Finally, we can substitute the above expression into Eq.~\eqref{eq:app_pre_final_upper} and get the expression in Lemma \ref{lemma:LG_entropy_bounds}:
\begin{equation}
	H(\lambda) \leq H(\loglambda) + \sum_{i=1}^{m-1} \mathbb{E}(\loglambda^i) - m\max \{ 0, \mathbb{E}(\loglambda^i) \}.
\end{equation}

\subsubsection{Lower Bound}

Let us look again at the integral in Eq.~\eqref{eq:LG_entropy_exact_2}. This time, the term $\log\left(1+\sum_{j=1}^{m-1} e^{\loglambda^j}\right)$ can be bounded from above by:
\begin{equation}
\begin{split}
	\log\left(1+\sum_{j=1}^{m-1} e^{\loglambda^j}\right) \leq &  
	\log(\max_i \{m,me^{\loglambda^i}\}) = \\ & \max_i \{0,\loglambda^i\} + \log(m). 
\end{split}
\end{equation}
Now, we substitute the above inequality into Eq.~\eqref{eq:LG_entropy_exact_2}, and we get the following expression:
\begin{equation}\label{eq:app_pre_final_lower}
H(\lambda) \leq H(\loglambda) + \sum_{i=1}^{m-1} \mathbb{E}(\loglambda^i) - m \log(m) - m\int_\loglambda \max_i \{0,\loglambda^i\} \cdot \prob{\loglambda} d\loglambda.
\end{equation}
This time we look for an upper bound for $\int_\loglambda \max_i \{0,\loglambda^i\} \cdot \prob{\loglambda} d\loglambda$. Let us consider that:
\begin{equation}\label{eq:app_lower_bound_integral}
	\int_0^\infty \frac{\loglambda^i }{\sqrt{2 \pi}} e^{-\frac{(\loglambda^i)^2}{2\Sigma_{ii}} 
	} d\loglambda^i = \sqrt{\frac{\Sigma_{ii}}{2\pi}} 
	\leq \sqrt{\frac{\Sigma^{max}_{ii}}{2\pi}}
\end{equation}
where $\Sigma_{ii}$ is the element $(i,i)$ in the diagonal of matrix $\Sigma$, and $\Sigma_{ii}^{max}$ is the largest element of $\Sigma$. Then we can bound $\int_\loglambda \max_i \{0,\loglambda^i\} \cdot \prob{\loglambda} d\loglambda$ by:
\[
\begin{cases}
	\int_\loglambda \max_i \{0,\loglambda^i\} \cdot \prob{\loglambda} d\loglambda \leq \sqrt{\frac{\Sigma_{ii}^{max}}{2\pi}} &
	\mathbb{E}(\loglambda^i) < 0: \forall i \\
	\int_\loglambda \max_i \{0,\loglambda^i\} \cdot \prob{\loglambda} d\loglambda \leq \max \mathbb{E}(\loglambda^i) + \sqrt{\frac{\Sigma_{ii}^{max}}{2\pi}} & \exists \mathbb{E}(\loglambda^i) \geq 0,
\end{cases}
\]
From the above equation, we reach:
\begin{equation}
	\int_\loglambda \max_i \{0,\loglambda^i\} \cdot \prob{\loglambda} d\loglambda \leq \max_i \{0,\mathbb{E}(\loglambda^i)\} + \sqrt{\frac{\Sigma_{ii}^{max}}{2\pi}},
\end{equation}
and by substituting into Eq.~\eqref{eq:LG_entropy_exact_2}, we reach the lower bound presented in Lemma \ref{lemma:LG_entropy_bounds}:
\begin{equation}
\begin{split}
	H(\lambda) \geq & H(\loglambda) + \sum_{i=1}^{m-1} \mathbb{E}(\loggamma^i) \\ & - m \cdot \max_i \{ 0 , \mathbb{E}(\loggamma^i) \} - m \log m - \sqrt{\frac{\sigma_{ii}^{max}}{2\pi}}.
\end{split}
\end{equation}

	%%
	%%\section*{Appendix A: Inference}
	%%\label{Sec:AppendixA}
	%%\input{Appendix-Inference}
	%%
	%%\section*{Appendix B: On Computational Complexity of DA-BSP}
	%%\label{Sec:Appendix-Complexity}
	%%\input{Appendix-Complexity}
	%%
	%%\section*{Appendix C: Degenerate Cases of Data Association Aware BSP}
	%%\label{Sec:Appendix-DegenerateCases}
	%%\input{Appendix-DegenerateCases}
	% ======================
	
%	\begin{appendices}
%		\section{Viewpoint-Dependent Semantic Factor}
%		\input{A-ClassifierModel}
%	\end{appendices}
%	\vspace{-0.3cm}
%	\bibliographystyle{unsrt}
	\bibliographystyle{IEEEtran}
	\bibliography{refs}
%	\bibliography{../../../References/refs}
%	\bibliography{/Users/indelman/Vadim/PROFESSIONAL/RESEARCH/PAPERS/References/refs}

\end{document}